\documentclass[12pt,oneside,reqno,english]{article}
\usepackage[T1]{fontenc}
\usepackage[latin9]{inputenc}
\usepackage{geometry}
\usepackage{xcolor}
\usepackage{babel}
\usepackage{float}
\usepackage{booktabs}
\usepackage{mathtools}
\usepackage{bm}
\usepackage{amstext}
\usepackage{amsthm}
\usepackage{amssymb}
\usepackage{cancel}
\usepackage{stmaryrd}
\usepackage{graphicx}
\usepackage[authoryear,round,sort]{natbib}
\usepackage[unicode=true,
 bookmarks=false,
 breaklinks=false,pdfborder={0 0 1},backref=false,colorlinks=true]
 {hyperref}

\makeatletter

\providecommand{\tabularnewline}{\\}
\floatstyle{ruled}
\newfloat{algorithm}{tbp}{loa}
\providecommand{\algorithmname}{Algorithm}
\floatname{algorithm}{\protect\algorithmname}

\numberwithin{equation}{section}
\numberwithin{figure}{section}
\theoremstyle{plain}
\newtheorem{thm}{\protect\theoremname}
\theoremstyle{plain}
\newtheorem{prop}[thm]{\protect\propositionname}
\theoremstyle{plain}
\newtheorem{cor}[thm]{\protect\corollaryname}
\theoremstyle{remark}
\newtheorem{rem}[thm]{\protect\remarkname}
\theoremstyle{definition}
\newtheorem{example}[thm]{\protect\examplename}

\usepackage{babel}
\usepackage{diagbox}
\usepackage{adjustbox}\usepackage{placeins}\usepackage{fancybox}
\usepackage{bm}
\usepackage{xcolor}
\usepackage{lscape}
\DeclareMathSizes{8}{8}{6}{4}

\providecommand{\propositionname}{Proposition}
\providecommand{\theoremname}{Theorem}

\providecommand{\corollaryname}{Corollary}
\providecommand{\examplename}{Example}

\usepackage{authblk}

\providecommand{\remarkname}{Remark}

\makeatother

\providecommand{\corollaryname}{Corollary}
\providecommand{\examplename}{Example}
\providecommand{\propositionname}{Proposition}
\providecommand{\remarkname}{Remark}
\providecommand{\theoremname}{Theorem}

\begin{document}
\title{Asymptotic Optimism of Random-Design Linear and Kernel Regression
Models}
\author[1]{Hengrui Luo}
\author[2]{Yunzhang Zhu}
\affil[1]{Department of Statistics, Rice University; Lawrence Berkeley National Laboratory. \url{hrluo@lbl.gov;hrluo@rice.edu}}
\affil[2]{Amazon\footnote{This work does not relate to this author's position at Amazon.}; Department of Statistics, the Ohio State University. \url{ryzhux@gmail.com}}
\date{}

\maketitle
\begin{abstract}
We derived the closed-form asymptotic optimism \citep{ye1998measuring,efron_estimation_2004}
of linear regression models under random designs, and generalizes
it to kernel ridge regression. Using scaled asymptotic optimism as
a generic predictive model complexity measure \citep{luan_predictive_2021},
we studied the fundamental different behaviors of linear regression
model, tangent kernel (NTK) regression model and three-layer fully
connected neural networks (NN). Our contribution is two-fold: we provided
theoretical ground for using scaled optimism as a model predictive
complexity measure; and we show empirically that NN with ReLUs behaves
differently from kernel models under this measure. With resampling
techniques, we can also compute the optimism for regression models
with real data.
\end{abstract}
Keywords: linear regression, kernel ridge regression, generalization
errors, model complexity measure.

\tableofcontents{}

\section{Introduction and Backgrounds}

\subsection{\label{subsec:The-Double-descent-Phenomena}The Double-descent Phenomena}

The double descent phenomenon is an intriguing observation in the
performance of machine learning models, including linear regression,
as their model capacity or complexity is increased \citep{belkin_reconciling_2019,ju2021generalization}.
In the underparameterized region, the model capacity is too low to
capture the underlying patterns in the training data fully. As a result,
both training and testing errors are high. The gap between these errors
(i.e., optimism) may be relatively small because the model isn't complex
enough to exhibit strong overfitting.

At the interpolation threshold point, the model has exactly enough
capacity to fit the training data perfectly, resulting in zero training
error. However, without additional regularization, this perfect fit
can lead to relatively high testing error if the learned patterns
do not generalize. Here, optimism reaches its maximum because the
training error is zero while the testing error is significantly higher
due to overfitting the noise or non-generalizable aspects of the training
set.

As the model complexity increases beyond the interpolation threshold,
according to the double descent curve, the testing error initially
increases and then may decrease again hence enter the overparameterized
region. This counterintuitive phenomenon of double descent, in contrast
to the classical bias-variance trade-off that covers only until interpolation
threshold) implies that further increasing the model's capacity allows
it to learn more generalizable patterns. In the initial part of the
overparameterized region, optimism might increase as the model fits
more noise in training. However, as we move further into the overparameterized
region, and if the double descent phenomenon holds true, the testing
error may decrease, potentially reducing optimism.

\subsection{\label{subsec:Model-Performance-on}Training and Testing Errors}

Motivated by quantifying the double-descent phenomena, recent interests
in describing the model complexity focus in the predictive setting
\citep{luan_predictive_2021,hastie_surprises_2020,rosset2019fixed}.
The calculation of optimism as a predictive model complexity measure
\citep{luan_predictive_2021}, is particularly interesting in the
context of double descent. Initially, as the model complexity increases,
the optimism increases due to overfitting. However, past the critical
point of complexity (somewhere after the interpolation threshold),
increased model capacity could theoretically lead to a more robust
model that generalizes better, thus decreasing optimism again. This
suggests a non-linear relationship between model complexity and optimism,
with a critical peak around the interpolation threshold.

On one hand, this critical understanding challenges traditional views
on model capacity and overfitting, indicating that sometimes \textquotedbl more
is better,\textquotedbl{} even when it seems counterintuitive according
to classical statistical learning theories \citep{belkin_reconciling_2019}.
On the other hand, model complexity is a central topic in statistics.
Popular choices of model complexity include the VC dimension (e.g.,
neural networks (NN), supported vector machines \citep{vapnik1999nature}),
the minimal length principle measures (e.g., encoders, decoders \citep{rissanen2007information})
and the degree of freedom for classical statistical models (e.g.,
linear and ANOVA models \citep{ravishanker2002first}). However, there
is not a well-accepted model complexity measure that can describe
a general model procedure across different types of tasks. Most of
these classical complexity measures focused on the model performance
on the training datasets. Therefore, classical model complexity measures
have difficulty incorporating the model performance on the testing
datasets.

The \emph{training error} describes the in-sample performance of model.
Given a fitted model $\hat{\mu}_{n}$ (e.g., linear regression model),
the well-accepted definition of training error over a training set
sample $\bm{X},\bm{y}$ of size $n$ is (e.g., (2.1) in \citet{luan_predictive_2021}):
$\text{}\frac{1}{n}\sum_{i=1}^{n}\ell(y_{i},\hat{\mu}_{n}(\bm{x}_{i}))\overset{\ell=L_{2}}{=}\sum_{i=1}^{n}\frac{1}{n}\left(y_{i}-\hat{\mu}_{n}(\bm{x}_{i})\right)^{2}$
which is the loss we use in the rest of the paper. We denote the fitted
mean model $\hat{\mu}_{n}=\hat{\mu}$ (for notational brevity) based
on the sample of size $n$ and $\ell$ denotes the loss function of
our choice. We fit the model by minimizing the training error with
optimization algorithms. The fitted model function $\hat{\mu}_{n}$
can be written into vector form $\hat{\bm{\mu}}=\left(\hat{\mu}_{n}(\bm{x}_{1}),\cdots,\hat{\mu}_{n}(\bm{x}_{n})\right)\in\mathbb{R}^{n}$
depends on the training set input $\bm{X}=\{\bm{x}_{1},\bm{x}_{2},\cdots,\bm{x}_{n}\},\bm{x}_{i}\in\mathbb{R}^{d}$
and response $\bm{y}=\{y_{1},y_{2},\cdots,y_{n}\},y_{i}\in\mathbb{R}^{1}$.
The notation $T_{\bm{X}}$ explicitly reminds us that the training
error depends on the training set $\bm{X}$ (and $\bm{y}$). 
\begin{equation}
\text{\text{Err}\ensuremath{T_{\bm{X}}}}\coloneqq\mathbb{E}_{\bm{y}}\mathbb{E}_{\bm{X},\bm{y}}\left\Vert \bm{y}-\hat{\mu}_{n}(\bm{X})\right\Vert ^{2}\approx\frac{1}{N}\sum_{\text{\ensuremath{\bm{y}} conditioned on }\bm{X}}\frac{1}{n}\sum_{i=1}^{n}\ell(y_{i},\hat{\mu}_{n}(\bm{x}_{i}))\,.\label{eq:apprx_ErrTX}
\end{equation}
This notation means that when we assume that the response $\bm{y}$
is random given the input $\bm{X}$, the average training error can
be described by $\text{Err}T_{\bm{X}}$. The summation in \eqref{eq:apprx_ErrTX}
means that we fix $\bm{X}$ and simulate $N$ different $\bm{y}$'s
and summing over these $N$ pairs of $\bm{X},y$.

The \emph{testing error} describes the out-sample predictive performance
of model, and it depends on both the training input $\bm{X}$ and
response $\bm{y}$. Unlike the well-accepted notion of training error
\eqref{eq:apprx_ErrTX}, \citet{rosset2019fixed} discussed three
different kinds of settings where model testing errors can be computed. 
\begin{itemize}
\item The fixed-X setting \citep{efron_estimation_2004}. The testing and
training set share the same input locations ($\bm{X}$ is nonrandom),
yet the response in testing set is regenerated to reflect the randomness
in response. 
\begin{align}
\text{Err}F_{\bm{X},\bm{y}} & \coloneqq\mathbb{E}_{\tilde{\bm{y}}\mid\bm{X},\bm{y}}\frac{1}{n}\sum_{i=1}^{n}\left\Vert \tilde{y}(\bm{x}_{i})-\hat{\mu}(\bm{x}_{i})\right\Vert _{2}^{2}.\label{eq:ErrFX}
\end{align}
The notation $\tilde{\bm{y}}=\{\tilde{y}(\bm{x}_{1}),\tilde{y}(\bm{x}_{2}),\cdots,\tilde{y}(\bm{x}_{n})\}$,
where each $\tilde{y}(\bm{x}_{i})$ is an independent copy of $y_{i}$
with the same distribution, corresponding to the input location $\bm{x}_{i}$.
The notation $\mathbb{E}_{\tilde{\bm{y}}\mid\bm{X},\bm{y}}$ means
that we take conditional expectation on $\tilde{\bm{y}}$ conditioning
on $\bm{X},\bm{y}$. 
\item The same-X setting \citep{rosset2019fixed}. The testing and training
set share the same input location distribution ($\bm{X}$ is random),
and the response in testing set is independently regenerated to reflect
the randomness in response. The same-X prediction error can be written
as: 
\begin{align}
\text{Err}S & \coloneqq\mathbb{E}_{\tilde{\bm{y}},\bm{X},\bm{y}}\frac{1}{n}\sum_{i=1}^{n}\left\Vert \tilde{y}(\bm{x}_{i})-\hat{\mu}(\bm{x}_{i})\right\Vert _{2}^{2}\label{eq:ErrS}\\
 & =\mathbb{E}_{\tilde{\bm{y}},\bm{X},\bm{y}}\left\Vert \tilde{y}(\bm{x}_{1})-\hat{\mu}(\bm{x}_{1})\right\Vert _{2}^{2}.\nonumber 
\end{align}
In this setting, the error $\text{Err}S$ does not investigate any
new input locations, but assume that the input locations are randomly
drawn. Unlike \eqref{eq:ErrFX}, \eqref{eq:ErrS} does not depend
on the input location $\bm{X}$ in the training set, because the notation
$\mathbb{E}_{\tilde{\bm{y}},\bm{X},\bm{y}}$ means that we take joint
expectation on $\bm{X},\bm{y},\tilde{\bm{y}}$ jointly and get rid
of the dependence on $\bm{X},\bm{y},\tilde{\bm{y}}$. 
\item The random-X setting \citep{luan_predictive_2021}. The testing and
training set may have different input location distributions ($\bm{X}$
is random), and the response in testing set is independently regenerated
to reflect the randomness in response. The random-X prediction error
can be written as: 
\begin{align}
\text{Err}R_{\bm{X}} & \coloneqq\mathbb{E}_{\bm{y}}\mathbb{E}_{\bm{x}_{*},\bm{y}_{*}\mid\bm{X},\bm{y}}\frac{1}{n}\sum_{i=1}^{n}\left\Vert y_{i,*}-\hat{\mu}(\bm{x}_{i,*})\right\Vert _{2}^{2}\label{eq:ErrRX}
\end{align}
\begin{align}
\text{Err}R_{\bm{X}} & \approx\frac{1}{N}\sum_{\text{\ensuremath{\bm{y}} conditioned on }\bm{X}}\mathbb{E}_{\bm{x}_{*},\bm{y}_{*}\mid\bm{X},\bm{y}}\frac{1}{n}\sum_{i=1}^{n}\ell(y_{i,*},\hat{\mu}(\bm{x}_{i,*}))\,.\label{eq:apprx_ErrRX}
\end{align}
In this setting, the error $\text{Err}R_{\bm{X}}$ investigates the
input locations where the model $\hat{\mu}$ any new input locations,
but assume that the input locations are fixed. 
\end{itemize}
\citet{rosset2019fixed} pointed out that the testing error in random-X
setting would be more appropriate for assessing model performance
on the testing set. \emph{We would focus only on the training error
\eqref{eq:apprx_ErrTX} and the testing error \eqref{eq:ErrRX}.}
We will also use the term prediction location $\bm{x}_{*}$, which
can be considered as a one-point testing set.

\subsection{\label{subsec:Linear-Regression-Model}Linear Regression Model}

In a linear regression model, whose model complexity (or capacity)
is well accepted as the number of features used in the model. The
double descent curve comprises three distinct regions: underparameterized,
interpolation threshold, and overparameterized. We consider a dataset
$\mathcal{D}=\{(\bm{x}_{i},y_{i})\}_{i=1}^{n}$ with $\bm{x}_{i}\in\mathbb{R}^{d\times1}$
and $y_{i}\in\mathbb{R}^{1}$ and the goal is to learn a function
$f:\mathbb{R}^{d}\rightarrow\mathbb{R}$ that approximates the relationship
between $\bm{x}_{i}$ and $y_{i}$ in form of 
\begin{equation}
f(\bm{x};\bm{\beta})=\bm{x}^{T}\bm{\beta},\label{eq:linear_model}
\end{equation}
where $\bm{\beta}\in\mathbb{R}^{d\times1}$ is the coefficient vector
to be learned. This setup covers both linear regression with and without
intercepts, since we can fix the last element of $\bm{x}$ to be a
deterministic constant and consider the distribution of $\bm{x}$
is degenerated at that location. We can optimize $\bm{\beta}$ by
minimizing the mean squared error (MSE, $L_{2}$-loss) on the training
data:

\begin{equation}
\hat{\bm{\beta}}=\arg\min_{\bm{\beta}}\frac{1}{n}\sum_{i=1}^{n}(y_{i}-f(\bm{x}_{i};\bm{\beta}))^{2}.
\end{equation}
The solution to the problem can be written as $\hat{\bm{\beta}}=\left(\bm{X}^{T}\bm{X}\right)^{-1}\bm{X}^{T}\bm{y}$
where $\bm{X}\in\mathbb{R}^{n\times d}$ is the matrix obtained by
stacking rows of $\bm{x}_{i}$'s in the training set. The classical
degree of freedom of the matrix form $\bm{y}=\bm{X}\bm{\beta}$ of
model \eqref{eq:linear_model} is defined as $tr(\bm{H}),\bm{H}=\bm{X}\left(\bm{X}^{T}\bm{X}\right)^{-1}\bm{X}^{T}$
for the linear regression model considers the in-sample error as a
model complexity measure. However, the U-shape with respect to $tr(\bm{H})$
may not exist when we consider modern machine learning models \citep{belkin_reconciling_2019,luan_predictive_2021}.
Then, assuming $d<n$, we can sequentially increase the number $d$
of parameters (i.e., regression coefficients) to get a better fit
in the sense that the smallest $L_{2}$ loss keeps decreasing. Once
we attain $d=n$, the linear regression model becomes \emph{saturated},
the $L_{2}$ loss would not decrease further.

The above single-descent intuition based on the bias-variance trade-off
tells us: a NN with moderate number of nodes (and layers) is preferred.
This view is natural at first until \citet{belkin_reconciling_2019}
pointed out that for a deep neural network (DNN, i.e., the network
architecture with a lot of nodes and layers), there occurs a \emph{double-descent
phenomena}.

When we plot the loss function against the model complexity measure
of NN. The model complexity measure for NN is chosen to be the number
of nodes and layers in the network architecture (i.e., the number
of hidden units). Then we would observe the U-shape curve, followed
by another descending curve after reaching the second peak. This is
known as the \emph{double-descent} phenomena for the loss function
in the NN setting \citep{belkin_reconciling_2019,neal2018modern}.
Although the training procedure (i.e., fitting the model by minimizing
the data-dependent loss function) remains the same, the traditional
bias-variance trade-off on the loss function does not hold for the
NN, otherwise we would expect the single-descent instead of the double-descent.
In linear regression models, \citet{hastie_surprises_2020} pointed
out that when the training dataset is not fixed, in an asymptotic
setting, the double-descent phenomena even exists for linear regression
models, motivating us to study the optimism using linear regression
model.

Linear regression models \citep{ravishanker2002first} is fitted by
minimizing the $L_{2}$ loss function with respect to the regression
coefficients $\bm{\beta}\in\mathbb{R}^{d\times1}$. In the matrix
form $\bm{X}\in\mathbb{R}^{n\times d}$, each instance is represented
by a vector $\bm{x}_{i}\in\mathbb{R}^{d\times1}$, $\bm{X}=\left[\begin{array}{c}
\bm{x}_{1}^{T}\\
\vdots\\
\bm{x}_{n}^{T}
\end{array}\right]$ represents a location; $\bm{y}=\left(y_{1},\cdots,y_{n}\right)^{T}\in\mathbb{R}^{n\times1}$,
each row is a scalar response. We want to use $n$ $d$-dimensional
inputs $\bm{X}\in\mathbb{R}^{n\times d}$ to predict the response
$\bm{y}\in\mathbb{R}^{n\times1}$. The (Gaussian) \emph{linear regression}
model (without intercept) can be written as below.

\begin{align}
\bm{y} & =\bm{X}\bm{\beta}+\bm{\epsilon},\bm{\beta}\in\mathbb{R}^{n},\nonumber \\
\bm{\epsilon} & \sim N_{n}(\bm{0},\sigma_{\epsilon}^{2}\bm{I}_{n})\in\mathbb{R}^{n},\label{eq:LM no intercept}
\end{align}
where $\bm{X}\bm{\beta}$ describes a linear relationship (or dependence)
between input $\bm{X}$ and response $y$, the random variable $\bm{\epsilon}$
picks up the potential Gaussian noise in observations. Therefore,
we can write the model as 
\begin{align}
\bm{y}(\bm{x}) & \sim N_{n}(\bm{x}^{T}\bm{\beta},\sigma_{\epsilon}^{2}\bm{I}_{n})\in\mathbb{R}^{n},\nonumber \\
\mu(\bm{X}) & =\bm{X}\beta.
\end{align}
To fit the linear model we consider the loss function 
\begin{equation}
\ell=L_{2}(\bm{\beta};\bm{X},\bm{y})\coloneqq\left\Vert \bm{y}-\bm{X}\bm{\beta}\right\Vert _{2}^{2}=\left(\bm{y}-\bm{X}\bm{\beta}\right)^{T}\left(\bm{y}-\bm{X}\bm{\beta}\right),
\end{equation}
and suppose that $\bm{X}$ is of full rank in the discussion below
for simplicity. By taking the matrix gradient $\frac{\partial}{\partial\bm{\beta}}L(\bm{\beta};\bm{X},\bm{y})$
to be zero, we can solve for a minimizer $\hat{\bm{\beta}}=\left(\bm{X}^{T}\bm{X}\right)^{-1}\bm{X}^{T}\bm{y}$,
which is known as the \emph{least square estimator}. Our model estimate
at observed locations $\bm{X}$ is 
\begin{align*}
\hat{\bm{\mu}}(\bm{X})=(\hat{\mu}(\bm{x}_{1}),\hat{\mu}(\bm{x}_{2}),\cdots,\hat{\mu}(\bm{x}_{n}))^{T}=(\bm{x}_{1}^{T}\hat{\bm{\beta}},\bm{x}_{2}^{T}\hat{\bm{\beta}},\cdots,\bm{x}_{n}^{T}\hat{\bm{\beta}})^{T}=\bm{X}\hat{\bm{\beta}}=\bm{H}\bm{y}
\end{align*}
with a hat matrix $\bm{H}=\bm{X}\left(\bm{X}^{T}\bm{X}\right)^{-1}\bm{X}^{T}$.

For a single prediction location $\bm{x}_{*}$, we use the following
notations $\bm{h}_{i}^{T}=\bm{x}_{i}^{T}\left(\bm{X}^{T}\bm{X}\right)^{-1}\bm{X}^{T}$,
$\bm{H}=\left[\begin{array}{c}
\bm{h}_{1}^{T}\\
\vdots\\
\bm{h}_{n}^{T}
\end{array}\right]$ and $\bm{h}_{*}^{T}=\bm{x}_{*}^{T}\left(\bm{X}^{T}\bm{X}\right)^{-1}\bm{X}^{T}$.
The prediction mean is $\hat{\mu}(\bm{x}_{*})=\bm{x}_{*}^{T}\hat{\bm{\beta}}=\bm{x}_{*}^{T}\left(\bm{X}^{T}\bm{X}\right)^{-1}\bm{X}^{T}\bm{y}=\bm{h}_{*}^{T}\bm{y}$.
The prediction error at a new input $\bm{x}_{*}$ can be written as:
\begin{align}
 & \mathbb{E}_{\bm{y}\mid\bm{X},\bm{x}_{*}}\left\Vert y_{*}(\bm{x}_{*})-\hat{\mu}(\bm{x}_{*})\right\Vert _{2}^{2}=\mathbb{E}_{\bm{y}\mid\bm{X},\bm{x}_{*}}\left\Vert \bm{x}_{*}^{T}\bm{\beta}+\epsilon_{*}-\bm{x}_{*}^{T}\hat{\bm{\beta}}\right\Vert _{2}^{2}\\
 & =\sigma_{\epsilon}^{2}+\mathbb{E}\left(\bm{x}_{*}^{T}\bm{\beta}-\mathbb{E}\bm{x}_{*}^{T}\hat{\bm{\beta}}+\mathbb{E}\bm{x}_{*}^{T}\hat{\bm{\beta}}-\bm{x}_{*}^{T}\hat{\bm{\beta}}\right)^{T}\left(\bm{x}_{*}^{T}\bm{\beta}-\mathbb{E}\bm{x}_{*}^{T}\hat{\bm{\beta}}+\mathbb{E}\bm{x}_{*}^{T}\hat{\bm{\beta}}-\bm{x}_{*}^{T}\hat{\bm{\beta}}\right)\nonumber \\
 & =\begin{array}{c}
\sigma_{\epsilon}^{2}\\
\text{\text{noise var.}}
\end{array}+\begin{array}{c}
\underbrace{\mathbb{E}\left(\bm{x}_{*}^{T}\bm{\beta}-\mathbb{E}\bm{x}_{*}^{T}\hat{\bm{\beta}}\right)^{T}\left(\bm{x}_{*}^{T}\bm{\beta}-\mathbb{E}\bm{x}_{*}^{T}\hat{\bm{\beta}}\right)}\\
\text{(square) bias of estimator \ensuremath{\ensuremath{\hat{\mu}}(\ensuremath{\bm{x}_{*}})}=\ensuremath{\bm{x}_{*}^{T}\hat{\bm{\beta}}}}
\end{array}+\nonumber \\
 & \begin{array}{c}
\underbrace{\mathbb{E}\left(\bm{x}_{*}^{T}\hat{\bm{\beta}}-\mathbb{E}\bm{x}_{*}^{T}\hat{\bm{\beta}}\right)^{T}\left(\bm{x}_{*}^{T}\hat{\bm{\beta}}-\mathbb{E}\bm{x}_{*}^{T}\hat{\bm{\beta}}\right)}\\
\text{\text{variance of estimator \ensuremath{\hat{\mu}}(\ensuremath{\bm{x}_{*}})=\ensuremath{\bm{x}_{*}^{T}\hat{\bm{\beta}}}}}
\end{array},\label{eq:bias-variance}
\end{align}
which induces the \emph{bias-variance trade-off}. The notation $\mathbb{E}_{\bm{y}\mid\bm{X},\bm{x}_{*}}$
means that we take the expectation with respect to response $\bm{y}$
given the observed locations $\bm{X}$ and the prediction location
$\bm{x}_{*}$, where we assume that the conditional distribution of
$\bm{y}\mid\bm{X},\bm{x}_{*}$ is known. This decomposition holds
for other settings as shown in \citet{rosset2019fixed} (i.e., $B^{+}$
and $V^{+}$ in their notations).

With the above decomposition of the expected loss function, if we
plot the $L_{2}$ loss of the fitted linear regression model (as y-axis)
against the degree of freedom $tr(\bm{H})$ as model complexity measure
(as the x-axis), then the loss function can be decomposed into bias
and the variance components. This exhibits the U-shape curve discussed
by multiple authors in classical regression setting \citep{friedman2017elements,neal2018modern}.
When there are few parameters (i.e., small $p$), the predictive variance
is relatively large; when there are too many parameters (i.e., large
$p$), the bias is relatively large.

After revisiting the linear models and testing training errors, to
reconcile the seemingly dilemma, we investigate the notion of optimism
in section \ref{sec:Optimism-Measures-of} and link it to predictive
model complexity measure. Detailed examples and our main results concerning
linear models are presented in sections \ref{subsec:Example-2:-Nonlinear},
followed by discussions in section \ref{sec:Contribution-and-Discussion}.

\section{\label{sec:Optimism-Measures-of}Optimism Measures of Model Complexity}

After identifying the first descent phenomena caused by the variance-bias
trade-off \citep{belkin_reconciling_2019}, it is unclear why the
second descent occurs in complex models like NN. One thinking (which
we would take) is that the x-axis of the prediction error against
complexity plot uses an incorrect choice of complexity measure; while
the others are suspicious in the robustness of an over-fitting model
\citep{jordan2015machine,ju2021generalization,ju2020overfitting}.
In essence, the prediction error against complexity plot should be
replaced with prediction error against a corrected version of ``predictive
complexity'' \citep{luan_predictive_2021,patil2024revisiting}.

An adjusted complexity measure, namely the \emph{optimism} \citep{efron_estimation_2004}
of the model, can be elicited as the difference between training and
testing errors. When a model is trained on one training set that is
different from the testing set where the model predicts, the optimism
would tend to be larger (in both offline and online scenarios \citep{luo2024hybrid}).
Extending the idea of using optimism, \citet{luan_predictive_2021}
propose to adopt part of the optimism as complexity measures, namely
the predictive complexity.

The first advantage of using optimism as a model complexity measure
is that it not only reflects the goodness-of-fit of the model but
also reflects the generalizability of the model from training to testing
datasets \citep{wang2024degrees}. In addition, a scaled optimism
can be shown to agree with the classical degree of freedom when we
consider the linear regression model \citep{ye1998measuring}. Therefore,
we could benefit from intuitions established in classic modeling contexts
where the number of parameters are used for measuring model complexity.

The second advantage of using model optimism is that it can be computed
via Monte-Carlo (MC) method since both \ref{eq:apprx_ErrTX} and \ref{eq:apprx_ErrRX}
can be approximated by definition, (See Algorithm \ref{alg:Simulation-algorithm-for})
for almost all predictive models without much assumption on the explicit
model forms. This allows us to define complexity descriptors for black-box
models like NN. We expect that this could be a more faithful model
complexity measure. Precisely, we have following proposition that
defines the optimism and we can have its closed form expression. 
\begin{prop}
\label{prop:Let}(Optimism in linear regression) The optimism (i.e.,
random-X prediction error \eqref{eq:ErrRX} minus averaged training
error \eqref{eq:apprx_ErrTX}) can be defined and computed as below
(e.g., (3.2) in \citet{luan_predictive_2021}): 
\begin{align}
\text{Opt }R_{\bm{X}} & \coloneqq\text{Err}R_{\bm{X}}-\text{Err}T_{\bm{X}}\\
 & =\mathbb{E}_{\bm{x}_{*}}\left\Vert \mu(\bm{x}_{*})-\bm{h}_{*}\mu(\bm{X})\right\Vert _{2}^{2}-\frac{1}{n}\left\Vert \bm{\mu}(\bm{X})-\bm{H}\bm{\mu}(\bm{X})\right\Vert _{2}^{2}\nonumber \\
 & +\sigma_{\epsilon}^{2}\left(\mathbb{E}_{\bm{x}_{*}}\|\bm{h}_{*}^{T}\|_{2}^{2}-\frac{1}{n}\text{trace \ensuremath{\left(\bm{H}^{T}\bm{H}\right)}}+\frac{1}{n}\text{trace}\left(2\bm{H}\right)\right).\label{eq:def opt}
\end{align}
\end{prop}

\begin{proof}
See Appendix \ref{sec:Proof-for-Proposition-1}. 
\end{proof}
In \eqref{eq:def opt}, the second line is $\Delta B_{\bm{X}}$ and
the third line is exactly (3.3) in \citet{luan_predictive_2021}.
Optimism is widely used as a complexity measure in modeling context
\citep{efron_estimation_2004,hastie_surprises_2020}, in addition,
\citet{ye1998measuring} showed that a scaled version of optimism
coincides with the degree of freedom. To show this fact, we want to
use the quantity in $\text{Opt }R_{\bm{X}}$ that is inside the last
bracket after $\sigma_{\epsilon}^{2}$ in \eqref{eq:def opt}. Specifically,
when $\bm{x}_{*}=\bm{X}$, we can cancel the first two terms and have
following expression, which is independent of signal $\mu$: 
\begin{align}
\text{Opt }R_{\bm{X}} & =\sigma_{\epsilon}^{2}\left(\|\bm{H}\|_{2}^{2}-\frac{1}{n}\text{trace \ensuremath{\left(\bm{H}^{T}\bm{H}\right)}}+\frac{1}{n}\text{trace}\left(2\bm{H}\right)\right)\\
 & =\sigma_{\epsilon}^{2}\left(\|\bm{H}\|_{2}^{2}-\frac{1}{n}\text{trace}\left(\bm{H}\right)\right).
\end{align}
This is the closed form expression when the model fitting procedure
can be described as a linear projection method with a certain choice
of basis functions (e.g., polynomial regression, B-splines \citep{gu2013smoothing}).
The optimism is related to GDF \citep{ye1998measuring}, Malow's $C_{p}$
and other complexity measures \citep{efron_estimation_2004}. In \eqref{eq:def opt},
we separate the expression into ``signal part'' involving $\mu$;
and the ``noise part'' involving $\sigma_{\epsilon}^{2}$. This
separation is different from equations (3), (4) and (5) in \citet{rosset2019fixed}
even without the $\text{Err}T_{\bm{X}}$.

\citet{luan_predictive_2021,hastie_surprises_2020} considered to
approximate the signal part using a leave-one-out cross validation
(LOOCV) technique with some adjustment. The LOOCV estimation is supported
by the numerical evidence when the training set $\bm{X}$ is fixed.
The original study focused on the estimation when the signal is fixed,
in our study below we derived asymptotic exact formula, showing how
this term depends on the signal.

We investigate the setting where the training set $\bm{X}$ is assumed
to be random and drawn from a distribution, as in $\text{Opt }R_{\bm{X}}$.
Unlike \citet{rosset2019fixed}(e.g., their Theorem 3), we do not
assume the model is correctly-specified and focus on the impact of
the actual signal on the behavior of optimism. Arguably, it is more
often than not that the model is not an unbiased estimate to the signal
in reality, our technical calculation can be extended to more general
models like linear smoothers at the cost of more complex notations.

Next, we would show that the optimism is signal-dependent, which is
different from the predictive model complexity measure \citep{luan_predictive_2021}.
That means, if the underlying data generating mechanism changes, then
the model complexity measure for a fitted model would also change.
A signal-independent model complexity measure could be defined through
applying the modeling procedure to white noise and compare the complexity
under white noise and nontrivial signals (e.g., the difference between
model optimism and white noise optimism).

\section{\label{subsec:Example-2:-Nonlinear}Optimism for Linear Regression
Model}

\subsection{Theoretical Results}

In the previous section, we have discussed the possible effect of
signal when we try to measure the model complexity in predictive setting.
In this section, we presume formally that we fit regression models
for a training dataset $(\bm{X},\bm{y})$ consisting of i.i.d. pairs
of input and responses and a testing dataset $(\bm{x}_{*,},y_{*})$
consisting of i.i.d. pairs of input and responses. Both of the rows
of training set $\bm{X}=\bm{X}_{n}\in\mathbb{R}^{n\times d}$ and
a new location in testing set $\bm{x}_{*}\in\mathbb{R}^{d}$ share
the same distribution (e.g., $N(\bm{0},\sigma^{2}\bm{I})$). Based
on the definitions of \eqref{eq:apprx_ErrTX} and \eqref{eq:ErrRX},
we can obtain the intuition that

\begin{align}
\text{Opt }R_{\bm{X}} & \coloneqq\text{Err}R_{\bm{X}}-\text{Err}T_{\bm{X}}\overset{\text{typically}}{\geq}0\label{eq:optR_X_ex2}
\end{align}
Although the above derivation focused on the $L_{2}$ loss function
in linear regression, the positivity holds in general model fitting
procedures. We prove this fact as a proposition below. 
\begin{prop}
\label{prop:(Positivity)-The-testing}(Positivity) The testing error
$\text{Err}R_{\bm{X}}$ is greater than the training error $\text{Err}T_{\bm{X}}$
for a loss function minimization procedure, therefore, the optimism
$\text{Opt }R_{\bm{X}}\geq0$. The trained model $\hat{\mu}_{\text{train}}$
is defined in the same functional space $\mathcal{F}_{n}$, which
is independent of $\{\bm{x}_{i},y_{i}\}_{i=1}^{n}$ and $\{\bm{x}_{*,i},y_{*,i}\}_{i=1}^{n}$
but may depend on sample size $n$: 
\begin{align}
\hat{\mu}_{\text{train}} & =\arg\min_{f\in\mathcal{F}_{n}}T_{\bm{X}}=\arg\min_{f\in\mathcal{F}_{n}}\frac{1}{n}\sum_{i=1}^{n}\ell(f(\bm{x}_{i}),y_{i}),\label{eq:train_lossL2}
\end{align}
For the optimism defined for $\hat{\mu}_{train}$ we have $\mathbb{E}_{\bm{X}}\text{Opt }R_{\bm{X}}\geq0$. 
\end{prop}

\begin{proof}
See Appendix \ref{sec:Proof-of-the-main-THM}. 
\end{proof}
For more complicated regression functions like NN shown in Figure
\ref{fig:Opt_vs_epoch} below, the loss landscape could be more complicated.
There can be more than one stationary points on the corresponding
loss landscape, where the NN may not converge to the minimizer stationary
point. Therefore, the resulting fitted regression model may not be
an interpolator, and the step \eqref{eq:proof_train2test} in our
proof cannot proceed. The mis-specification can arise from wrong smoothness
(e.g., sigmoid) or hard misfit, in both situations the signal does
not live in the space spanned by the activation functions (e.g., ReLU).
Unfortunately, the correct activation (as a basis of interpolation,
e.g., linear) is not known to the practitioner.

The following theorem gives the asymptotic formula for scaled optimism
up to $O_{p}\left(\frac{1}{\sqrt{n}}\right)$ for linear regression
models with intercept, and for linear models without intercept \eqref{eq:LM no intercept},
it remains the same except that we need to assume that the design
matrix $\bm{X}$ has one fixed constant column consisting of 1's and
a $\bm{\Sigma}$ with the corresponding diagonal element degenerated
as 0.

\textbf{Assumptions A1. } Let $\hat{\bm{\eta}}=\frac{1}{n}\bm{X}^{T}\bm{y}(\bm{X})=\frac{1}{n}\bm{X}^{T}\bm{y}$
and $\hat{\bm{\Sigma}}=\frac{1}{n}\bm{X}^{T}\bm{X}$. We assume that
\begin{equation}
\left\Vert \hat{\bm{\eta}}-\bm{\eta}\right\Vert _{2}=O_{p}\left(\frac{1}{\sqrt{n}}\right),\left\Vert \hat{\bm{\Sigma}}-\bm{\Sigma}\right\Vert _{2}=O_{p}\left(\frac{1}{\sqrt{n}}\right)\label{eq:assumption_A1}
\end{equation}
where $\bm{\eta}=\mathbb{E}_{\bm{x}_{*}}\bm{x}_{*}y(\bm{x}_{*})=\mathbb{E}_{\bm{x}_{*}}\bm{x}_{*}y_{*}$
and $\bm{\Sigma}=\mathbb{E}(\bm{x}_{*}\bm{x}_{*}^{T})$. 
\begin{thm}
\label{thm:main theorem} Under Assumption A1, we can write down the
errors as

\begin{align*}
\mathbb{E}_{\bm{X}}\text{Err}R_{\bm{X}} & =\mathbb{E}_{\bm{x}_{*}}\left\Vert y_{*}-\bm{x}_{*}^{T}\hat{\bm{\beta}}\right\Vert _{2}^{2}\\
 & =\mathbb{E}_{\bm{x}_{*}}\left(y_{*}-\bm{x}_{*}^{T}\bm{\Sigma}^{-1}\bm{\eta}\right)^{2}\\
 & +\frac{1}{n}\mathbb{E}_{\bm{x}_{*}}\left(y_{*}-\bm{x}_{*}^{T}\bm{\Sigma}^{-1}\bm{\eta}\right)^{2}\left(\bm{x}_{*}^{T}\bm{\Sigma}^{-1}\bm{x}_{*}\right)+O_{p}\left(\frac{1}{n^{3/2}}\right).\\
\mathbb{E}_{\bm{X}}\text{Err}T_{\bm{X}} & =\frac{1}{n}\mathbb{E}_{\bm{X}}\left\Vert \bm{y}-\bm{X}\hat{\bm{\beta}}\right\Vert _{2}^{2}\\
 & =\mathbb{E}_{\bm{x}_{*}}\left(y_{*}-\bm{x}_{*}^{T}\bm{\Sigma}^{-1}\bm{\eta}\right)^{2}\\
 & -\frac{1}{n}\mathbb{E}_{\bm{x}_{*}}\left(y_{*}-\bm{x}_{*}^{T}\bm{\Sigma}^{-1}\bm{\eta}\right)^{2}\left(\bm{x}_{*}^{T}\bm{\Sigma}^{-1}\bm{x}_{*}\right)+O_{p}\left(\frac{1}{n^{3/2}}\right).
\end{align*}
The expected random optimism for the least squares estimator is 
\begin{align}
\mathbb{E}_{\bm{X}}\text{Opt }R_{\bm{X}}=\frac{2}{n}\mathbb{E}_{\bm{X}}\left[\mathbb{E}_{\bm{x}_{*}}\left\Vert y_{*}-\bm{x}_{*}^{T}\bm{\Sigma}^{-1}\bm{\eta}\right\Vert _{2}^{2}\left(\bm{x}_{*}^{T}\bm{\Sigma}^{-1}\bm{x}_{*}\right)\right]+O_{p}\left(\frac{1}{n^{3/2}}\right) & .\label{eq:general_opt_expression}
\end{align}
\end{thm}

\begin{proof}
See Appendix \ref{sec:New-Proof}; 
\end{proof}
We will investigate next set of results about the scenario when the
model is a perfect fit of the signal, the signal-dependent term vanishes. 
\begin{cor}
\label{cor:Under-the-positivity}Under the same assumptions of Theorem
\ref{thm:main theorem}, the term 
\[
\mathbb{E}_{\bm{X}}\left[\left\Vert y_{*}-\bm{x}_{*}^{T}\bm{\Sigma}^{-1}\bm{\eta}\right\Vert _{2}^{2}\left(\bm{x}_{*}^{T}\bm{\Sigma}^{-1}\bm{x}_{*}\right)\right]
\]
in \eqref{eq:general_opt_expression} attains zero if and only if
the function $\bm{\mu}(\bm{x})=\bm{x}^{T}\bm{\beta}$ for some $\bm{\beta}\in\mathbb{R}^{d+1}$. 
\end{cor}

\begin{proof}
From the \eqref{eq:general_opt_expression}, a non-negative random
variable $\left\Vert \bm{\Sigma}^{-1/2}\bm{x}_{*}\right\Vert _{2}^{2}>0$
unless $\bm{x}_{*}=\bm{0}$ due to the positive definiteness of the
$\bm{\Sigma}$. Therefore $\bm{x}_{*}^{T}\bm{\Sigma}^{-1}\bm{\eta}-\bm{\mu}(\bm{x}_{*})\equiv\bm{0}\Leftrightarrow\bm{\mu}(\bm{x}_{*})=\bm{x}_{*}^{T}\bm{\Sigma}^{-1}\bm{\eta}$
which makes $\bm{\mu}$ a linear function in $\bm{x}_{*}$ with coefficient
$\bm{\beta}=\bm{\Sigma}^{-1}\bm{\eta}$. And this also makes the second
term to be zero. 
\end{proof}
\begin{cor}
\label{cor:When-reducing-to-classical}When $\bm{X}_{i},\bm{x}_{*i}\sim N(\bm{0},\bm{\Sigma})$
and $y(\bm{x})=m(\bm{x})+\bm{\epsilon}$ for an additive independent
noise $\bm{\epsilon}\sim N(0,\sigma_{\epsilon}^{2})$ with $\sigma_{\epsilon}^{2}>0$,
we can yield formula \eqref{eq:general_opt_expression} and write
the expected scaled optimism as 
\begin{align}
\frac{n\mathbb{E}_{\bm{X}}\text{Opt }R_{\bm{X}}}{2\sigma_{\epsilon}^{2}}\sim & \frac{1}{\sigma_{\epsilon}^{2}}\mathbb{E}_{\bm{x}_{*}}\left[\left\Vert m(\bm{x}_{*})-\bm{x}_{*}^{T}\bm{\Sigma}^{-1}\bm{\eta}\right\Vert _{2}^{2}\left\Vert \bm{\Sigma}^{-1/2}\bm{x}_{*}\right\Vert _{2}^{2}\right]+d+O_{p}\left(\frac{1}{n^{1/2}}\right)\label{eq:Gaussian_opt_expression_Stein}
\end{align}
\end{cor}

\begin{proof}
See Appendix \ref{sec:Proof-of-Corollary-reducing}. 
\end{proof}
\begin{rem}
\label{rem:If-Stein}If more generally $\bm{X}_{i},\bm{x}_{*i}\sim N(\bm{\mu},\bm{\Sigma})$,
we can yield multivariate Stein's lemma to simplify the term $\bm{\Sigma}^{-1}\bm{\eta}=\left[\mathbb{E}_{\bm{X}}\left(\bm{X}^{T}\bm{X}\right)\right]^{-1}\left[\mathbb{E}_{\bm{X}}\bm{X}\bm{y}\right]$
in \eqref{eq:Gaussian_opt_expression_Stein} via when the $m$ is
continuously differentiable. Assuming this, we observe that 
\begin{align*}
\bm{\Sigma}^{-1}\bm{\eta} & =\left[\mathbb{E}_{\bm{X}}\left(\bm{X}^{T}\bm{X}\right)\right]^{-1}\cdot\left[\mathbb{E}_{\bm{X}}\bm{X}\left(m(\bm{X})+\bm{\epsilon}\right)\right]\\
 & =\left[\mathbb{E}(\bm{X}\bm{X}^{T})\right]{}^{-1}\mathbb{E}\left[\bm{X}m(\bm{X})\right].
\end{align*}
Then we can derive that $\mathbb{E}\bm{X}m(\bm{X})=\mathbb{E}(\bm{X}-\bm{\mu})m(X)+\bm{\mu}\mathbb{E}m(\bm{X})=\bm{\Sigma}\mathbb{E}[\nabla m(\bm{X})]+\bm{\mu}\mathbb{E}m(\bm{X})$.
By Woodbury lemma, $(\bm{\Sigma}+\bm{\mu}\bm{\mu}^{T})^{-1}=\bm{\Sigma}^{-1}-\bm{\Sigma}^{-1}\bm{\mu}\bm{\mu}^{T}\bm{\Sigma}^{-1}/(1+\bm{\mu}^{T}\bm{\Sigma}^{-1}\bm{\mu})$,
we have 
\begin{align}
 & \left[\mathbb{E}(\bm{X}\bm{X}^{T})\right]{}^{-1}\mathbb{E}\left[\bm{X}m(\bm{X})\right]\\
 & =\left(\bm{I}-\left(1+\bm{\mu}^{T}\bm{\Sigma}^{-1}\bm{\mu}\right)^{-1}\bm{\Sigma}^{-1}\bm{\mu}\bm{\mu}^{T}\right)\mathbb{E}[\nabla m(\bm{X})]+\left(\bm{\Sigma}+\bm{\mu}\bm{\mu}^{T}\right)^{-1}\bm{\mu}\mathbb{E}(m(\bm{X})).
\end{align}
\end{rem}

For 1-dimensional linear regression we have a formula for the scaled
optimism in \eqref{eq:general_opt_expression}: 
\begin{cor}
\label{cor:When-1d-normal-1}When $\bm{x}_{*}\sim N(0,1)$ and $\bm{X}\sim N(0,1)$
we have a special form of \eqref{eq:general_opt_expression} using
an independent standard normal random variable $Z$: 
\end{cor}

\begin{align}
\frac{n\mathbb{E}_{\bm{X}}\text{Opt }R_{\bm{X}}}{2\sigma_{\epsilon}^{2}} & \overset{Z\sim N(0,1)}{\asymp}\frac{3\left(\mathbb{E}Z\mu(Z)\right)^{2}+\mathbb{E}Z^{2}\mu(Z)^{2}-2\mathbb{E}Z^{3}\mu(Z)\cdot\mathbb{E}Z\mu(Z)}{\sigma_{\epsilon}^{2}}+1+O_{p}\left(\frac{1}{n^{1/2}}\right).\label{eq:main eq-1-1}
\end{align}

\begin{proof}
See Appendix \ref{sec:Proof-of-Corollary}. 
\end{proof}
We can further write down the complexity measure when the actual signal
$\mu(x)$ is of the form $\sum_{i=0}^{\infty}A_{i}x^{i}$: 
\begin{cor}
\label{cor:analytic_signal to linear model-1}Under the same assumption
of Corollary \ref{cor:When-1d-normal-1}, when the signal $\mu(x)$
is of the form $\sum_{i=0}^{\infty}A_{i}x^{i}$, we have 
\begin{align}
\mathbb{E}_{\bm{X}}\frac{n}{2\sigma_{\epsilon}^{2}}\cdot\text{Opt }R_{\bm{X}} & \asymp\frac{1}{2\sigma_{\epsilon}^{2}}\cdot\left(F(A_{i},i\neq1)\right)+1+o(1),\label{eq:main eq-1}
\end{align}
which means that the signal part is a function that does not depend
on $A_{1}$, the linear part of the signal. 
\end{cor}

This result further confirms that the linear model only removes the
linear part (when there is an explicit linear part in the signal)
from the signal (as shown in Example \ref{exa:(Polynomial-signal)-When-1}
in Appendix \ref{sec:Computational-Examples-for}). When there is
not an explicit expression for the linear part in the signal, this
is less obvious (as shown in Example \ref{exa:Dirac signal} in Appendix
\ref{sec:Computational-Examples-for}). In the above corollaries \ref{cor:When-reducing-to-classical},
\ref{cor:When-1d-normal-1} and \ref{cor:analytic_signal to linear model-1},
we can observe that if we take the signal-independent part in \eqref{eq:general_opt_expression},
its scaled version coincide with the classical model degree of freedom.
In \citet{luan_predictive_2021}, they suggested that the signal-independent
part can be used as generic predictive complexity measure for a wide
class of models.

\subsection{More Theoretical Results}

A variant of the result in Theorem \ref{thm:main theorem} can be
elicited when we consider Eckhart-Young theorem in the context of
the low-rank regressions, where the input is projected onto a low-dimensional
space through projections \citep{ju2020overfitting,luo2024spherical}.
When computing the covariance matrix, it is a common practice to use
low-rank approximation to attain model sparsity or to reduce the cost
of repeated matrix inversions \citep{luo2022sparse,luo2024hybrid}.
Precisely, we use a rank-$k$ approximation $\bm{\Sigma}_{k}$ to
the matrix $\bm{\Sigma}=\mathbb{E}(\bm{x}_{*}\bm{x}_{*}^{T})$ in
prediction. The following theorem ensures that such a low-rank approximation
will not increase optimism that exceeds a perturbation bound \eqref{eq:general_opt_expression-1}. 
\begin{thm}
\label{thm:main theorem-low-rank} Under Assumption A1 and suppose
that $\bm{\Sigma}_{k}$ is a rank-$k$ approximation to the $\bm{\Sigma}$,
we can write down the expected random optimism for the rank-$k$ least
squares estimator is 
\begin{align}
 & \mathbb{E}_{\bm{X}}\text{Opt }R_{\bm{X}}\label{eq:general_opt_expression-1}\\
 & \leq\frac{2}{n}\mathbb{E}_{\bm{X}}\left[\left(\mathbb{E}_{\bm{x}_{*}}\left\Vert y_{*}-\bm{x}_{*}^{T}\left[\bm{\Sigma}_{k}^{-1}+\sigma_{k+1}^{-1}\bm{I}\right]\bm{\eta}\right\Vert _{2}^{2}\right)\cdot\left(\bm{x}_{*}^{T}\left(\bm{\Sigma}_{k}^{-1}+\sigma_{k+1}^{-1}\bm{I}\right)\bm{x}_{*}\right)\right]+O_{p}\left(\frac{1}{n^{3/2}}\right)\nonumber 
\end{align}
where $\sigma_{k+1}$ is the $(k+1)$-th largest singular value of
$\bm{\Sigma}$. 
\end{thm}

\begin{proof}
See Appendix \ref{sec:Proof-of-Theoremlowrank}. 
\end{proof}
In other words, the optimism is ``regularized by'' an amount $\sigma_{k+1}^{-1}\bm{I}$,
and we can choose the most appropriate rank $k$ based on the design
of $\bm{x}_{*}$. This form of covariance $\bm{\Sigma}_{k}^{-1}+\sigma_{k+1}^{-1}\bm{I}$
in \eqref{eq:general_opt_expression-1} inspired us to investigate
the related ridge linear regression model. Then, we state a variant
of Theorem \ref{thm:main theorem} also holds for ridge regression
and kernel ridge regressions under the following set of assumptions.

\textbf{Assumptions A2. } Let $\hat{\bm{\eta}}=\frac{1}{n}\bm{X}^{T}\bm{y}(\bm{X})=\frac{1}{n}\bm{X}^{T}\bm{y}$
and $\hat{\bm{\Sigma}}_{\lambda}=\frac{1}{n}\left(\bm{X}^{T}\bm{X}+\lambda\bm{I}\right)\in\mathbb{R}^{d\times d}$
for a fixed positive $\lambda$. We assume that 
\begin{equation}
\left\Vert \hat{\bm{\eta}}-\bm{\eta}\right\Vert _{2}=O_{p}\left(\frac{1}{\sqrt{n}}\right),\left\Vert \hat{\bm{\Sigma}}_{\lambda}-\bm{\Sigma}_{\lambda}\right\Vert _{2}=O_{p}\left(\frac{1}{\sqrt{n}}\right)\label{eq:assumption_A1-1-2}
\end{equation}
where $\bm{\eta}=\mathbb{E}_{\bm{x}_{*}}\bm{x}_{*}y(\bm{x}_{*})=\mathbb{E}_{\bm{x}_{*}}\bm{x}_{*}y_{*}$
and $\bm{\Sigma}_{\lambda}=\mathbb{E}_{\bm{x}_{*}}(\bm{x}_{*}\bm{x}_{*}^{T}+\lambda\bm{I})$.

By definitions of $\hat{\bm{\Sigma}}_{\lambda},\bm{\Sigma}_{\lambda}$,
Assumption A1 implies A2 for any $0\leq\lambda<\infty$. When $\lambda=0$,
this reduces to Theorem \ref{thm:main theorem}, hence can be considered
as a generalization to our main result. 
\begin{thm}
\label{thm:main theorem-ridge-2} Under Assumption A2, we can write
down the errors as

\begin{align*}
\mathbb{E}_{\bm{X}}\text{Err}R_{\bm{X}} & =\mathbb{E}_{\bm{x}_{*}}\left\Vert y_{*}-\bm{x}_{*}^{T}\hat{\bm{\beta}}\right\Vert _{2}^{2}\\
 & =\mathbb{E}_{\bm{x}_{*}}\left(y_{*}-\bm{x}_{*}^{T}\bm{\Sigma}_{\lambda}^{-1}\bm{\eta}\right)^{2}\\
 & +\frac{1}{n}\mathbb{E}_{\bm{x}_{*}}\left\Vert \bm{\Sigma}^{1/2}\bm{\Sigma}_{\lambda}^{-1}\left[\bm{x}_{*}y_{*}-\left(\bm{x}_{*}\bm{x}_{*}^{T}+\lambda\bm{I}\right)\bm{\Sigma}_{\lambda}^{-1}\bm{\eta}\right]\right\Vert _{2}^{2}+O_{p}\left(\frac{1}{n^{3/2}}\right).\\
\mathbb{E}_{\bm{X}}\text{Err}T_{\bm{X}} & =\frac{1}{n}\mathbb{E}_{\bm{X}}\left\Vert \bm{y}-\bm{X}\hat{\bm{\beta}}\right\Vert _{2}^{2}\\
 & =\mathbb{E}_{\bm{x}_{*}}\left(y_{*}-\bm{x}_{*}^{T}\bm{\Sigma}_{\lambda}^{-1}\bm{\eta}\right)^{2}\\
 & -\frac{1}{n}\mathbb{E}_{\bm{x}_{*}}\left\Vert \bm{\Sigma}^{-1/2}\left[\bm{x}_{*}y_{*}-\left(\bm{x}_{*}\bm{x}_{*}^{T}+\lambda\bm{I}\right)\bm{\Sigma}_{\lambda}^{-1}\bm{\eta}\right]\right\Vert _{2}^{2}+O_{p}\left(\frac{1}{n^{3/2}}\right).\\
\end{align*}
The expected random optimism for the least squares estimator is 
\begin{align}
\mathbb{E}_{\bm{X}}\text{Opt }R_{\bm{X}} & =\frac{1}{n}\mathbb{E}\left[\left(\bm{\Sigma}_{\lambda}^{-1}\bm{\Sigma}\bm{\Sigma}_{\lambda}^{-1}+\bm{\Sigma}_{\lambda}^{-1}\right)^{1/2}\left\Vert \left[\bm{x}_{*}y_{*}-\left(\bm{x}_{*}\bm{x}_{*}^{T}+\lambda\bm{I}\right)\bm{\Sigma}_{\lambda}^{-1}\bm{\eta}\right]\right\Vert _{2}^{2}\right]\nonumber \\
 & +O_{p}\left(\frac{1}{n^{3/2}}\right).\label{eq:easy-to-read-ridge-2}
\end{align}
\end{thm}

\begin{proof}
See Appendix \ref{sec:Proof-of-Theorem ridge}. 
\end{proof}
\begin{rem}
Using Neumann series for $\left\Vert \bm{A}^{-1}\bm{B}\right\Vert _{2}<1$:
\begin{equation}
(\bm{A}+\bm{B})^{-1}=\bm{A}^{-1}-\bm{A}^{-1}\bm{B}\bm{A}^{-1}+\bm{A}^{-1}\bm{B}\bm{A}^{-1}\bm{B}\bm{A}^{-1}+\cdots\label{eq:Neumann}
\end{equation}
for $\bm{A}=\bm{\Sigma},\bm{B}=\bm{I}$, we can see that the effect
of low-rank approximation in linear models is connected to ridge linear
regression if we can find an $\lambda$ such that 
\begin{align*}
\left\Vert \bm{\Sigma}_{\lambda}^{-1}\right\Vert  & =\left\Vert \left(\bm{\Sigma}+\lambda\bm{I}\right)^{-1}\right\Vert \\
 & =\left\Vert \bm{\Sigma}^{-1}-\lambda\bm{\Sigma}^{-2}+\lambda^{2}\bm{\Sigma}^{-3}+\cdots\right\Vert \\
 & =\left\Vert \bm{\Sigma}_{k}^{-1}+\bm{\Sigma}^{-1}-\bm{\Sigma}_{k}^{-1}-\lambda\bm{\Sigma}^{-2}+\lambda^{2}\bm{\Sigma}^{-3}+\cdots\right\Vert \\
 & =\left\Vert \bm{\Sigma}_{k}^{-1}+\bm{\Sigma}^{-1}-\bm{\Sigma}_{k}^{-1}\right\Vert +O_{p}(\lambda\bm{\Sigma}^{-2})\\
 & \asymp\left\Vert \bm{\Sigma}_{k}^{-1}+\sigma_{k+1}^{-1}\bm{I}\right\Vert +O_{p}(\lambda\bm{\Sigma}^{-2})\text{ as in }\eqref{eq:general_opt_expression-1}.
\end{align*}
with our results in Theorems \ref{thm:main theorem-low-rank} and
\ref{thm:main theorem-ridge-2}. This means that when $\left\Vert \bm{\Sigma}^{-1}\right\Vert _{2}<1$,
with appropriate choices of $\lambda$'s, the ridge regression models
and low-rank approximated linear models can behave similarly in terms
of optimism (i.e., generalization errors).

Note that the $\lambda$ terms in \eqref{eq:easy-to-read-ridge-2}
depends on both signal $y_{*}$ and the model $\bm{\Sigma}_{\lambda}$,
which makes the signal-dependent and signal-independent parts no longer
separable as in \citet{luan_predictive_2021}. This motivates us to
consider optimism as a more general form of predictive complexity
that also applies to regularized models. In the case where $\lambda=0$,
the positivity of the optimism is ensured; but when regularization
is introduced, it is possible to obtain a negative optimism (See Appendix
\ref{sec:Proof-of-Theorem ridge} for detailed discussion of positivity
in line with Corollary \ref{cor:Under-the-positivity}). 
\end{rem}

It is clear that when $\lambda=0$, \eqref{eq:easy-to-read-ridge-2}
reduces to \eqref{eq:general_opt_expression}. When $\lambda\rightarrow\infty$,
the fitted model will be a constant model, hence produce the same
$\mathbb{E}_{\bm{X}}\text{Err}T_{\bm{X}}$ and $\mathbb{E}_{\bm{X}}\text{Err}R_{\bm{X}}$
and zero $\mathbb{E}_{\bm{X}}\text{Opt}R{}_{\bm{X}}$. To establish
at what rate $\mathbb{E}_{\bm{X}}\text{Opt}R{}_{\bm{X}}$ converges
to zero, we first note that $\bm{\Sigma}_{\lambda}^{-1}=\left(\bm{\Sigma}+\lambda\bm{I}\right)^{-1}=\lambda^{-1}\bm{I}-\lambda^{-2}\bm{\Sigma}^{1/2}\left(\lambda^{-1}\bm{\Sigma}+\bm{I}\right)^{-1}\bm{\Sigma}^{1/2}$
by Woodbury lemma. So 
\begin{align*}
\bm{\Sigma}_{\lambda}^{-1}\bm{\Sigma}\bm{\Sigma}_{\lambda}^{-1}+\bm{\Sigma}_{\lambda}^{-1} & =\left(\lambda^{-1}\bm{\Sigma}-\lambda^{-2}\bm{\Sigma}^{1/2}\left(\lambda^{-1}\bm{\Sigma}+\bm{I}\right)^{-1}\bm{\Sigma}^{3/2}\right)\bm{\Sigma}_{\lambda}^{-1}\\
 & =\lambda^{-1}\bm{\Sigma}\bm{\Sigma}_{\lambda}^{-1}-\lambda^{-2}\bm{\Sigma}^{1/2}\left(\lambda^{-1}\bm{\Sigma}+\bm{I}\right)^{-1}\bm{\Sigma}^{3/2}\bm{\Sigma}_{\lambda}^{-1}
\end{align*}
Using this expansion 
\begin{align*}
\lim_{\lambda\rightarrow\infty}\bm{\Sigma}_{\lambda} & =O(\lambda\bm{I}),\\
\lim_{\lambda\rightarrow\infty}\bm{\Sigma}_{\lambda}^{-1} & =\lambda^{-1}\bm{I}+O(\lambda^{-2}\bm{I}).
\end{align*}
Then using these two limits we analyze terms in \eqref{eq:easy-to-read-ridge-2},
we obtain that the optimism \eqref{eq:easy-to-read-ridge-2} converges
to $0$ at a rate $O\left(\lambda^{-1}\right)$.

When $\lambda=0$, we are fitting a linear model and can observe the
same trend (zero for $k>0.5$, non-zero for $k\leq0.5$) (See Figures
\ref{fig:Noise_N_MCMC_grid} and \ref{fig:Different-models-comparison}
for more details). When $\lambda\rightarrow\infty$, we are fitting
a horizontal straight line model and $k=0.5$ is the only correctly
fitted model with zero optimism. The interesting phenomenon is when
$\lambda\approx1000$, where the difference in signals (different
$k$'s) is highlighted in the optimism calculation. To describe this
generalization in the kernel ridge regression setting, we consider
feature mapping $\phi:\mathbb{R}^{d}\rightarrow\mathbb{R}^{q}$, and
$\bm{\Phi}=\left(\phi(\bm{x}_{1})^{T},\cdots,\phi(\bm{x}_{n})^{T}\right)\in\mathbb{R}^{n\times q}$
consisting of row feature vectors $\phi(\bm{x}_{i})\in\mathbb{R}^{q\times1}$.
We consider the following regression problem as a special case of
\eqref{eq:train_lossL2}: 
\begin{equation}
\hat{\mu}=\arg\min_{f\in\mathcal{H}_{K}}\frac{1}{n}\sum_{i=1}^{n}\left(y_{i}-f(\bm{x}_{i})\right)^{2}+\lambda\|f\|_{K}^{2},
\end{equation}
where we take the loss function $\ell$ as $\|\cdot\|_{2}$ and $\mathcal{F}_{n}=\mathcal{H}_{K}$
as the reproducing Hilbert kernel space \citep{aronszajn1950theory}
and its norm $\|\cdot\|_{K}$ induced by (the inner product) kernel
function $K:\mathbb{R}^{d}\times\mathbb{R}^{d}\rightarrow\mathbb{R}$.
Its solution is given by: 
\begin{align}
\hat{\mu}(\bm{x}_{*}) & =\phi(\bm{x}_{*})^{T}\underset{\in\mathbb{R}^{q\times q}}{\underbrace{\left(\bm{\Phi}^{T}\bm{\Phi}+\lambda\bm{I}\right)^{-1}}}\bm{\Phi}^{T}\bm{y},\nonumber \\
 & =\phi(\bm{x}_{*})^{T}\bm{\Phi}^{T}\underset{\in\mathbb{R}^{n\times n}}{\underbrace{\left(\bm{\Phi}\bm{\Phi}^{T}+\lambda\bm{I}\right)^{-1}}}\bm{y},\nonumber \\
 & =K(\bm{x}_{*},\bm{X})\left(K(\bm{X},\bm{X})+\lambda\bm{I}\right)^{-1}\bm{y},\label{eq:kernel ridge estimator}
\end{align}
where $\bm{\Phi}=\left(\phi(\bm{x}_{1})^{T},\cdots,\phi(\bm{x}_{n})^{T}\right)\in\mathbb{R}^{n\times q}$
consisting of row feature vectors $\phi(\bm{x}_{i})\in\mathbb{R}^{q\times1}$
via feature mapping $\phi:\mathbb{R}^{d}\rightarrow\mathbb{R}^{q}$,
$K(\bm{X},\bm{X})=\left\llbracket K(\bm{x}_{i},\bm{x}_{j})\right\rrbracket _{i,j=1}^{n}=\bm{\Phi}^{T}\bm{\Phi}\in\mathbb{R}^{q\times q}$
is the Gram matrix of the kernel $K:\mathbb{R}^{d}\times\mathbb{R}^{d}\rightarrow\mathbb{R}$,
$K(\bm{x}_{*},\bm{X})$ is the $1\times n$ kernelized vector $\left(K(\bm{x}_{*},\bm{x}_{1}),\cdots,K(\bm{x}_{*},\bm{x}_{n})\right)$
and $\lambda$ is the regularization parameter. The following assumption
holds if the feature mapping $\phi$ is Lipschitz bounded and Assumption
A2 holds.

\textbf{Assumptions A3. }Let $\hat{\bm{\eta}}_{\phi}=\frac{1}{n}\bm{\Phi}^{T}\bm{y}(\bm{X})=\frac{1}{n}\bm{\Phi}^{T}\bm{y}$
and $\hat{\bm{\Sigma}}_{\phi,\lambda}=\frac{1}{n}\left(\bm{\Phi}^{T}\bm{\Phi}+\lambda\bm{I}\right)\in\mathbb{R}^{q\times q}$
for a fixed positive $\lambda$. We assume that 
\begin{equation}
\left\Vert \hat{\bm{\eta}}_{\phi}-\bm{\eta}_{\phi}\right\Vert _{2}=O_{p}\left(\frac{1}{\sqrt{n}}\right),\left\Vert \hat{\bm{\Sigma}}_{\phi,\lambda}-\bm{\Sigma}_{\phi,\lambda}\right\Vert _{2}=O_{p}\left(\frac{1}{\sqrt{n}}\right)\label{eq:assumption_A1-1-1-1}
\end{equation}
where $\bm{\eta}_{\phi}=\mathbb{E}_{\bm{x}_{*}}\phi(\bm{x}_{*})y(\bm{x}_{*})=\mathbb{E}_{\bm{x}_{*}}\phi(\bm{x}_{*})y_{*}\in\mathbb{R}^{q\times1}$
and $\bm{\Sigma}_{\phi,\lambda}=\mathbb{E}(\phi(\bm{x}_{*})\phi(\bm{x}_{*})^{T}+\lambda\bm{I})\in\mathbb{R}^{q\times q},\bm{\Sigma}_{\phi}=\bm{\Sigma}_{\phi,0}$.

Under this assumption, the following result can be derived using identical
arguments as Theorem \ref{thm:main theorem-ridge-2} with $\bm{x}_{*}$
replaced with $\phi(\bm{x}_{*})$. 
\begin{thm}
\label{thm:main-thm-kernel}Under Assumption A3, the expected random
optimism for the least squares kernel ridge estimator \eqref{eq:kernel ridge estimator}
defined by the kernel $K(\cdot,\cdot)=\phi(\cdot)^{T}\phi(\cdot)$,
is 
\begin{align}
\mathbb{E}_{\bm{X}}\text{Opt }R_{\bm{X}} & =\frac{2}{n}\mathbb{E}_{\bm{X}}\left[\left\Vert \left(\bm{\Sigma}_{\phi}^{1/2}\bm{\Sigma}_{\phi,\lambda}^{-1}\right)\left[\phi(\bm{x}_{*})y_{*}-\left(\phi(\bm{x}_{*})\phi(\bm{x}_{*}){}^{T}+\lambda\bm{I}\right)\bm{\Sigma}_{\phi,\lambda}^{-1}\bm{\eta}_{\phi}\right]\right\Vert _{2}^{2}\right]\label{eq:general_opt_expression-2-1-1}\\
 & +O_{p}\left(\frac{1}{n^{3/2}}\right)\label{eq:easy-to-read-ridge-1-1}
\end{align}
\end{thm}

\begin{proof}
This can be derived using identical arguments as in Appendix \ref{sec:Proof-of-Theorem ridge}
for Theorem \ref{thm:main theorem-ridge-2} with $\bm{x}_{*}$ replaced
with $\phi(\bm{x}_{*})$. 
\end{proof}
\begin{rem}
This result does not only apply to kernel ridge regressions (KRR)
\citep{hastie2009elements}, but also applicable to the posterior
mean estimator of a Gaussian process regression \citep{kanagawa2018gaussian,kimeldorf1970correspondence}.
This result in optimism also applies to GP regression with nugget
$\lambda$, hence available to us when we need optimism for model
selection as shown in \citet{luo2024hybrid} and kernel selection
as shown in \citet{allerbo2022bandwidth}. 
\end{rem}

Using NTK in KRR establishes a bridge between NN training and kernel
methods. The neural tangent kernel (NTK) provides a linearized framework
where the network output is approximated by a fixed kernel function.
The assumption of small weight changes ensures this equivalence and
validates the NTK's role as a linear approximation of NNs during training.
Consider a two-layer fully connected NN \citep{jacot2018neural,arora2019fine,geifman2020similarity}
with $m$ ReLU activation functions in the hidden layer, its functional
form is: 
\[
g(\bm{x};W,a)=\frac{1}{\sqrt{m}}\sum_{j=1}^{m}a_{j}\sigma(w_{j}^{T}\bm{x}),
\]
where: $W=\left(w_{1},\dots,w_{m}\right)\in\mathbb{R}^{d\times m}$
and $w_{j}\in\mathbb{R}^{d\times1}$are the bottom-layer weights,
$\bm{a}=(a_{1},\dots,a_{m})^{T}\in\mathbb{R}^{m}$ are the top-layer
weights and $\sigma(z)=\max\{z,0\}$ is the ReLU activation function.
During training, the bottom-layer weights $W$ are updated using gradient
descent (c.f., Section 2 and Proposition 1 of \citet{jacot2018neural}).
Let the change in weights be denoted by $\Delta W$, which is assumed
to be small. In this regime, the network output can be linearized
as: 
\[
g(x;W_{0}+\Delta W,\bm{a})\approx g(x;W_{0},\bm{a})+\nabla_{W}g(x;W_{0},\bm{a})\cdot\text{vec}(\Delta W),
\]
where $W_{0}$ is the initialization of weights, $\text{vec}(\Delta W)$
is the vectorization of the weight updates. The neural tangent kernel
$\Theta:\mathbb{R}^{d}\times\mathbb{R}^{d}\rightarrow\mathbb{R}$
is then defined through the mapping $\phi(\bm{x})=\nabla_{W}g(\bm{x};W_{0},\bm{a})$
as: 
\begin{equation}
\Theta(\bm{x},\bm{x}')=\nabla_{W}g(\bm{x};W_{0},\bm{a})^{T}\nabla_{W}g(\bm{x}';W_{0},\bm{a}),\label{eq:NTK_kernel}
\end{equation}
with the same architecture as in Algorithm \ref{alg:3-layer-neural-network}.
This kernel takes gradient only with respect to the bottom layer weights
$W$ but has parameters $W,\bm{a}$ and can be fitted as a kernel
regression model as detailed in Algorithm \ref{alg:Simulation-algorithm-NTK}.

Then, we can use the optimism to delineate the difference between
linear models and NN under the setup \citep{arora2019fine}, NTK acts
as a kernel that transforms the input space into a feature space where
regression is linear and regularized.

\subsection{Simulation Results\label{subsec:Simulation-Results}}

In this section, we show that if we use optimism as a model complexity
measure, the NN may have a very low complexity measure value because
the NN usually generalize well even when trained on one set but tested
on another.

In the subsequent simulation experiments, we set the $N=100$ and
$n_{\text{train}}=n_{\text{test}}=1000$ unless otherwise is stated.
We consider the following signal function $f_{k}$ parameterized by
$k\in[0,1]$ on the domain $[-1,1]\subset\mathbb{R}^{2}$ with additive
i.i.d. noise $\epsilon\sim N(0,\sigma^{2})$, i.e., $y(x)=f_{k}(x)+\epsilon$.
\begin{align}
f_{k}(x) & =\begin{cases}
\frac{0.5-k}{0.5}\max LU(x,0)=\frac{0.5-k}{0.5}\max\left(0,x\right) & k<0.5\\
\frac{k-0.5}{0.5}(-x) & k\geq0.5
\end{cases}\label{eq:test fun}
\end{align}

\begin{figure}
\centering

\includegraphics[clip,width=0.95\textwidth,viewport=0bp 0bp 1072bp 215.973bp]{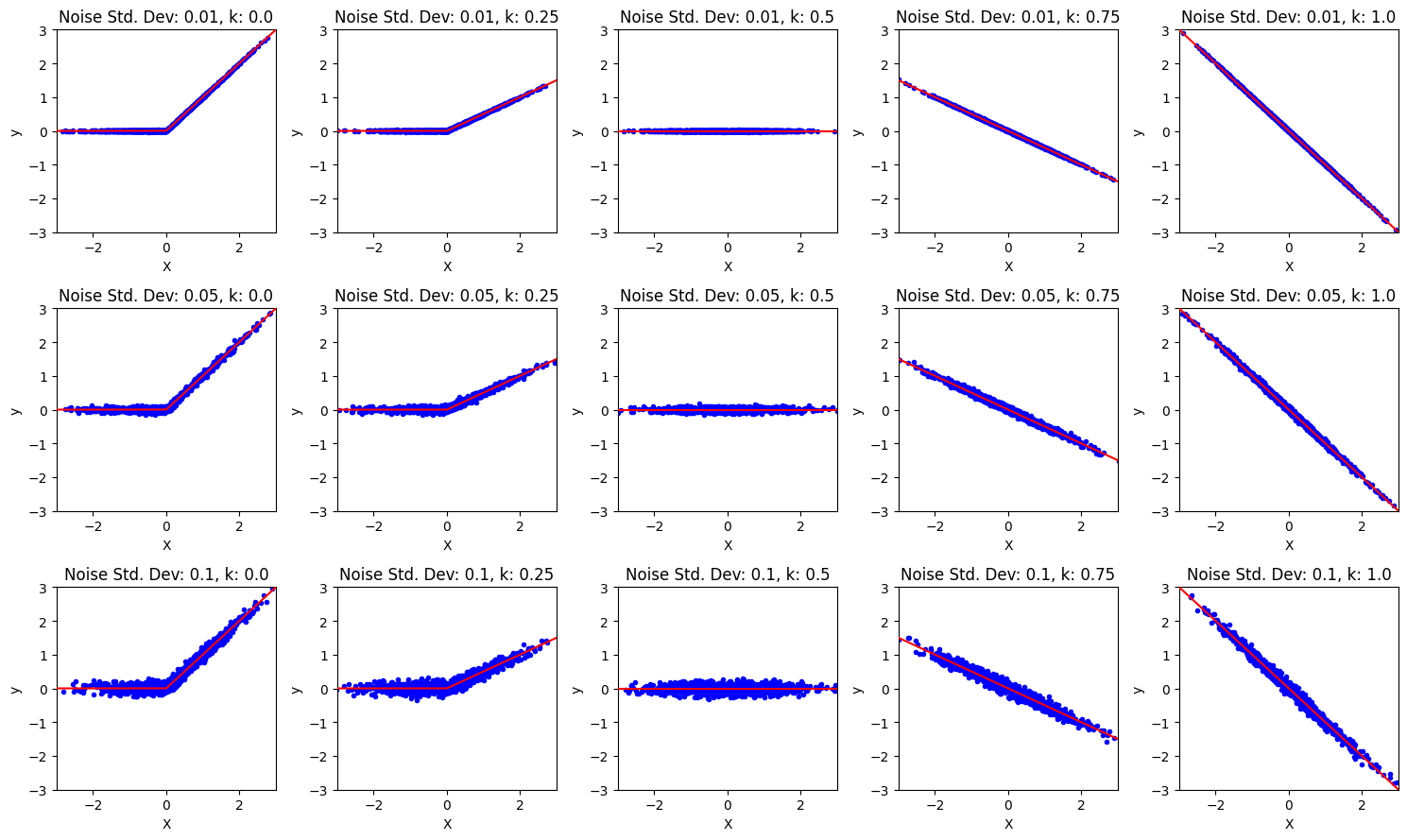}

\caption{\label{fig:Testing-signals-functions}Testing signals functions $f_{k}$
in \eqref{eq:test fun} with white noise variance $0.1$. The red
solid line indicates the true signal, the blue dots are sample points
with noise. \protect \protect}
\end{figure}

To empirically verify our results, we study the signal function \eqref{eq:test fun}
fitted to the following linear (and ridge regression with $\lambda=0.01,0.1$),
bended and 3-layer NN models with a specified number of hidden nodes
(as expressed in Algorithm \ref{alg:3-layer-neural-network} and Figure
\ref{fig:Opt_vs_epoch}). For the linear and bended (a.k.a., ReLU)
models, they assume explicit forms as:

\begin{align}
\mu(x)=\alpha x+\beta, & \text{(linear)}\label{eq:linear_assumption}\\
\mu(x)=\alpha+\beta\cdot\max(x,0), & \text{(bended)}\label{eq:bended_assumption}
\end{align}
where the optimization problem associated with models \eqref{eq:linear_assumption}
and \eqref{eq:bended_assumption} are both convex from a direct verification.

The 3-layer NN we consider can be described by its fitting procedure
in Algorithm \ref{alg:3-layer-neural-network} (See Appendix \ref{sec:Related-Algorithms-for})
has 50 hidden nodes and ReLU activation function as an architecture
choice. The choice of activation functions affects the weighting scheme
between layers and has certain degree of influence in the resulting
fit \citep{bishop1995neural}. However, we observe that when we increase
the number of nodes in the only hidden layer from 2 to 20, then the
mis-specification effect seems to become milder. The optimism of a
3-layer NN with 50 hidden nodes is close to correctly specified models
while the 3-layer NN with 2 hidden nodes shows random behavior. Its
corresponding NTK kernel regression model, however, behaves more like
simple linear and bended models.

The optimism can be computed using an MC Algorithm \ref{alg:Simulation-algorithm-for}
(See Appendix \ref{sec:Related-Algorithms-for} and \ref{sec:enough MC size})
where both the signal (varying with parameter $k$) and the noise
variance can change. We want to investigate how the scaled expected
optimism (divided by the known noise variance $\overline{\text{Opt}}=\frac{\text{Opt }}{2\cdot\sigma^{2}}\cdot n_{\text{train}}$)
and the raw expected optimism (simply $\text{Opt}$ in \eqref{eq:general_opt_expression})
changes when the noise variance changes. We fix the $k$ in the signal
function, resulting in different signals whose shapes are shown in
Figure \ref{fig:Testing-signals-functions}. The scaled expected optimism
of a KRR with NTK kernel, however, is not similar to the \ref{fig:Different-models-comparison-1}.
This empirical findings show that the expected optimism can tell kernel
models apart from the NN in practice.

For the effect of different noise variances (on different panels of
scaled expected optimism shown in Figure \ref{fig:Different-models-comparison}),
we observe the magnitude of generalization errors changes. Most importantly,
the relative magnitude of scaled optimism changes as the noise variance
increases, even if scaled by the noise variance. The NN has an increasing
optimism when the noise variance increases compared to linear models
(and ridge, kernel models). Increasing optimism indicates a worse
generalization ability as the additive noise in the signal increase
(i.e., signal-to-noise ratio decreases)

As for the effect of different values of $k$ (on the scale/magnitude
of scaled expected optimism shown in Figure \ref{fig:Different-models-comparison}),
this would depend on the specific form of the signal function and
how it interacts with the $\bm{x}$ variables in the model. If the
signal function does not accurately capture the true relationship
between the variables for certain values of $k$, then the model would
be mis-specified for those values of $k$, which explains the trends
for linear (correctly specified only when $k=1.0$) and bended (correctly
specified only when $k=0.0$) models in Figure \ref{fig:Different-models-comparison}.

When $k\geq0.5$, linear and NN converge to a correctly specified
linear model within its model family and result in relatively small
generalization errors. When $k=0$, bended and NN converge to a correctly
specified bended model within its model family and result in relatively
small generalization errors. When $k<0.5$, while parametric models
(i.e., linear bended) both converge to constant white noise and result
in nearly 0 generalization error, NN seem to be more sensitive to
the amount of noise. This echoes the empirical fact that the NN rarely
exhibits mis-specification due to its high flexibility. For kernel
regression with NTK kernel, its optimism is lower than mis-specified
linear and bended models, but never attain low optimism as good as
correctly specified models nor NN (except for $k=0.5$), arguably
presenting robustness against mis-spcifications.

Ridge models with $\lambda=0.1,0.01$ are among the worst models,
especially when the noise variance is low. In Figure \ref{fig:Different-models-comparison-1},
we can observe that when the model is mis-specified the regularization
only deteriorates the generalization. The kernel regression with NTK
exhibits different behavior in terms of expected optimism, compared
to NN. This comparison strengthens our theoretical results and support
the findings that the NN is different from simple kernelization.

Using our Theorem \ref{thm:main theorem}, we plug in the expression
\eqref{eq:test fun} of $f_{k}$ into \eqref{eq:general_opt_expression}
to compute the closed form of (scaled) optimism for linear model \eqref{eq:linear_assumption},
and with the assumption that both training and testing sets are standard
normal (See Appendix \ref{sec:Calculation-for} for detailed calculations).

\begin{align}
\mathbb{E}_{\bm{X}}\frac{n}{2\sigma_{\epsilon}^{2}}\cdot\text{Opt }R_{\bm{X}} & \asymp\begin{cases}
\frac{1}{2\sigma_{\epsilon}^{2}}\cdot\frac{3}{2}(1-2k)^{2} & k<0.5\\
0 & k\geq0.5
\end{cases}+1+o(1).\label{eq:test fun-1}
\end{align}
This formula perfectly coincides with the experimental results in
Figure \ref{fig:Different-models-comparison}, where linear model
shows a quadratic decreasing trend when $k<0.5,$ follwed by a nearly
zero generaliztion error for linear signals after $k\geq0.5$.

\begin{figure}
\centering \includegraphics[clip,width=0.95\textwidth,viewport=0bp 0bp 708bp 250.7215bp]{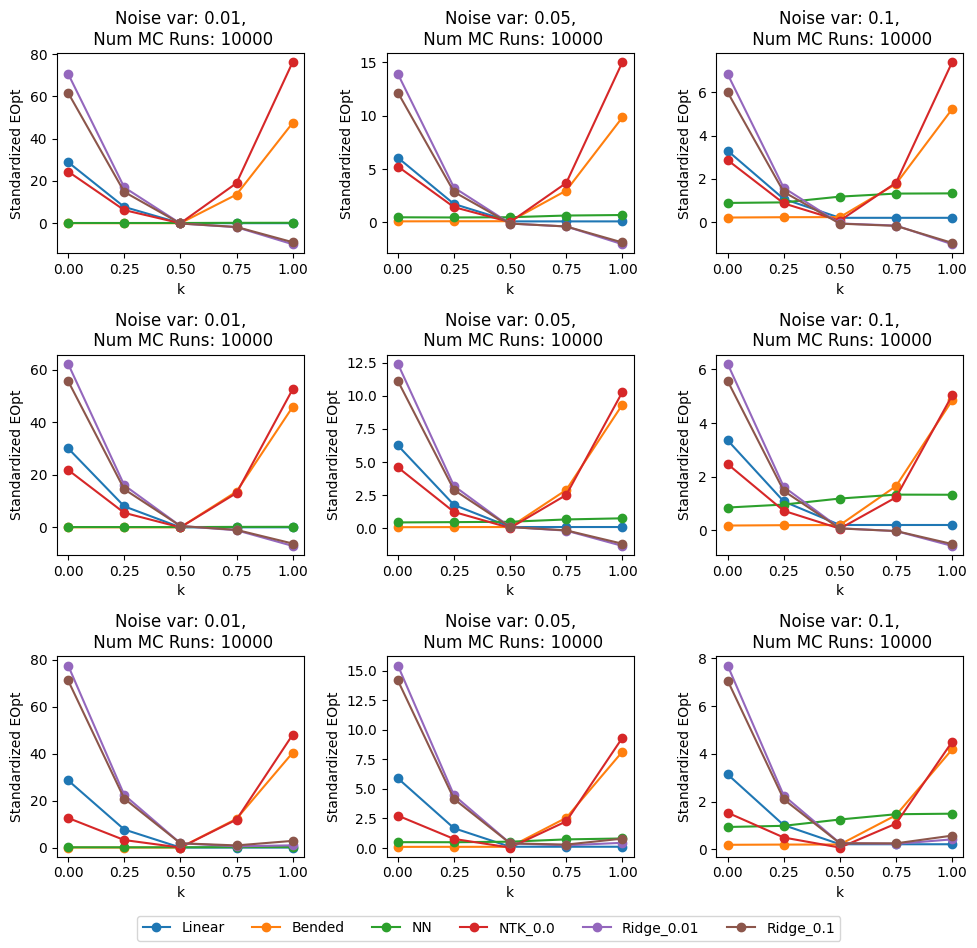}
\caption{\label{fig:Different-models-comparison}\label{fig:Different-models-comparison-1}Different
columns indicates difference additive noise variances $\sigma_{\epsilon}^{2}=0.01,0.05,0.1$
$k=0.0,0.1,...,0.9,1.0$ which controls the shapes of signals. In
each panel, the x-axis is the changing $k$, y-axis is the (scaled)
optimism computed from $N_{MC}=10,000$. The model $NTK\_0$ means
kernel regression using \eqref{eq:NTK_kernel} (See Algorithm \ref{alg:Simulation-algorithm-NTK}
in Appendix \ref{sec:Related-Algorithms-for}) with no regularization;
$Ridge\_\lambda$ means linear ridge regression with different regularization
paramters $\lambda$. }
\end{figure}

\subsection{Real-data Experiments}

\begin{figure}
\centering \includegraphics[width=0.75\textwidth]{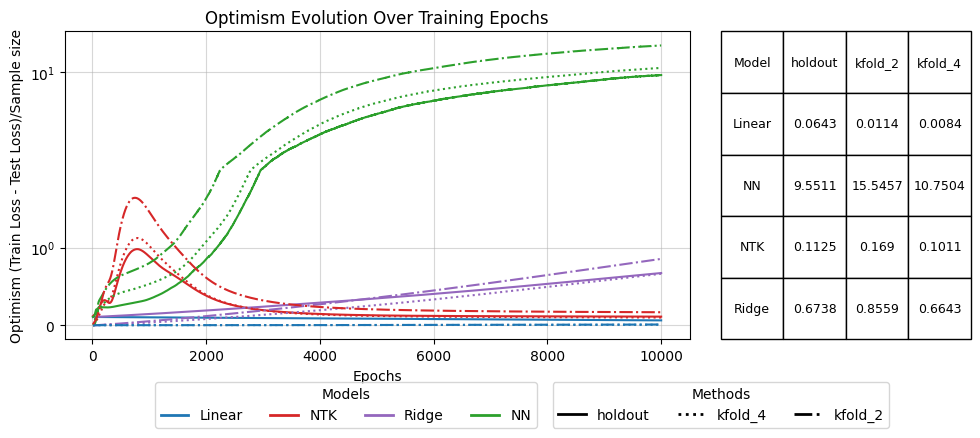}
\caption{\label{fig:Realdata-models-comparison-}Different models fitted on
the diabetes dataset \citep{efron2004least} with 442 samples and
10-dimensional input. The x-axis is the training epochs using the
same Adam optimizer (fixed learning rate 0.1), y-axis is the optimism
divided by sample size computed from $N_{MC}=\texttt{num\_runs}=10,000$
using hold-out (Algorithm \ref{alg:Holdout}) and $\texttt{k}$-fold
(Algorithm \ref{alg:KFold} with $\texttt{k}=2,4$) methods. The table
shows the optimism computed at epoch 10,000. Linear and NN models
have the same architecture as in Section \ref{subsec:Simulation-Results}.
The model $NTK$ means kernel regression using \eqref{eq:NTK_kernel}
(See Algorithm \ref{alg:Simulation-algorithm-NTK} in Appendix \ref{sec:Related-Algorithms-for})
with no regularization; $Ridge$ means linear ridge regression with
regularization paramter $\lambda=0.01$.}
\end{figure}

For real datasets, one cannot simulate multiple batches of testing
and training sets, Algorithms \ref{alg:Holdout} and \ref{alg:KFold}
in Appendix \ref{sec:Related-Algorithms-for} generalize the simulation
procedure of Algorithm \ref{alg:Simulation-algorithm-for} (which
generates synthetic training and testing data) to real-world datasets
where no such generation mechanism is available. Instead of sampling
from a known function, these methods estimate the generalization error
in terms of optimism \citep{efron_estimation_2004} by splitting or
re-sampling finite data. Algorithm \ref{alg:Holdout} (hold-out) divides
the dataset once into training and testing partitions, then computes
out-of-sample performance directly, the expectated values are average
across different hold-out splots. Algorithm \ref{alg:KFold} (\texttt{k}-fold
cross-validation \citep{geisser1975predictive}) partitions the data
into \texttt{k} folds, cycling each fold as the test set for a more
robust and often less biased estimate of error. While hold-out is
faster and suited for large datasets, \texttt{k}-fold is preferred
when data are limited or when a more stable error estimate is desired.
Both approaches replace synthetic generation with principled observed
data splitting, thereby offering practical methods to evaluate and
correct for overfitting in real-data scenarios.

In Figure \ref{fig:Realdata-models-comparison-}, we can observe that
these methods (Algorithm \ref{alg:Holdout} and \ref{alg:KFold})
yield similar estimation of raw testing minus training errors, and
normalized by the sample size 442.

First, the NN's optimism grows substantially as training proceeds,
underscoring that large-capacity models can strongly overfit if trained
for many epochs. By contrast, the Linear model's optimism remains
near zero throughout, reflecting its relatively low capacity and inability
to overfit severely on this dataset. The Ridge model and the NTK kernel
regression approach lie between these extremes, showing moderate overfitting
that eventually plateaus. Second, the final optimism estimates in
the table vary with the data-splitting strategy. In particular, the
$\texttt{k}$-fold estimates ($\texttt{k}=2,4$) often differ from
the hold-out estimate because cross-validation both uses the data
more efficiently and can lead to slightly different estimates of the
gap between train and test performance.

Overall, the figure highlights that higher-capacity NN models may
be more pone to growing train--test gaps over long training, yet
regularization (as in Ridge regression) mitigates overfitting but
does not eliminate it entirely. For real datasets, different resampling
methods can yield different numerical estimates of optimism, especially
in small-to-medium data settings like the diabetes dataset \citep{efron2004least}.

\section{\label{sec:Contribution-and-Discussion}Contribution and Discussion}

\subsection{Contributions}

In this paper, we study the performance of linear regression and its
variant kernel ridge regression in terms of the (scaled) optimism.
As a predictive model complexity measure, we defined and computed
expected optimism as the difference between testing and training errors
under random-X setting. Then, we derive the asymptotic analytically
closed expressions for the optimism for both linear (Theorem \ref{thm:main theorem})
and kernel ridge regressions (Theorem \ref{thm:main theorem-ridge-2}),
showing its positivity and its connection to low-rank approximated
model. A key contribution of our study is the closed-form expressions
for regression models and the extension of theoretical understanding
around the optimism metric --- the expected difference between testing
error and training error in model predictions under these models.

Our results show that the optimism is closely related to the model
capacity (e.g., degree of freedom in linear model), and the intrinsic
complexity of the underlying signal. With regularized and kernelized
models, the asymptotic results can be used to study more complex models.
By analyzing the asymptotic expressions for the optimism, we may gain
more insights into the factors that drive the double descent phenomenon
and understand how different models behave in the underparameterized,
interpolation threshold, and overparameterized regions.

Our paper further delineates how various types of regression functions
(linear, bended, NTK kernel) and regression NNs behave under different
signal settings, thus contributing a layered complexity to the understanding
of the double descent curve \citep{luan_predictive_2021,jiao2024assessment}.
With analytically closed asymptotic expected optimism, we can compute
it as a model predictive measure and we also find an interesting difference
in generalization behavior between NTK kernel and NN regressions,
showing that although NTK can approximate behavior of NN; NN is fundamentally
different from simple kernelizations using NTK kernels \citep{jacot2018neural}
in terms of optimism metrics.

\subsection{Future works}

The paper sets the groundwork for several promising directions of
research. One immediate area for further exploration is the application
of our theoretical findings to a different loss function than $L_{2}$
(e.g., $L_{1}$/LASSO regressions \citep{ju2020overfitting}, classification
\citep{belkin2018overfitting}), which could potentially validate
the applicability of the optimism metric across different types of
predictive modeling beyond regression.

Since NN is different from simple kernelizations in terms of optimism
behavior under different signals, it remains an open problem whether
a recursive kernelization (e.g., deep GP \citep{dunlop2018deep})
can approximates the NN better in terms of optimism in the context
of Theorem \ref{thm:main theorem}, \ref{thm:main theorem-ridge-2}
and \ref{thm:main-thm-kernel}. This could uncover additional insights
into the behavior of optimism in relation to kernel methods, and assists
in adjusting the number of layers and nodes \citep{martin2021implicit}~based
on generalization errors.

Extending the model complexity discussion into higher dimensional
input spaces, tensor regressions usually have low-rank structures
in the input space \citep{kolda2009tensor,luo2024efficient,kielstra2024linear},
therefore low-rank approximation are widely adopted in regression
models while Eckhart-Young type theorem no longer holds. One interesting
direction is to consider low-rank regression described in Theorem
\ref{thm:main theorem-low-rank} for tensor inputs to calibrate tensor
regressions and for rank estimation based on optimism.

Theorem \ref{thm:main-thm-kernel} provides a generic expression for
kernel ridge regression, however, its in-depth analysis will involve
approximation characterization of kernel random matrices $\bm{\Sigma}_{\phi},\bm{\Sigma}_{\phi,\lambda}$'s
as shown in \citet{koltchinskii2000random} when $n>d$. When $n<d$,
we can also directly inspect kernel function $K$ our random design
assumption for inner product and stationary kernels satisfy the assumption
for increasing dimension $d$ \citep{el2010spectrum}. This could
lead to a more comprehensive understanding of model generalization
behaviors in high-dimensional input spaces.

\section*{Acknowledgment}

HL was supported by U.S. Department of Energy under Contract DE-AC02-05CH11231
and U.S.~National Science Foundation NSF-DMS 2412403. Authors thank
Haoming Shi for proofreading our manuscripts and discussions on main
results.

 \bibliographystyle{plainnat}
\bibliography{Optimism}

\FloatBarrier

\appendix

\section{\label{sec:Related-Algorithms-for}Related Algorithms for Simulations}

In this section, we introduce algorithms used for asymptotic optimism
\eqref{eq:optR_X_ex2} for synthetic data and real data.

Algorithm \ref{alg:Simulation-algorithm-for} employs synthetically
generated data for both training and testing, wherein each run constructs
the input features and corresponding response values from a known
function (e.g., $f_{k}$ as in \eqref{eq:test fun}) with simulated
additive noise. This simulation-based framework is particularly useful
for controlled experiments and theoretical investigations, since it
allows the researcher to manipulate the level of noise or the complexity
of the signal and then observe how the model responds during training
and testing. However, this approach is not directly applicable to
real-world data scenarios because one typically does not have the
procedure of freely generating labeled samples from a specified known
function. Instead, in practice, data are finite and often cannot be
easily replaced or expanded through artificial means.

Algorithms \ref{alg:Holdout} and \ref{alg:KFold} address this limitation
by adapting the training and testing procedure to real datasets. The
core difference is that instead of generating new training and testing
sets at each run, the algorithms split or re-sample an existing dataset
in systematic ways to estimate both the training performance and the
generalization error. Algorithm \ref{alg:Holdout}, referred to as
the ``hold-out'' generalization, creates a single split of the dataset
into a training portion and a test portion, trains the model on the
training set over a specified number of epochs, and then evaluates
on the test set. By repeating this procedure several times with different
splits or different random initializations, one can measure how the
model performs on unseen data and thereby estimate its tendency to
overfit. The hold-out approach is straightforward to implement and
computationally less intensive; it is therefore appealing when the
dataset is large enough that a single (e.g., 80-20 split) split will
still produce a sufficiently reliable estimate of test performance.

Algorithm \ref{alg:KFold}, often referred to as the \texttt{k}-fold
cross-validation generalization, takes a more systematic approach
by partitioning the dataset into \texttt{k} roughly equal folds. Each
fold is used as a test set once, while the remaining \texttt{k-1}
folds are used for training. The average test performance across \texttt{k-1}
folds provides a more robust estimate of the model's generalization
ability because every data point has served as test data exactly once.
Although it is typically more computationally expensive than a single
hold-out split: since one must train and evaluate the model \texttt{k}
times. This procedure is especially valuable when the dataset is small
and the goal is to make the most efficient use of available data while
still obtaining a stable measure of test performance.

Deciding which method to adopt, hold-out or \texttt{k}-fold cross-validation,
generally depends on the size of the dataset, the computational costs
of training, and the desired precision in estimating generalization
error. When ample data are available and training the model is computationally
demanding, a single hold-out split (with or without repeated runs)
is often sufficient. In contrast, when the dataset is relatively small
or when a more reliable performance estimate is necessary, \texttt{k}-fold
cross-validation is typically preferred. Both of these real-data generalizations
of Algorithm \ref{alg:Simulation-algorithm-for} thus serve to replace
synthetic data generation with proven data-splitting or re-sampling
techniques, ensuring that model performance can be assessed appropriately
in practical applications. \FloatBarrier $ $ {\footnotesize{}{}{}}
\begin{algorithm}
\begin{itemize}
\item \textbf{\footnotesize{}{}{}Input:}{\footnotesize{}{}{} Original
dataset $\mathcal{D}=\{\bm{X},\bm{y}\}$ of size $N$, number of runs
}\texttt{\footnotesize{}{}num\_runs}{\footnotesize{}{}, number of
epochs }\texttt{\footnotesize{}{}num\_epochs}{\footnotesize{}{},
penalty term $\lambda$, and choice of optimizer (e.g., Adam/SGD).}{\footnotesize\par}
\item \textbf{\footnotesize{}{}{}For}{\footnotesize{}{}{} each run {}}\textbf{\footnotesize{}{}in}{\footnotesize{}{}
}\texttt{\footnotesize{}{}num\_runs}{\footnotesize{}{}:}{\footnotesize\par}
\begin{itemize}
\item {\footnotesize{}{}{}Generate training data $\bm{X}_{train}$ and
response values $\bm{y}_{train}=f_{k}(\bm{X}_{train})$ with noise
$N(0,\sigma_{\epsilon}^{2})$}{\footnotesize\par}
\item {\footnotesize{}{}{}Generate testing data $\bm{X}_{test}$ and response
values $\bm{y}_{test}=f_{k}(\bm{X}_{test})$ with noise $N(0,\sigma_{\epsilon}^{2})$}{\footnotesize\par}
\item {\footnotesize{}{}{}Initialize the model using a different random
seed (neural network or other function models)}{\footnotesize\par}
\item {\footnotesize{}{}{}Define the loss function (MSE) and the optimizer
(Adam/SGD)}{\footnotesize\par}
\item \textbf{\footnotesize{}{}For each}{\footnotesize{}{} epoch }\textbf{\footnotesize{}{}in}{\footnotesize{}{}
}\texttt{\footnotesize{}{}num\_epochs}{\footnotesize{}{}:}{\footnotesize\par}
\begin{itemize}
\item {\footnotesize{}{}{}Perform a forward pass/fitting of the model
with the training data $\bm{X}_{\mathrm{train}}$ and $\bm{y}_{\mathrm{train}}$.}{\footnotesize\par}
\item {\footnotesize{}{}{}Calculate the training loss with a penalty term
$\lambda\cdot\bm{I}$ using model prediction $\hat{\bm{y}}_{train}$
and $\bm{y}_{train}$}{\footnotesize\par}

\item {\footnotesize{}{}{}Perform a forward pass/prediction with the testing
data (without gradient computation)}{\footnotesize\par}
\item {\footnotesize{}{}Calculate the test loss: $\mathcal{L}_{\mathrm{test}}\;=\;\mathrm{MSE}\Bigl(\hat{\bm{y}}_{\mathrm{test}},\bm{y}_{\mathrm{test}}\Bigr).$
} 
\end{itemize}
\end{itemize}
\item {\footnotesize{}{}{}Compute the average $\overline{\mathcal{L}}_{\mathrm{train}}$
and the average $\overline{\mathcal{L}}_{\mathrm{test}}$ over all
runs, as well as any variability measures.}{\footnotesize\par}
\end{itemize}
{\footnotesize{}{}{}\caption{\label{alg:Simulation-algorithm-for}Simulation algorithm for estimating
optimism in linear and ridge regression models. This algorithm computes
the average performance of different models (linear, hinge, and bended)
over a specified number of training epochs and runs, which is assessed
by the mean train and test losses. We use it to study how different
levels of noise in the training data (as shown in Figure \ref{fig:Noise_N_MCMC_grid})
and different signal complexities (controlled by parameter $k$) affect
the models' learning process (as shown in Figure \ref{fig:Opt_vs_epoch})
and their ability to generalize from training data to unseen test
data (as shown in Figure \ref{fig:Different-models-comparison}).\protect
\protect \protect \\
 }
}{\footnotesize\par}
\end{algorithm}

{\footnotesize{}{}{} } 
\begin{algorithm}
{\footnotesize{}{}\caption{\label{alg:Holdout}Hold-out generalization of Algorithm \ref{alg:Simulation-algorithm-for}
for real-data computation of optimism.}
}{\footnotesize\par}

\begin{itemize}
\item \textbf{\footnotesize{}{}{}Input:}{\footnotesize{}{}{} Original
dataset $\mathcal{D}=\{\bm{X},\bm{y}\}$ of size $N$, number of runs
}\texttt{\footnotesize{}{}num\_runs}{\footnotesize{}{}, number of
epochs }\texttt{\footnotesize{}{}num\_epochs}{\footnotesize{}{},
penalty term $\lambda$, and choice of optimizer (e.g., Adam/SGD).}{\footnotesize\par}

\item \textbf{\footnotesize{}{}{}For}{\footnotesize{}{}{} each run {}}\textbf{\footnotesize{}{}in}{\footnotesize{}{}
}\texttt{\footnotesize{}{}num\_runs}{\footnotesize{}{}:}{\footnotesize\par}

{\footnotesize{}{} }{\footnotesize\par}

\begin{itemize}
\item {\footnotesize{}{}Split $\mathcal{D}$ into $(\mathcal{D}_{\mathrm{train}},\mathcal{D}_{\mathrm{test}})$,
for example using an 80\%--20\% random split.}{\footnotesize\par}

{\footnotesize{}{} }{\footnotesize\par}

\item {\footnotesize{}{}Let $\bm{X}_{\mathrm{train}},\bm{y}_{\mathrm{train}}$
be the }\textbf{\footnotesize{}{}hold-out}{\footnotesize{}{} training
set $\mathcal{D}_{\mathrm{train}}$ and $\bm{X}_{\mathrm{test}},\bm{y}_{\mathrm{test}}$
be }\textbf{\footnotesize{}{}a bootstrap sample}{\footnotesize{}{}
from the test set $\mathcal{D}_{\mathrm{test}}$.}{\footnotesize\par}

{\footnotesize{}{} }{\footnotesize\par}

\item {\footnotesize{}{}Initialize the model using a different random seed
(neural network or other function models).}{\footnotesize\par}

{\footnotesize{}{} }{\footnotesize\par}

\item {\footnotesize{}{}Define the loss function (MSE) and the optimizer
(Adam/SGD).}{\footnotesize\par}

{\footnotesize{}{} }{\footnotesize\par}

\item \textbf{\footnotesize{}{}For each}{\footnotesize{}{} epoch }\textbf{\footnotesize{}{}in}{\footnotesize{}{}
}\texttt{\footnotesize{}{}num\_epochs}{\footnotesize{}{}:}{\footnotesize\par}

{\footnotesize{}{} }{\footnotesize\par}

\begin{itemize}
\item {\footnotesize{}{}Perform a forward pass/predicting of the model
with the training data $\bm{X}_{\mathrm{train}}$ and $\bm{y}_{\mathrm{train}}$.}{\footnotesize\par}

{\footnotesize{}{} }{\footnotesize\par}

\item {\footnotesize{}{}Calculate the training loss with a penalty term
$\lambda\cdot\bm{I}$ using model prediction $\hat{\bm{y}}_{train}$
and $\bm{y}_{train}$ }{\footnotesize\par}

\item {\footnotesize{}{}Perform backpropagation and update the model parameters
(via optimizer).}{\footnotesize\par}

{\footnotesize{}{} }{\footnotesize\par}

\end{itemize}
\item {\footnotesize{}{}After the final epoch, perform a forward pass with
$\bm{X}_{\mathrm{test}}$ (no gradient computation).}{\footnotesize\par}

{\footnotesize{}{} }{\footnotesize\par}

\item {\footnotesize{}{}Calculate the test loss: $\mathcal{L}_{\mathrm{test}}\;=\;\mathrm{MSE}\Bigl(\hat{\bm{y}}_{\mathrm{test}},\bm{y}_{\mathrm{test}}\Bigr).$ }{\footnotesize\par}

\item {\footnotesize{}{}Store or record $\mathcal{L}_{\mathrm{train}}$
and $\mathcal{L}_{\mathrm{test}}$.}{\footnotesize\par}

{\footnotesize{}{} }{\footnotesize\par}

\end{itemize}
\item {\footnotesize{}{}{}Compute the average $\overline{\mathcal{L}}_{\mathrm{train}}$
and the average $\overline{\mathcal{L}}_{\mathrm{test}}$ over all
runs, as well as any variability measures.}{\footnotesize\par}

{\footnotesize{}{} }{\footnotesize\par}

\end{itemize}
{\footnotesize{}{} }{\footnotesize\par}

\end{algorithm}

{\footnotesize{}{}} 
\begin{algorithm}
{\footnotesize{}{}\caption{\label{alg:KFold}\texttt{k}-Fold generalization of Algorithm \ref{alg:Simulation-algorithm-for}
for real-data computation of optimism.}
}{\footnotesize\par}

\begin{itemize}
\item \textbf{\footnotesize{}{}{}Input:}{\footnotesize{}{}{} Original
dataset $\mathcal{D}=\{\bm{X},\bm{y}\}$ of size $N$, number of runs
}\texttt{\footnotesize{}{}num\_runs}{\footnotesize{}{}, number of
epochs }\texttt{\footnotesize{}{}num\_epochs}{\footnotesize{}{},
number of folds }\texttt{\footnotesize{}{}k}{\footnotesize{}{},
penalty term $\lambda$, and choice of optimizer (e.g., Adam/SGD).}{\footnotesize\par}

{\footnotesize{}{}
}{\footnotesize\par}

\item \textbf{\footnotesize{}{}{}For}{\footnotesize{}{}{} each run {}}\textbf{\footnotesize{}{}in}{\footnotesize{}{}
}\texttt{\footnotesize{}{}num\_runs}{\footnotesize{}{}:}{\footnotesize\par}

{\footnotesize{}{} }{\footnotesize\par}

\begin{itemize}
\item {\footnotesize{}{}Partition $\mathcal{D}$ into }\texttt{\footnotesize{}{}k}{\footnotesize{}{}
folds of (approximately) equal sizes that are disjoint. Let }\textbf{\footnotesize{}{}one
fold}{\footnotesize{}{} be $\mathcal{D}_{\mathrm{test}}$ = $(\bm{X}_{\mathrm{test}},\bm{y}_{\mathrm{test}})$. }{\footnotesize\par}

{\footnotesize{}{} }{\footnotesize\par}

\item \textbf{\footnotesize{}{}For each}{\footnotesize{}{} fold }\textbf{\footnotesize{}{}in}{\footnotesize{}{}
the rest }\texttt{\footnotesize{}{}k-1}{\footnotesize{}{} folds:}{\footnotesize\par}

{\footnotesize{}{} }{\footnotesize\par}

\begin{itemize}
\item {\footnotesize{}{}Let $\bm{X}_{\mathrm{train}},\bm{y}_{\mathrm{train}}$
be the }\textbf{\footnotesize{}{}fixed fold}{\footnotesize{}{} training
set $\mathcal{D}_{\mathrm{train}}$ and use the current fold of the
remaining }\texttt{\footnotesize{}{}k}{\footnotesize{}{}-1 folds
into $\mathcal{D}_{\mathrm{train}}$ = $(\bm{X}_{\mathrm{train}},\bm{y}_{\mathrm{train}})$.}{\footnotesize\par}

{\footnotesize{}{} }{\footnotesize\par}

\item {\footnotesize{}{}Initialize the model using a different random seed
(neural network or other function models).}{\footnotesize\par}

{\footnotesize{}{} }{\footnotesize\par}

\item {\footnotesize{}{}Define the loss function (MSE) and the optimizer
(Adam/SGD).}{\footnotesize\par}

{\footnotesize{}{} }{\footnotesize\par}

\item \textbf{\footnotesize{}{}For each}{\footnotesize{}{} epoch }\textbf{\footnotesize{}{}in}{\footnotesize{}{}
}\texttt{\footnotesize{}{}num\_epochs}{\footnotesize{}{}:}{\footnotesize\par}

{\footnotesize{}{} }{\footnotesize\par}

\begin{itemize}
\item {\footnotesize{}{}Perform a forward pass/predicting of the model
with the training data $\bm{X}_{\mathrm{train}}$ and $\bm{y}_{\mathrm{train}}$.}{\footnotesize\par}

{\footnotesize{}{} }{\footnotesize\par}

\item {\footnotesize{}{}Calculate the training loss with a penalty term
$\lambda\cdot\bm{I}$ using model prediction $\hat{\bm{y}}_{train}$
and $\bm{y}_{train}$.}{\footnotesize\par}

{\footnotesize{}{} }{\footnotesize\par}

\item {\footnotesize{}{}Perform backpropagation and update the model parameters
(via optimizer).}{\footnotesize\par}

{\footnotesize{}{} }{\footnotesize\par}

\end{itemize}
\item {\footnotesize{}{}After the final epoch, perform a forward pass with
$\bm{X}_{\mathrm{test}}$ (no gradient computation).}{\footnotesize\par}

{\footnotesize{}{} }{\footnotesize\par}

\item {\footnotesize{}{}Calculate the test loss: $\mathcal{L}_{\mathrm{test}}\;=\;\mathrm{MSE}\Bigl(\hat{\bm{y}}_{\mathrm{test}},\bm{y}_{\mathrm{test}}\Bigr).$ }{\footnotesize\par}

\item {\footnotesize{}{}Store or record $\mathcal{L}_{\mathrm{train}}$
and $\mathcal{L}_{\mathrm{test}}$ for this fold.}{\footnotesize\par}

{\footnotesize{}{} }{\footnotesize\par}

\end{itemize}
\item {\footnotesize{}{}{}Compute the average $\overline{\mathcal{L}}_{\mathrm{train}}$
and the average $\overline{\mathcal{L}}_{\mathrm{test}}$ over all
runs, as well as any variability measures.}{\footnotesize\par}

{\footnotesize{}{} 
 }{\footnotesize\par}

\end{itemize}
{\footnotesize{}{}
}{\footnotesize\par}

\end{itemize}
{\footnotesize{}{} }{\footnotesize\par}

\end{algorithm}

$ $

\medskip{}

{\footnotesize{}{}{}} 
\begin{algorithm}
\begin{itemize}
\item {\footnotesize{}{}{}class SimpleNN(nn.Module): }{\footnotesize\par}
\item {\footnotesize{}{}{}def \_\_init\_\_(seed,self): }{\footnotesize\par}
\begin{itemize}
\item {\footnotesize{}{}{}super(SimpleNN, self).\_\_init\_\_() }{\footnotesize\par}
\item {\footnotesize{}{}{}nn.manual\_seed(seed) }{\footnotesize\par}
\item {\footnotesize{}{}{}self.layers = nn.Sequential( nn.Linear(1, 50),
nn.ReLU(), nn.Linear(50, 50), nn.ReLU(), nn.Linear(50, 1) ) }{\footnotesize\par}
\end{itemize}
\item {\footnotesize{}{}{}net = SimpleNN() }{\footnotesize\par}
\item {\footnotesize{}{}{}criterion = nn.MSELoss() }{\footnotesize\par}
\item {\footnotesize{}{}{}optimizer = optim.Adam(net.parameters(), lr=0.01)
or }\\
 {\footnotesize{}{}{}optim.SGD(net.parameters(), lr=0.01, momentum=0.9) }{\footnotesize\par}
\item \textbf{\footnotesize{}{}For each}{\footnotesize{}{} epoch }\textbf{\footnotesize{}{}in}{\footnotesize{}{}
}\texttt{\footnotesize{}{}num\_epochs}{\footnotesize{}{}:}{\footnotesize\par}
\begin{itemize}
\item {\footnotesize{}{}{}optimizer.zero\_grad() }{\footnotesize\par}
\item {\footnotesize{}{}{}outputs = net(train\_X) }{\footnotesize\par}
\item {\footnotesize{}{}{}loss = criterion(outputs, train\_y) }{\footnotesize\par}
\item {\footnotesize{}{}{}loss.backward() }{\footnotesize\par}
\item {\footnotesize{}{}{}optimizer.step() }{\footnotesize\par}
\item \textbf{\footnotesize{}{}{}with}{\footnotesize{}{}{} torch.no\_grad(): }{\footnotesize\par}
\begin{itemize}
\item {\footnotesize{}{}{}outputs\_test = net(test\_X) }{\footnotesize\par}
\item {\footnotesize{}{}{}loss\_test = criterion(outputs\_test, test\_y) }{\footnotesize\par}
\end{itemize}
\end{itemize}
\end{itemize}
{\footnotesize{}{}{}\caption{\label{alg:3-layer-neural-network}3-layer NN construction in python
using pytorch. The network consists of linear input layer, with ReLU
of 50 outputs; hidden layer with ReLU of 50 outputs; output layer
with ReLU of 1 output.}
}{\footnotesize\par}
\end{algorithm}

\medskip{}

Above algorithms \ref{alg:Simulation-algorithm-for}, \ref{alg:Holdout}
and \ref{alg:KFold} also works for NN and NTK fitting, like the one
described in Algorithm \ref{alg:3-layer-neural-network} and \ref{alg:Simulation-algorithm-NTK}.
These two methods usually performs layer-wise fitting. Instead of
updating all layers as in Algorithm \ref{alg:3-layer-neural-network},
only the bottom-layer weights $W$ and the top-layer weights $\bm{a}$
are considered as kernel parameters and optimized, which approximates
the feature learning and top-layer fitting in the NN. The NTK features
$Z_{1}$ emulate the learned features of the NN where the ridge regression
penalty on the NTK kernel matrix $Z_{2}$ approximates the NN's regularization,
capturing overparameterization effects inherent in wide networks.
We do not use any regularization, namely $\lambda=0$ in this setting.
By dynamically computing the NTK features and kernel, this code emulates
the training of a NN while leveraging the fixed NTK, consistent with
the theoretical behavior of wide NNs.

{\footnotesize{}{}{}} 
\begin{algorithm}
\begin{itemize}
\item \textbf{\footnotesize{}{}{}Input:}{\footnotesize{}{}{} Original
dataset $\mathcal{D}=\{\bm{X},\bm{y}\}$ of size $N$, number of runs
}\texttt{\footnotesize{}{}num\_runs}{\footnotesize{}{}, number of
epochs }\texttt{\footnotesize{}{}num\_epochs}{\footnotesize{}{},
penalty term $\lambda$, and choice of optimizer (e.g., Adam/SGD).}{\footnotesize\par}
\item \textbf{\footnotesize{}{}{}For}{\footnotesize{}{}{} each run }\textbf{\footnotesize{}{}{}in}{\footnotesize{}{}{}
}\texttt{\footnotesize{}{}num\_runs}{\footnotesize{}{}: }{\footnotesize\par}
\begin{itemize}
\item {\footnotesize{}{}{}Generate training data $\bm{X}_{train}$ and
response values $\bm{y}_{train}=f_{k}(\bm{X}_{train})$ with noise
$N(0,\sigma_{\epsilon}^{2})$ }{\footnotesize\par}
\item {\footnotesize{}{}{}Generate testing data $\bm{X}_{test}$ and response
values $\bm{y}_{test}=f_{k}(\bm{X}_{test})$ with noise $N(0,\sigma_{\epsilon}^{2})$ }{\footnotesize\par}
\item {\footnotesize{}{}{}Initialize the model using a different random
seed }{\footnotesize\par}
\item {\footnotesize{}{}{}Initialize $W$: Bottom-layer weights; $\bm{a}$:
Top-layer weights both as i.i.d. $N(0,1)$ }{\footnotesize\par}
\item {\footnotesize{}{}{}Define the loss function (MSE) and the optimizer
(Adam/SGD) }{\footnotesize\par}
\item \textbf{\footnotesize{}{}{}For}{\footnotesize{}{}{} each epoch
}\textbf{\footnotesize{}{}{}in}{\footnotesize{}{}{} epochs: }{\footnotesize\par}
\begin{itemize}
\item {\footnotesize{}{}{}Perform a forward pass of the model with the
training data }{\footnotesize\par}
\begin{itemize}
\item {\footnotesize{}{}{}$Z_{1}$= ReLU($W^{T}\bm{X}_{train}$) as the
feature mapping $\phi$ }{\footnotesize\par}
\item {\footnotesize{}{}{}$Z_{2}$=$Z_{1}\cdot Z_{1}^{T}$ as the NTK
feature of the corresponding kernel $\Theta$ in \eqref{eq:NTK_kernel}.
}\\
\item {\footnotesize{}{}{}$\hat{\bm{y}}_{train}=Z_{1}\bm{a}$ }{\footnotesize\par}
\end{itemize}
\item {\footnotesize{}{}{}Calculate the training loss with a penalty term
$\lambda\cdot\text{trace}(Z_{2})$ using model prediction $\hat{\bm{y}}_{train}$
and $\bm{y}_{train}$ }{\footnotesize\par}
\item {\footnotesize{}{}{}Perform backpropagation and update the model
parameters $W,\bm{a}$ }{\footnotesize\par}
\item {\footnotesize{}{}{}Perform a forward pass with the testing data
(without gradient computation) }{\footnotesize\par}
\item {\footnotesize{}{}Calculate the test loss: $\mathcal{L}_{\mathrm{test}}\;=\;\mathrm{MSE}\Bigl(\hat{\bm{y}}_{\mathrm{test}},\bm{y}_{\mathrm{test}}\Bigr).$
} 
\end{itemize}
\end{itemize}
\item {\footnotesize{}{}{}Compute the ${\mathcal{L}}_{\mathrm{train}}$
and the average $\overline{\mathcal{L}}_{\mathrm{test}}$ over all
runs, as well as any variability measures (e.g., standard deviation).}{\footnotesize\par}
\end{itemize}
{\footnotesize{}{}{}\caption{\label{alg:Simulation-algorithm-NTK}Simulation algorithm for estimating
optimism in kernel regression model with NTK. This algorithm computes
the average performance of kernel ridge regression models, specifically,
this NTK kernel corresponds to a NN consists of linear input layer,
with ReLU of 50 outputs; hidden layer with ReLU of 50 outputs; output
layer with ReLU of 1 output as in Algorithm \ref{alg:3-layer-neural-network}.
\protect \protect \protect \\
 }
}{\footnotesize\par}
\end{algorithm}


\section{\label{sec:enough MC size}Simulation Monte-Carlo Sample Sizes}

\textbf{MC simulation settings.} From different simulation settings
in Figure \ref{fig:Noise_N_MCMC_grid}, we report the final scaled
expected optimism estimated from each simulation. We observe that
the MC error for this experiment is not negligible, especially at
the initial stages of the training (i.e., when the number of epochs
is small). This is due to a large variance of the scaled optimism
but is also affected by the magnitude of the initialization weights
(i.e., weights of nodes in NN, $\alpha,\beta$ in \eqref{eq:linear_assumption}
and \eqref{eq:bended_assumption}). Further increasing the MC sample
size could resolve this issue, but requires significantly more compute
time. In our experiments, we notice that the improvement of the estimate
is relatively small after the MC sample size exceeds 10,000. It is
also observed that: when the noise variance is small, the model fit
is basically determined by the signal shape. Then the raw optimism
is relatively stable; when the noise variance is large ($>1$, not
shown), the model fit is basically determined by the noise shape.
If the true relationship between the variables is not linear or does
not follow the specified signal function, then the model would be
mis-specified.

We choose the MC sample size to be 10000, which seems to guarantee
the accuracy of estimated model optimism for the signal we considered.

\textbf{Trends in optimism versus epoches.} Figure \ref{fig:Opt_vs_epoch}
shows us that for different signals of different complexities $k$,
the same NN takes different number of epochs to attain convergence.
For different $k=0.0,0.1,\cdots,1.0$, the scaled expected optimism,
as a measure of overfitting, exhibits a distinct trend in contrast
to the linear and bended models. Initially, the optimism is minimal,
reflecting the random initialization of weights in the NN. Then the
optimism increases to a peak, during which the model is trained to
fit the training data. As the training concludes, the optimism stabilizes
and shows a better performance over the testing dataset, indicating
the attainment of a balance between model complexity and generalization.

We choose the max iteration to be 1000, which seems to guarantee the
convergence of model training for the sample sizes we considered.

\begin{figure}
\centering

\includegraphics[width=0.95\textwidth]{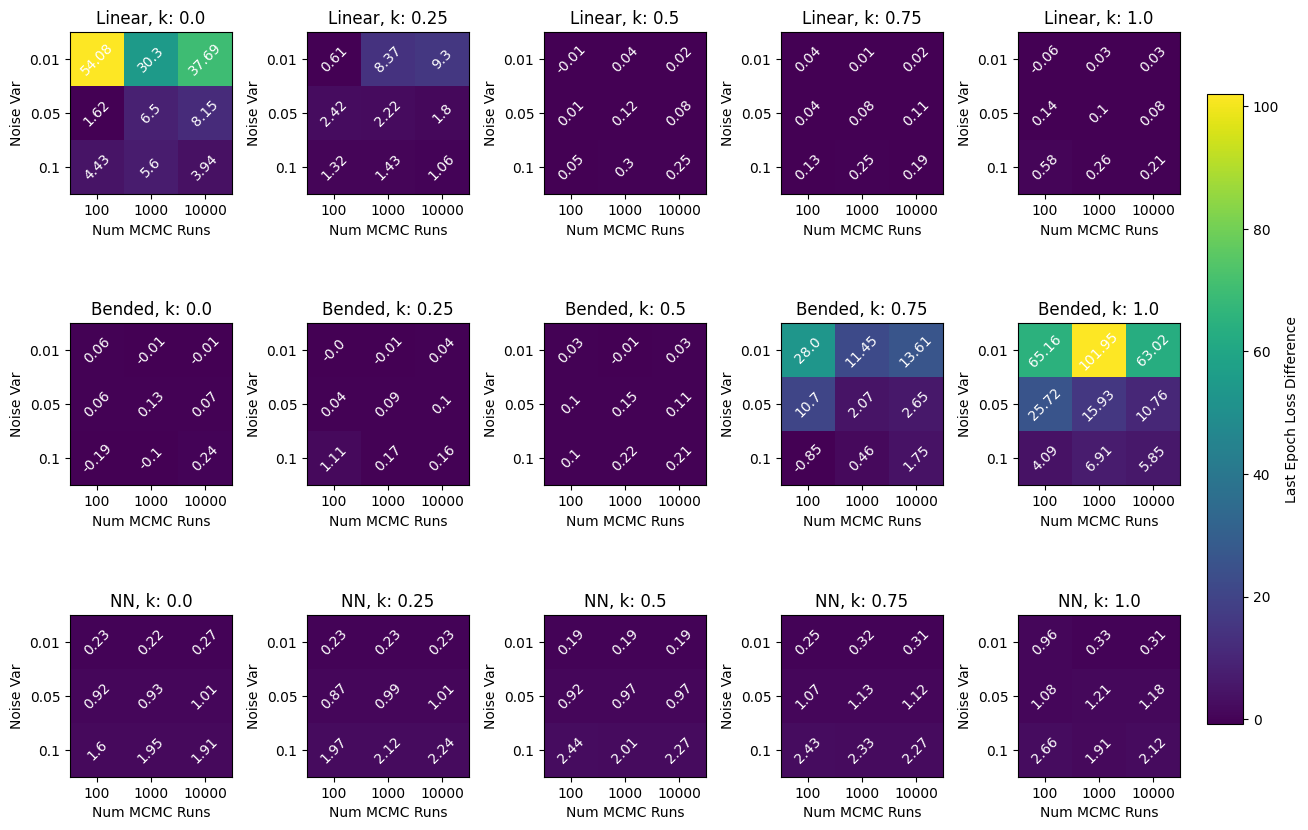}

\caption{\label{fig:Noise_N_MCMC_grid}Different rows of panels denote Linear,
Bended and NN models correspondingly; different columns of panels
denote $k=0.0,0.25,0.5,0.75,1.0$ which controls the shapes of signals.
In each panel, the x-axis is the changing number of MC runs, y-axis
is the noise variance added to the signal and the color indicates
the actual value (with text) of the (raw) optimism, i.e., the testing
minus training loss (at the last epoch).}
\end{figure}

\begin{figure}[ht]
\centering

\includegraphics[clip,width=0.95\textwidth,viewport=0bp 0bp 1073bp 777.902bp]{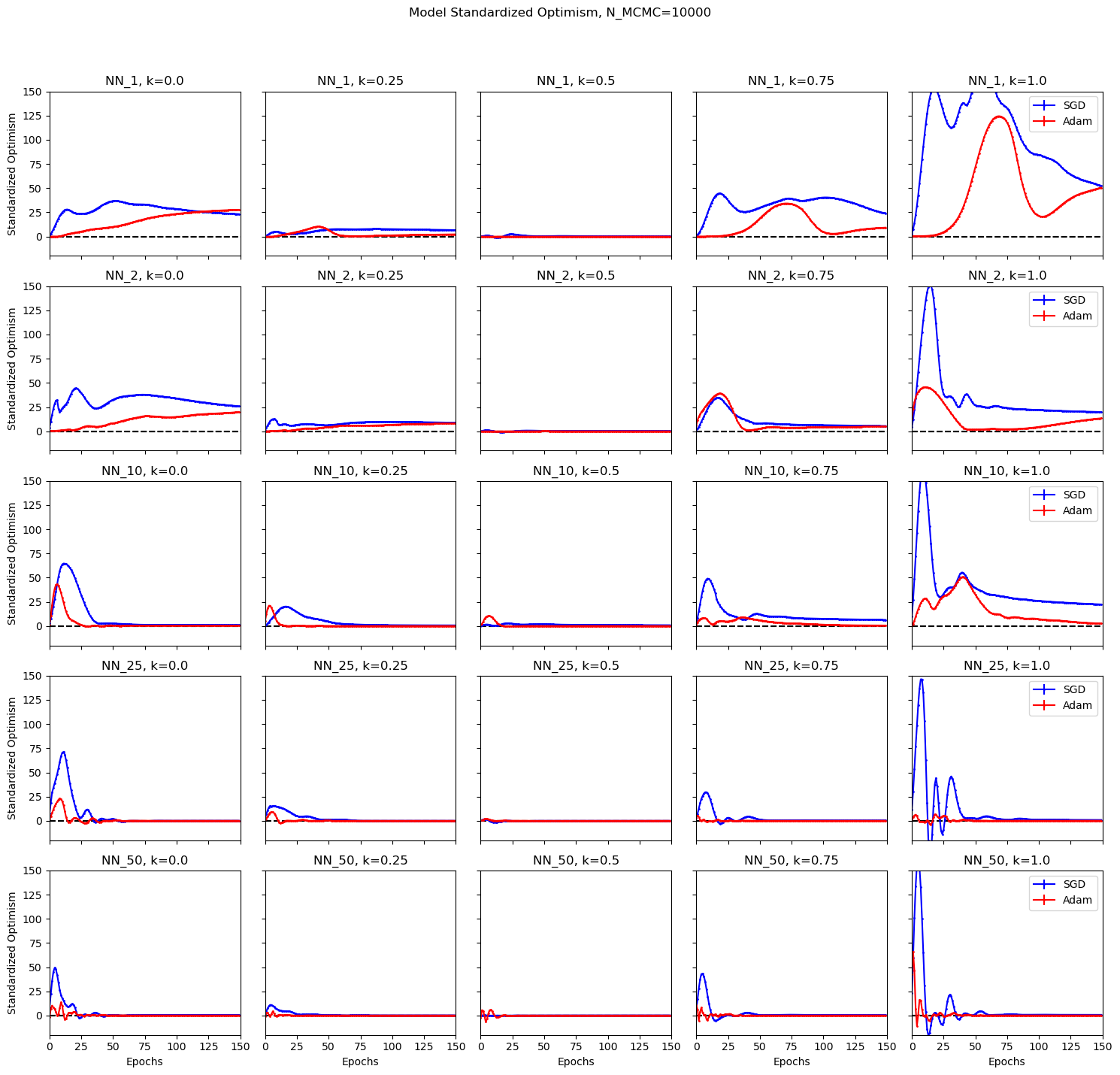}

\includegraphics[width=0.95\textwidth]{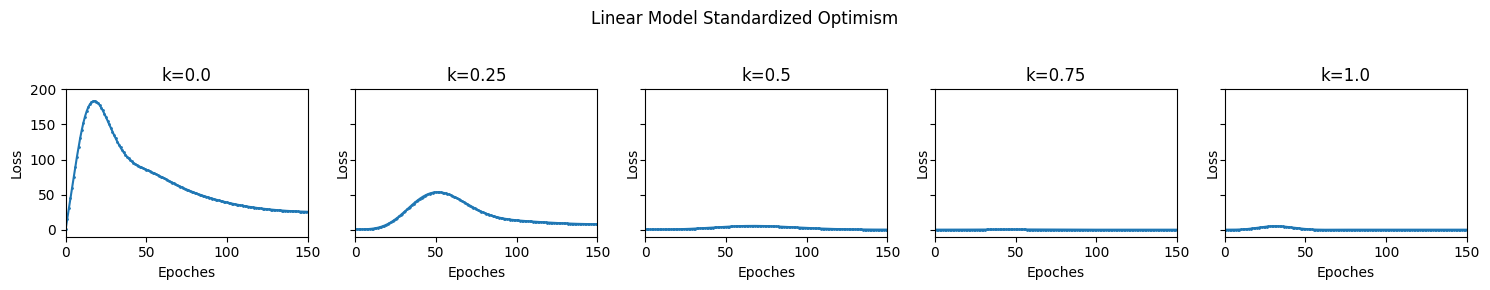}

\caption{\label{fig:Opt_vs_epoch}Expected Optimism (averaged from 10000 MC
simulations) versus the number of NN epochs for different $k$ in
\eqref{eq:test fun} with $\sigma_{\epsilon}^{2}=0.01$ for a training
set sampled from $N(0,1)$ of size 1000; and a testing set sampled
from $N(0,1)$ of size 1000. The network (with 2, 10, 25, 50 hidden
nodes) is trained with 1000 maximum epoches and reLU activation functions.
NNs are optimized via Adam optimizer with learning rate 0.01 or SGD
with learning rate 0.01 and momentum 0.9, and we provide the optimism
for the linear model for comparison.}
\end{figure}

\FloatBarrier

\section{\label{sec:Proof-for-Proposition-1}Proof of Proposition \ref{prop:Let}}

Recall the row vector notations $\bm{h}_{i}^{T}=\bm{x}_{i}^{T}\left(\bm{X}^{T}\bm{X}\right)^{-1}\bm{X}^{T}$
and $\bm{h}_{*}^{T}=\bm{x}_{*}^{T}\left(\bm{X}^{T}\bm{X}\right)^{-1}\bm{X}^{T}$
we defined and the fact that $\bm{x}_{i}^{T}\hat{\bm{\beta}}=\bm{h}_{i}^{T}\bm{y}=\hat{\mu}(\bm{x}_{i})$,
$\bm{x}_{*}^{T}\hat{\bm{\beta}}=\bm{h}_{*}^{T}\bm{y}=\hat{\mu}(\bm{x}_{*})$.
The in-sample optimism (or classical optimism) can be defined as\footnote{Note that $\mathbb{E}\bm{X}^{T}\bm{A}\bm{X}=\text{trace}(\bm{A}\text{Var}\bm{X})+(\mathbb{E}\bm{X})^{T}\bm{A}\mathbb{E}\bm{X}$},
\begin{align}
\text{Err}T_{\bm{X}} & \coloneqq\mathbb{E}_{\bm{y}\mid\bm{X}}T_{\bm{X}}=\frac{1}{n}\sum_{i=1}^{n}\mathbb{E}_{y_{i}\mid\bm{x}_{i}}\left\Vert y_{i}-\bm{x}_{i}^{T}\hat{\bm{\beta}}\right\Vert _{2}^{2}\nonumber \\
 & =\frac{1}{n}\sum_{i=1}^{n}\mathbb{E}_{y_{i}\mid\bm{x}_{i}}\left(y_{i}^{T}y_{i}-y_{i}^{T}\bm{h}_{i}^{T}\bm{y}-\bm{y}^{T}\bm{h}_{i}y_{i}+\bm{y}^{T}\bm{h}_{i}\bm{h}_{i}^{T}\bm{y}\right)\nonumber \\
 & =\sigma_{\epsilon}^{2}+\frac{1}{n}\sum_{i=1}^{n}\left(\mu(\bm{x}_{i})^{T}\mu(\bm{x}_{i})-2\mathbb{E}y_{i}^{T}\bm{h}_{i}^{T}\bm{y}+\sigma_{\epsilon}^{2}\text{trace \ensuremath{\left(\bm{h}_{i}\bm{h}_{i}^{T}\right)}}+\bm{\mu}(\bm{X})^{T}\bm{h}_{i}\bm{h}_{i}^{T}\bm{\mu}(\bm{X})\right)\nonumber \\
 & =\sigma_{\epsilon}^{2}\left(\text{trace \ensuremath{\left(\frac{1}{n}\bm{I}+\frac{1}{n}\bm{H}^{T}\bm{H}\right)}}\right)-\frac{2}{n}\sum_{i=1}^{n}\left(\mathbb{E}y_{i}^{T}\bm{h}_{i}^{T}\bm{y}\right)+\frac{1}{n}\bm{\mu}(\bm{X})^{T}(\bm{I}+\bm{H}^{T}\bm{H})\bm{\mu}(\bm{X})
\end{align}

We define out-of-sample prediction error (testing error) and calculate
the corresponding optimism. 
\begin{align*}
\text{Err}R_{\bm{X},\bm{y}} & \coloneqq\mathbb{E}_{y_{*}\mid\bm{x}_{*}}\left\Vert y_{*}-\bm{x}_{*}^{T}\hat{\bm{\beta}}\right\Vert _{2}^{2}\\
y_{*}\mid\bm{x}_{*} & \sim N_{1}(\bm{x}_{*}^{T}\bm{\beta},\sigma_{\epsilon}^{2})\in\mathbb{R}^{1}
\end{align*}
and the expectation of this quantity is defined as $\text{Err}R_{\bm{X},\bm{x}_{*}}\coloneqq\mathbb{E}_{\bm{y},y_{*}\mid\bm{X},\bm{x}_{*}}\left(\text{Err}R_{\bm{X},\bm{y}}\right)$.

We have the following expression for optimism $\text{Opt }R_{\bm{X}}$.

\begin{align}
\text{Opt }R_{\bm{X}} & \coloneqq\text{Err}R_{\bm{X}}-\text{Err}T_{\bm{X}}\label{eq:optR_X}\\
 & =\sigma_{\epsilon}^{2}\left(1+\mathbb{E}_{\bm{x}_{*}}\|\bm{h}_{*}\|_{2}^{2}\right)+\mathbb{E}_{\bm{x}_{*}}\left\Vert \mu(\bm{x}_{*})-\bm{h}_{*}^{T}\bm{\mu}(\bm{X})\right\Vert _{2}^{2}\nonumber \\
 & -\sigma_{\epsilon}^{2}\left(\text{trace \ensuremath{\left(\frac{1}{n}\bm{I}+\frac{1}{n}\bm{H}^{T}\bm{H}\right)}}\right)+\frac{2}{n}\sum_{i=1}^{n}\left(\mathbb{E}y_{i}^{T}\bm{h}_{i}^{T}\bm{y}\right)-\frac{1}{n}\bm{\mu}(\bm{X})^{T}(\bm{I}+\bm{H}^{T}\bm{H})\bm{\mu}(\bm{X})\nonumber \\
 & =\mathbb{E}_{\bm{x}_{*}}\left\Vert \mu(\bm{x}_{*})-\bm{h}_{*}^{T}\bm{\mu}(\bm{X})\right\Vert _{2}^{2}-\frac{1}{n}\bm{\mu}(\bm{X})^{T}(\bm{I}+\bm{H}^{T}\bm{H})\bm{\mu}(\bm{X})\nonumber \\
 & +\sigma_{\epsilon}^{2}\left(\mathbb{E}_{\bm{x}_{*}}\|\bm{h}_{*}^{T}\|_{2}^{2}-\frac{1}{n}\text{trace \ensuremath{\left(\bm{H}^{T}\bm{H}\right)}}\right)\nonumber \\
 & +\frac{2}{n}\sum_{i=1}^{n}\left(\mathbb{E}y_{i}^{T}\bm{h}_{i}^{T}\bm{y}\right).\label{eq:optR_X_1}
\end{align}
To simplify this expression further, we notice that 
\begin{align}
\frac{1}{n}\left\Vert \bm{\mu}(\bm{X})-\bm{H}\bm{\mu}(\bm{X})\right\Vert _{2}^{2} & =\frac{1}{n}\left(\bm{\mu}(\bm{X})^{T}\bm{\mu}(\bm{X})-2\bm{\mu}(\bm{X})^{T}\bm{H}\bm{\mu}(\bm{X})+\bm{\mu}(\bm{X})^{T}\bm{H}^{T}\bm{H}\bm{\mu}(\bm{X})\right)\nonumber \\
 & \text{Note that }\bm{H}^{T}\bm{H}=\bm{H}\nonumber \\
 & =\frac{1}{n}\bm{\mu}(\bm{X})^{T}(\bm{I}+\bm{H}^{T}\bm{H})\bm{\mu}(\bm{X})-\frac{2}{n}\text{trace}\left(\bm{\mu}(\bm{X})^{T}\bm{H}\bm{\mu}(\bm{X})\right)\nonumber \\
 & =\frac{1}{n}\bm{\mu}(\bm{X})^{T}(\bm{I}+\bm{H}^{T}\bm{H})\bm{\mu}(\bm{X})-\frac{2}{n}\text{trace}\left(\bm{\beta}^{T}\bm{X}^{T}\bm{HX\beta}\right)\label{eq:useful_expression1}
\end{align}
and use the fact that $\mathbb{E}\bm{X}^{T}\bm{A}\bm{X}=\text{trace}(\bm{A}\text{Var}\bm{X})+(\mathbb{E}\bm{X})^{T}\bm{A}\mathbb{E}\bm{X}$,
\begin{align}
\frac{2}{n}\sum_{i=1}^{n}\left(\mathbb{E}y_{i}^{T}\bm{h}_{i}^{T}\bm{y}\right)-\frac{2}{n}\text{trace}\left(\bm{\mu}(\bm{X})^{T}\bm{H}\bm{\mu}(\bm{X})\right) & =\frac{2}{n}\sum_{i=1}^{n}\left(\mathbb{E}y_{i}^{T}\bm{h}_{i}^{T}\bm{y}\right)-\frac{2}{n}\text{trace}\left(\mathbb{E}\bm{y}^{T}\bm{H}\mathbb{E}\bm{y}\right)\nonumber \\
 & =\frac{1}{n}\text{trace}\left(2\bm{H}\right)\cdot\sigma_{\epsilon}^{2}.\label{eq:useful_expression2}
\end{align}
We can insert $-\frac{2}{n}\text{trace}\left(\bm{\beta}^{T}\bm{X}^{T}\bm{HX\beta}\right)+\frac{2}{n}\text{trace}\left(\bm{\beta}^{T}\bm{X}^{T}\bm{HX\beta}\right)$
into (\ref{eq:optR_X_1}) and get, 
\begin{align}
\text{Opt }R_{\bm{X}} & =\mathbb{E}_{\bm{x}_{*}}\left\Vert \mu(\bm{x}_{*})-\bm{h}_{*}^{T}\bm{\mu}(\bm{X})\right\Vert _{2}^{2}\nonumber \\
 & \underset{\text{using \eqref{eq:useful_expression1}}}{\underbrace{-\frac{1}{n}\bm{\mu}(\bm{X})^{T}(\bm{I}+\bm{H}^{T}\bm{H})\bm{\mu}(\bm{X})+\frac{2}{n}\text{trace}\left(\bm{\beta}^{T}\bm{X}^{T}\bm{HX\beta}\right)}}\nonumber \\
 & +\sigma_{\epsilon}^{2}\left(\mathbb{E}_{\bm{x}_{*}}\|\bm{h}_{*}^{T}\|_{2}^{2}-\frac{1}{n}\text{trace \ensuremath{\left(\bm{H}^{T}\bm{H}\right)}}\right)\nonumber \\
 & \underset{\text{using \eqref{eq:useful_expression2}}}{\underbrace{+\frac{2}{n}\sum_{i=1}^{n}\left(\mathbb{E}y_{i}^{T}\bm{h}_{i}^{T}\bm{y}\right)-\frac{2}{n}\text{trace}\left(\bm{\mu}(\bm{X})^{T}\bm{H}\bm{\mu}(\bm{X})\right)}}\label{eq:optR_X_2}\\
 & =\mathbb{E}_{\bm{x}_{*}}\left\Vert \mu(\bm{x}_{*})-\bm{h}_{*}^{T}\bm{\mu}(\bm{X})\right\Vert _{2}^{2}-\frac{1}{n}\left\Vert \bm{\mu}(\bm{X})-\bm{H}\bm{\mu}(\bm{X})\right\Vert _{2}^{2}\nonumber \\
 & +\sigma_{\epsilon}^{2}\left(\mathbb{E}_{\bm{x}_{*}}\|\bm{h}_{*}^{T}\|_{2}^{2}-\frac{1}{n}\text{trace \ensuremath{\left(\bm{H}^{T}\bm{H}\right)}}\right)\nonumber \\
 & +\frac{1}{n}\text{trace}\left(2\bm{H}\right)\cdot\sigma_{\epsilon}^{2}
\end{align}
which can be reduced into following familiar form 
\begin{align}
\text{Opt }R_{\bm{X}} & =\mathbb{E}_{\bm{x}_{*}}\left\Vert \mu(\bm{x}_{*})-\bm{h}_{*}^{T}\mu(\bm{X})\right\Vert _{2}^{2}-\frac{1}{n}\left\Vert \bm{\mu}(\bm{X})-\bm{H}\bm{\mu}(\bm{X})\right\Vert _{2}^{2}\nonumber \\
 & +\sigma_{\epsilon}^{2}\left(\mathbb{E}_{\bm{x}_{*}}\|\bm{h}_{*}^{T}\|_{2}^{2}-\frac{1}{n}\text{trace \ensuremath{\left(\bm{H}^{T}\bm{H}\right)}}+\frac{1}{n}\text{trace}\left(2\bm{H}\right)\right).\label{eq:optimism_decomposition}
\end{align}
where the $\text{Opt }R_{\bm{X}}$ can be split into two parts as
shown in the main text: 
\begin{align*}
\text{signal part: } & \mathbb{E}_{\bm{x}_{*}}\left\Vert \mu(\bm{x}_{*})-\bm{h}_{*}^{T}\mu(\bm{X})\right\Vert _{2}^{2}-\frac{1}{n}\left\Vert \bm{\mu}(\bm{X})-\bm{H}\bm{\mu}(\bm{X})\right\Vert _{2}^{2}\\
\text{noise part: } & \mathbb{E}_{\bm{x}_{*}}\|\bm{h}_{*}^{T}\|_{2}^{2}-\frac{1}{n}\text{trace \ensuremath{\left(\bm{H}^{T}\bm{H}\right)}}+\frac{1}{n}\text{trace}\left(2\bm{H}\right)
\end{align*}

\section{\label{sec:Proof-of-the-main-THM}Proof of Proposition \ref{prop:(Positivity)-The-testing} }

Consider an empirical risk minimization prediction rule $\hat{\mu}$
over $\mathcal{F}_{n}$, the model fitted on training data is defined
as 
\begin{equation}
\hat{\mu}_{\text{train}}=\arg\min_{f\in\mathcal{F}_{n}}\frac{1}{n_{\text{train}}}\sum_{i=1}^{n_{\text{train}}}\ell(f(\bm{x}_{i}),y_{i}).\label{eq:trained_model_on_training_data}
\end{equation}
where the loss function is taken as the $L_{2}$ loss function $\ell(x,x')=\|x-x'\|_{2}^{2}$.
If the training data $\{\bm{x}_{i},y_{i}\}_{i=1}^{n}$ and the testing
data $\left\{ \bm{x}_{*,i},y_{*,i}\right\} _{i=1}^{n}$ follow the
same distribution, then \eqref{eq:trained_model_on_training_data}
and \eqref{eq:trained_model_on_test_data} define the same solution.

We also need to define the model fitted using the testing set as follows:
\begin{align}
\hat{\mu}_{\text{test}} & =\arg\min_{f\in\mathcal{F}_{n}}\frac{1}{n}\sum_{i=1}^{n}\ell(f(\bm{x}_{*,i}),y_{*,i}).\label{eq:trained_model_on_test_data}
\end{align}

Here, we assume that the model space $\mathcal{F}_{n}$ only depends
on the training sample size $n=n_{\text{train}}$, and does not depend
on the training data $(\bm{X}_{n},\bm{y}_{n})=\{\bm{x}_{i},y_{i}\}_{i=1}^{n}$.
We want to show that its testing error $\text{Err}R_{\bm{X}}\coloneqq\mathbb{E}_{y_{*}\mid\bm{x}_{*}}\left\Vert y_{*}-\bm{x}_{*}^{T}\hat{\bm{\beta}}\right\Vert _{2}^{2}$
is no smaller than its training error $\text{Err}T_{\bm{X}}\coloneqq\frac{1}{n}\sum_{i=1}^{n}\mathbb{E}_{y_{i}\mid\bm{x}_{i}}\left\Vert y_{i}-\bm{x}_{i}^{T}\hat{\bm{\beta}}\right\Vert _{2}^{2}$,
i.e., we want to prove 
\begin{align}
\text{Err}R_{\bm{X}} & =\mathbb{E}_{\left\{ \bm{x}_{*,i},y_{*,i}\right\} _{i=1}^{n},(\bm{X}_{n},\bm{y}_{n})}\left(\frac{1}{n}\sum_{i=1}^{n}\ell(\hat{\mu}_{\text{train}}(\bm{x}_{*,i}),y_{*,i})\right)\nonumber \\
 & \geq\mathbb{E}_{(\bm{x}_{*},y_{*}),(\bm{X}_{n},\bm{y}_{n})}\frac{1}{n}\sum_{i=1}^{n}\ell(\hat{\mu}_{\text{train}}(\bm{x}_{i}),y_{i})=\text{Err}T_{\bm{X}}
\end{align}
The equality comes from the fact that we assume the same distribution
for the training and testing sets. For test data point $(\bm{x}_{*},y_{*})$,
we have 
\begin{align*}
\mathbb{E}_{(\bm{x}_{*},y_{*})}\ell(\hat{\mu}_{\text{train}}(\bm{x}_{*}),y_{*}) & =\frac{1}{n}\sum_{i=1}^{n}\mathbb{E}_{\bm{x}_{*,i},y_{*,i}}\ell(\hat{\mu}_{\text{train}}(\bm{x}_{*,i}),y_{*,i}).
\end{align*}
The equality follows from taking $n_{\text{train}}$ independent identical
copies $\bm{x}_{*,i},y_{*,i}$of $(\bm{x}_{*},y_{*})$. 
\begin{align}
\frac{1}{n}\sum_{i=1}^{n}\mathbb{E}_{\bm{x}_{*,i},y_{*,i}}\ell(\hat{\mu}_{\text{train}}(\bm{x}_{*,i}),y_{*,i}) & =\mathbb{E}_{\left\{ \bm{x}_{*,i},y_{*,i}\right\} _{i=1}^{n}}\left(\frac{1}{n}\sum_{i=1}^{n}\ell(\hat{\mu}_{\text{train}}(\bm{x}_{*,i}),y_{*,i})\right)\\
 & \geq\mathbb{E}_{\left\{ \bm{x}_{*,i},y_{*,i}\right\} _{i=1}^{n}}\left(\frac{1}{n}\sum_{i=1}^{n}\ell(\hat{\mu}_{\text{test}}(\bm{x}_{*,i}),y_{*,i})\right)\label{eq:interpolation_minimizer_inequality}\\
 & =\mathbb{E}_{\left\{ \bm{x}_{i},y_{i}\right\} _{i=1}^{n}}\left(\frac{1}{n}\sum_{i=1}^{n}\ell(\hat{\mu}_{\text{train}}(\bm{x}_{i}),y_{i})\right).\label{eq:proof_train2test}
\end{align}
The first equality comes from the fact that we assume the same distribution
for the training and testing sets. The inequality comes from the definition
of $\hat{\mu}_{\text{test}}$ in \eqref{eq:trained_model_on_training_data}
that it minimizes the loss among all possible functions in the functional
space $\mathcal{F}_{n}$. The last equality comes from the fact that
the training and testing dataset follow the same distribution and
the definitions in \eqref{eq:trained_model_on_training_data} 
.

Collecting above arguments, we have 
\begin{align*}
\mathbb{E}_{(\bm{x}_{*},y_{*})}\ell(\hat{\mu}_{\text{train}}(\bm{x}_{*}),y_{*}) & \geq\mathbb{E}_{\left\{ \bm{x}_{*,i},y_{*,i}\right\} _{i=1}^{n}}\left(\frac{1}{n}\sum_{i=1}^{n}\ell(\hat{\mu}_{\text{train}}(\bm{x}_{*,i}),y_{*,i})\right),
\end{align*}
and we can take expectation with respect to the training data $(\bm{X}_{n},\bm{y}_{n})=\{(\bm{x}_{i},y_{i})\}_{i=1}^{n}$,
yielding 
\begin{align*}
\mathbb{E}_{(\bm{x}_{*},y_{*}),(\bm{X}_{n},\bm{y}_{n})}\ell(\hat{\mu}_{\text{train}}(\bm{x}_{*}),y_{*}) & \geq\mathbb{E}_{\left\{ \bm{x}_{*,i},y_{*,i}\right\} _{i=1}^{n},(\bm{X}_{n},\bm{y}_{n})}\left(\frac{1}{n}\sum_{i=1}^{n}\ell(\hat{\mu}_{\text{test}}(\bm{x}_{*,i}),y_{*,i})\right).
\end{align*}
For the above proof to hold, we emphasize that in \eqref{eq:trained_model_on_training_data}
the functional space $\mathcal{F}_{n}$ must be the same and independent
of training and testing dataset, although they can vary with the sample
size $n=n_{\text{test}}$.

\section{\label{sec:Calculation-for}Calculation for \eqref{eq:test fun-1}}

When $k<0.5$, the calculation follows as

\begin{align*}
\left(\mathbb{E}_{\bm{X}}x\mu(x)\right)^{2} & =\left(\mathbb{E}_{\bm{X}}x\cdot\frac{0.5-k}{0.5}\max\left(0,x\right)\right)^{2}\\
 & =\left(\int_{0}^{\infty}(1-2k)x^{2}\frac{1}{\sqrt{2\pi}}\exp\left(-\frac{x^{2}}{2}\right)dx\right)^{2}=\frac{1}{4}(1-2k)^{2}\\
\mathbb{E}_{\bm{X}}x^{2}\mu(x)^{2} & =\mathbb{E}_{\bm{X}}x^{2}\left(\frac{0.5-k}{0.5}\max\left(0,x\right)\right)^{2}\\
 & =\int_{0}^{\infty}(1-2k)^{2}x^{4}\frac{1}{\sqrt{2\pi}}\exp\left(-\frac{x^{2}}{2}\right)dx=\frac{3}{2}(1-2k)^{2}\\
\mathbb{E}_{\bm{X}}x^{3}\mu(x) & =\mathbb{E}_{\bm{X}}x^{3}\left(\frac{0.5-k}{0.5}\max\left(0,x\right)\right)\\
 & =\int_{0}^{\infty}(1-2k)x^{4}\frac{1}{\sqrt{2\pi}}\exp\left(-\frac{x^{2}}{2}\right)dx=\frac{3}{2}(1-2k)\\
\mathbb{E}_{\bm{X}}x^{3}\mu(x_{i})\cdot\mathbb{E}_{\bm{X}}x^{'}\mu(x^{'}) & =\frac{3}{4}(1-2k)^{2}
\end{align*}
\begin{align*}
\mathbb{E}_{\bm{X}}\frac{n}{2\sigma_{\epsilon}^{2}}\cdot\text{Opt }R_{\bm{X}} & \asymp\frac{1}{2\sigma_{\epsilon}^{2}}\cdot\frac{3}{2}(1-2k)^{2}+1+o(1).
\end{align*}

When $k\geq0.5$, the calculation follows as 
\begin{align*}
\left(\mathbb{E}_{\bm{X}}x\mu(x)\right)^{2} & =\left(\mathbb{E}_{\bm{X}}x\cdot\frac{k-0.5}{0.5}(-x)\right)^{2}\\
 & =\left(\int_{-\infty}^{\infty}(1-2k)x^{2}\frac{1}{\sqrt{2\pi}}\exp\left(-\frac{x^{2}}{2}\right)dx\right)^{2}=(1-2k)^{2}\\
\mathbb{E}_{\bm{X}}x^{2}\mu(x)^{2} & =\mathbb{E}_{\bm{X}}x^{2}\left(\frac{0.5-k}{0.5}\max\left(0,x\right)\right)^{2}\\
 & =\int_{0}^{\infty}(1-2k)^{2}x^{4}\frac{1}{\sqrt{2\pi}}\exp\left(-\frac{x^{2}}{2}\right)dx=3(1-2k)^{2}\\
\mathbb{E}_{\bm{X}}x^{3}\mu(x) & =\mathbb{E}_{\bm{X}}x^{3}\left(\frac{0.5-k}{0.5}\max\left(0,x\right)\right)\\
 & =\int_{0}^{\infty}(1-2k)x^{4}\frac{1}{\sqrt{2\pi}}\exp\left(-\frac{x^{2}}{2}\right)dx=3(1-2k)\\
\mathbb{E}_{\bm{X}}x^{3}\mu(x_{i})\cdot\mathbb{E}_{\bm{X}}x^{'}\mu(x^{'}) & =3(1-2k)^{2}
\end{align*}
\begin{align*}
\mathbb{E}_{\bm{X}}\frac{n}{2\sigma_{\epsilon}^{2}}\cdot\text{Opt }R_{\bm{X}} & \asymp0+1+o(1).
\end{align*}

\section{\label{sec:New-Proof}Proof of Theorem \ref{thm:main theorem}}

For testing error, we know that the coefficient estimate for training
set $\bm{X}\in\mathbb{R}^{n\times d}$ and single testing point $\bm{x}_{*}\in\mathbb{R}^{d\times1}$
as a column vector. We adopt the following notations. 
\begin{align*}
\hat{\bm{\beta}} & =\left(\bm{X}^{T}\bm{X}\right)^{-1}\bm{X}^{T}\bm{y}=\hat{\bm{\Sigma}}^{-1}\hat{\bm{\eta}}\in\mathbb{R}^{d\times1},\\
\hat{\bm{\Sigma}} & =\frac{1}{n}\left(\bm{X}^{T}\bm{X}\right)\in\mathbb{R}^{d\times d},\\
\bm{\Sigma} & =\mathbb{E}_{\bm{X}}\hat{\bm{\Sigma}}=\mathbb{E}_{\bm{x}_{*}}\bm{x}_{*}\bm{x}_{*}^{T}\in\mathbb{R}^{d\times d},\\
\hat{\bm{\eta}} & =\frac{1}{n}\left(\bm{X}^{T}\bm{y}\right)\in\mathbb{R}^{d\times1},\\
\bm{\eta} & =\mathbb{E}_{\bm{X}}\hat{\bm{\eta}}=\mathbb{E}_{\bm{x}_{*}}\bm{x}_{*}y_{*}\in\mathbb{R}^{d\times1}.
\end{align*}
Here we use the $y(\bm{x}_{*})$ to denote the observed response value
$y$ at this single testing point $\bm{x}_{*}$ which is not necessarily
in the training set. Now consider an arbitrary pair $(\bm{x}_{*},y_{*})$
as an independent draw from the same distribution of $\bm{X},\bm{y}$
and the $L_{2}$ loss function: 
\begin{align}
 & \mathbb{E}_{\bm{x}_{*}}\left\Vert y_{*}-\bm{x}_{*}^{T}\hat{\bm{\beta}}\right\Vert _{2}^{2}\nonumber \\
= & \mathbb{E}_{\bm{x}_{*}}\left\Vert y_{*}-\bm{x}_{*}^{T}\bm{\Sigma}^{-1}\bm{\eta}+\bm{x}_{*}^{T}\bm{\Sigma}^{-1}\bm{\eta}-\bm{x}_{*}^{T}\hat{\bm{\beta}}\right\Vert _{2}^{2}\nonumber \\
= & \mathbb{E}_{\bm{x}_{*}}\left\Vert \left(y_{*}-\bm{x}_{*}^{T}\bm{\Sigma}^{-1}\bm{\eta}\right)+\bm{x}_{*}^{T}\left(\bm{\Sigma}^{-1}\bm{\eta}-\hat{\bm{\Sigma}}^{-1}\hat{\bm{\eta}}\right)\right\Vert _{2}^{2}\nonumber \\
= & \mathbb{E}_{\bm{x}_{*}}\left(y_{*}-\bm{x}_{*}^{T}\bm{\Sigma}^{-1}\bm{\eta}\right)^{2}+\mathbb{E}_{\bm{x}_{*}}\left[\bm{x}_{*}^{T}\left(\bm{\Sigma}^{-1}\bm{\eta}-\hat{\bm{\Sigma}}^{-1}\hat{\bm{\eta}}\right)\right]^{2}\nonumber \\
+ & 2\mathbb{E}_{\bm{x}_{*}}\left(y_{*}-\bm{x}_{*}^{T}\bm{\Sigma}^{-1}\bm{\eta}\right)^{T}\cdot\bm{x}_{*}^{T}\left(\bm{\Sigma}^{-1}\bm{\eta}-\hat{\bm{\Sigma}}^{-1}\hat{\bm{\eta}}\right)\label{eq:newp_1}
\end{align}
where we observe that in \eqref{eq:newp_1} has a quadratic term 
\begin{align}
 & \mathbb{E}_{\bm{x}_{*}}\left[\bm{x}_{*}^{T}\left(\bm{\Sigma}^{-1}\bm{\eta}-\hat{\bm{\Sigma}}^{-1}\hat{\bm{\eta}}\right)\right]^{2}\nonumber \\
= & \mathbb{E}_{\bm{x}_{*}}\left(\bm{\Sigma}^{-1}\bm{\eta}-\hat{\bm{\Sigma}}^{-1}\hat{\bm{\eta}}\right)^{T}\bm{x}_{*}\bm{x}_{*}^{T}\left(\bm{\Sigma}^{-1}\bm{\eta}-\hat{\bm{\Sigma}}^{-1}\hat{\bm{\eta}}\right)\nonumber \\
= & \left(\bm{\Sigma}^{-1}\bm{\eta}-\hat{\bm{\Sigma}}^{-1}\hat{\bm{\eta}}\right)^{T}\bm{\Sigma}\left(\bm{\Sigma}^{-1}\bm{\eta}-\hat{\bm{\Sigma}}^{-1}\hat{\bm{\eta}}\right),
\end{align}
but the cross-product term in \eqref{eq:newp_1} vanishes due to the
fact that $\left(\bm{\eta}^{T}-\bm{\eta}^{T}\bm{\Sigma}^{-1}\bm{\Sigma}\right)=0$
under expectation with respect to the new observations $(\bm{x}_{*},y_{*})$:
\begin{align}
 & \mathbb{E}_{\bm{x}_{*}}\left(y_{*}-\bm{x}_{*}^{T}\bm{\Sigma}^{-1}\bm{\eta}\right)^{T}\cdot\bm{x}_{*}^{T}\left(\bm{\Sigma}^{-1}\bm{\eta}-\hat{\bm{\Sigma}}^{-1}\hat{\bm{\eta}}\right)\nonumber \\
= & \mathbb{E}_{\bm{x}_{*}}\left(y_{*}\bm{x}_{*}^{T}-\bm{\eta}^{T}\bm{\Sigma}^{-1}\bm{x}_{*}\bm{x}_{*}^{T}\right)\left(\bm{\Sigma}^{-1}\bm{\eta}-\hat{\bm{\Sigma}}^{-1}\hat{\bm{\eta}}\right)\nonumber \\
= & \left(\mathbb{E}_{\bm{x}_{*}}\bm{x}_{*}^{T}y_{*}-\bm{\eta}^{T}\bm{\Sigma}^{-1}\mathbb{E}_{\bm{x}_{*}}\bm{x}_{*}\bm{x}_{*}^{T}\right)\left(\bm{\Sigma}^{-1}\bm{\eta}-\hat{\bm{\Sigma}}^{-1}\hat{\bm{\eta}}\right)\nonumber \\
= & \underset{=0}{\underbrace{\left(\bm{\eta}^{T}-\bm{\eta}^{T}\bm{\Sigma}^{-1}\bm{\Sigma}\right)}}\left(\bm{\Sigma}^{-1}\bm{\eta}-\hat{\bm{\Sigma}}^{-1}\hat{\bm{\eta}}\right)\nonumber \\
= & 0.
\end{align}
Therefore if we take expectation with respect to training set, the
$\mathbb{E}_{\bm{X}}$\eqref{eq:newp_1} simplifies into \textcolor{red}{{}}
\begin{align}
\mathbb{E}_{\bm{X}}\eqref{eq:newp_1} & =\mathbb{E}_{\bm{x}_{*}}\left(y_{*}-\bm{x}_{*}^{T}\bm{\Sigma}^{-1}\bm{\eta}\right)^{2}+\mathbb{E}_{\bm{X}}\left(\bm{\Sigma}^{-1}\bm{\eta}-\hat{\bm{\Sigma}}^{-1}\hat{\bm{\eta}}\right)^{T}\bm{\Sigma}\left(\bm{\Sigma}^{-1}\bm{\eta}-\hat{\bm{\Sigma}}^{-1}\hat{\bm{\eta}}\right)\label{eq:newp_1B}\\
 & =\mathbb{E}_{\bm{x}_{*}}\left(y_{*}-\bm{x}_{*}^{T}\bm{\Sigma}^{-1}\bm{\eta}\right)^{2}+\underset{(1)}{\underbrace{\mathbb{E}_{\bm{X}}\left(\hat{\bm{\eta}}-\hat{\bm{\Sigma}}\bm{\Sigma}^{-1}\bm{\eta}\right)^{T}\bm{\Sigma}^{-1}\left(\hat{\bm{\eta}}-\hat{\bm{\Sigma}}\bm{\Sigma}^{-1}\bm{\eta}\right)}}+O_{p}\left(\frac{1}{n^{2}}\right).\label{eq:newp_2}
\end{align}
The step taken in \eqref{eq:newp_2} comes from the assumption that
$\left\Vert \hat{\bm{\eta}}-\bm{\eta}\right\Vert _{2}=O_{p}\left(\frac{1}{\sqrt{n}}\right),\left\Vert \hat{\bm{\Sigma}}-\bm{\Sigma}\right\Vert _{2}=O_{p}\left(\frac{1}{\sqrt{n}}\right)$
and the following manipulation of $\bm{\Sigma}^{-1}\bm{\eta}-\hat{\bm{\Sigma}}^{-1}\hat{\bm{\eta}}$
in \eqref{eq:new_q1}. First we observe that $\hat{\bm{\Sigma}}\bm{\Sigma}^{-1}-\bm{\Sigma}\hat{\bm{\Sigma}}^{-1}=O_{p}\left(\frac{1}{n}\right)$,
then 
\begin{align}
\left(\hat{\bm{\Sigma}}^{-1}-\bm{\Sigma}^{-1}\right) & =\hat{\bm{\Sigma}}^{-1}-\hat{\bm{\Sigma}}^{-1}\hat{\bm{\Sigma}}\bm{\Sigma}^{-1}\nonumber \\
 & =\hat{\bm{\Sigma}}^{-1}-\hat{\bm{\Sigma}}^{-1}\bm{\Sigma}\hat{\bm{\Sigma}}^{-1}+O_{p}\left(\frac{1}{n}\right)\nonumber \\
 & =\hat{\bm{\Sigma}}^{-1}-\hat{\bm{\Sigma}}^{-1}\left(\hat{\bm{\Sigma}}+O_{p}\left(\frac{1}{\sqrt{n}}\right)\right)\hat{\bm{\Sigma}}^{-1}+O_{p}\left(\frac{1}{n}\right)\nonumber \\
 & =O_{p}\left(\frac{1}{\sqrt{n}}\right).\label{eq:new_q0}
\end{align}
Then, we can estimate: 
\begin{align*}
\left(\hat{\bm{\eta}}-\hat{\bm{\Sigma}}\bm{\Sigma}^{-1}\bm{\eta}\right) & =\left(\left(\bm{\eta}+O_{p}\left(\frac{1}{\sqrt{n}}\right)\right)-\hat{\bm{\Sigma}}\left(\hat{\bm{\Sigma}}^{-1}+O_{p}\left(\frac{1}{\sqrt{n}}\right)\right)\bm{\eta}\right)\\
 & =O_{p}\left(\frac{1}{\sqrt{n}}\right).
\end{align*}
\begin{align}
\bm{\Sigma}^{-1}\bm{\eta}-\hat{\bm{\Sigma}}^{-1}\hat{\bm{\eta}} & =\hat{\bm{\Sigma}}^{-1}\left(\hat{\bm{\eta}}-\hat{\bm{\Sigma}}\bm{\Sigma}^{-1}\bm{\eta}\right)\nonumber \\
 & =\bm{\Sigma}^{-1}\left(\hat{\bm{\eta}}-\hat{\bm{\Sigma}}\bm{\Sigma}^{-1}\bm{\eta}\right)+\left(\hat{\bm{\Sigma}}^{-1}-\bm{\Sigma}^{-1}\right)\left(\hat{\bm{\eta}}-\hat{\bm{\Sigma}}\bm{\Sigma}^{-1}\bm{\eta}\right)\nonumber \\
 & =\bm{\Sigma}^{-1}\left(\hat{\bm{\eta}}-\hat{\bm{\Sigma}}\bm{\Sigma}^{-1}\bm{\eta}\right)+O_{p}\left(\frac{1}{n}\right).\label{eq:new_q1}
\end{align}
Now part (1) in \eqref{eq:newp_2} can be expanded using another arbitrary
pair $(\bm{x}_{*},y_{*})$ as an independent copy of $\bm{X},\bm{y}$:
\begin{align}
(1) & =\mathbb{E}_{\bm{X}}\left(\hat{\bm{\eta}}-\hat{\bm{\Sigma}}\bm{\Sigma}^{-1}\bm{\eta}\right)^{T}\bm{\Sigma}^{-1}\left(\hat{\bm{\eta}}-\hat{\bm{\Sigma}}\bm{\Sigma}^{-1}\bm{\eta}\right)\nonumber \\
 & =\mathbb{E}_{\bm{X}}\left[\hat{\bm{\eta}}-\frac{1}{n}\left(\bm{X}^{T}\bm{X}\right)\bm{\Sigma}^{-1}\bm{\eta}\right]^{T}\bm{\Sigma}^{-1}\left[\hat{\bm{\eta}}-\frac{1}{n}\left(\bm{X}^{T}\bm{X}\right)\bm{\Sigma}^{-1}\bm{\eta}\right]\nonumber \\
 & =\frac{1}{n}\mathbb{E}_{\bm{x}_{*}}\left[\bm{x}_{*}y_{*}-\left(\bm{x}_{*}\bm{x}_{*}^{T}\right)\bm{\Sigma}^{-1}\bm{\eta}\right]^{T}\bm{\Sigma}^{-1}\left[\bm{x}_{*}y_{*}-\left(\bm{x}_{*}\bm{x}_{*}^{T}\right)\bm{\Sigma}^{-1}\bm{\eta}\right]\nonumber \\
 & =\frac{1}{n}\mathbb{E}_{\bm{x}_{*}}\left(y_{*}-\bm{x}_{*}^{T}\bm{\Sigma}^{-1}\bm{\eta}\right)^{T}\left(\bm{x}_{*}^{T}\bm{\Sigma}^{-1}\bm{x}_{*}\right)\left(y_{*}-\bm{x}_{*}^{T}\bm{\Sigma}^{-1}\bm{\eta}\right)\nonumber \\
 & =\frac{1}{n}\mathbb{E}_{\bm{x}_{*}}\left(y_{*}-\bm{x}_{*}^{T}\bm{\Sigma}^{-1}\bm{\eta}\right)^{2}\left(\bm{x}_{*}^{T}\bm{\Sigma}^{-1}\bm{x}_{*}\right)\label{eq:newp_2A}
\end{align}
\textcolor{red}{{} }To sum up, plugging \eqref{eq:newp_2A} back
into \ref{eq:newp_2}: 
\begin{align}
\mathbb{E}_{\bm{X}}\mathbb{E}_{\bm{x}_{*}}\left\Vert y_{*}-\bm{x}_{*}^{T}\hat{\bm{\beta}}\right\Vert _{2}^{2} & =\mathbb{E}_{\bm{x}_{*}}\left(y_{*}-\bm{x}_{*}^{T}\bm{\Sigma}^{-1}\bm{\eta}\right)^{2}+\frac{1}{n}\mathbb{E}_{\bm{x}_{*}}\left(y_{*}-\bm{x}_{*}^{T}\bm{\Sigma}^{-1}\bm{\eta}\right)^{2}\left(\bm{x}_{*}^{T}\bm{\Sigma}^{-1}\bm{x}_{*}\right)+O_{p}\left(\frac{1}{n^{2}}\right).\label{eq:new_p_test_err}
\end{align}

For training error, we recall the definition of hat matrix $\bm{H}=\bm{X}\left(\bm{X}^{T}\bm{X}\right)^{-1}\bm{X}^{T}=\bm{X}\hat{\bm{\Sigma}}^{-1}\bm{X}^{T}$
and $\bm{H}^{T}\bm{H}=\bm{H}$, and take yet another arbitrary pair
$(\bm{x}_{*},y_{*})$ as an independent copy of $\bm{X},\bm{y}$:
\begin{align}
 & \frac{1}{n}\mathbb{E}_{\bm{X}}\left\Vert \bm{y}-\bm{X}\hat{\bm{\beta}}\right\Vert _{2}^{2}\nonumber \\
= & \frac{1}{n}\mathbb{E}_{\bm{X}}\bm{y}^{T}\left(\bm{I}-\bm{H}\right)\bm{y}\nonumber \\
= & \frac{1}{n}\mathbb{E}_{\bm{X}}\left(\bm{y}^{T}\bm{y}-n\cdot\hat{\bm{\eta}}^{T}\hat{\bm{\Sigma}}^{-1}\hat{\bm{\eta}}\right)\nonumber \\
= & \mathbb{E}_{\bm{x}_{*}}y_{*}^{2}-\mathbb{E}_{\bm{X}}\hat{\bm{\eta}}^{T}\hat{\bm{\Sigma}}^{-1}\hat{\bm{\eta}}\label{eq:newp_3}\\
= & \mathbb{E}_{\bm{x}_{*}}\left[\left(y_{*}-\bm{x}_{*}^{T}\bm{\Sigma}^{-1}\bm{\eta}\right)^{2}+2y_{*}\cdot\bm{x}_{*}^{T}\bm{\Sigma}^{-1}\bm{\eta}-\left(\bm{x}_{*}^{T}\bm{\Sigma}^{-1}\bm{\eta}\right)^{2}\right]\nonumber \\
 & -\mathbb{E}_{\bm{X}}\hat{\bm{\eta}}^{T}\hat{\bm{\Sigma}}^{-1}\hat{\bm{\eta}}.\label{eq:newp_4}
\end{align}
We can compute the expectation $\mathbb{E}_{\bm{x}_{*}}$ in \eqref{eq:newp_4},
where 
\begin{align*}
\mathbb{E}_{\bm{x}_{*}}\left(\bm{x}_{*}^{T}\bm{\Sigma}^{-1}\bm{\eta}\right)^{2} & =\mathbb{E}_{\bm{x}_{*}}\bm{x}_{*}^{T}\bm{\Sigma}^{-1}\bm{\eta}\bm{\eta}^{T}\bm{\Sigma}^{-1}\bm{x}_{*}\\
 & =\text{trace}\left(\bm{\Sigma}^{-1}\bm{\eta}\bm{\eta}^{T}\bm{\Sigma}^{-1}\cdot\bm{\Sigma}\right)+\mathbb{E}_{\bm{x}_{*}}\bm{x}_{*}\left(\bm{\Sigma}^{-1}\bm{\eta}\bm{\eta}^{T}\bm{\Sigma}^{-1}\right)\mathbb{E}_{\bm{x}_{*}}\bm{x}_{*}\\
 & =\bm{\eta}^{T}\bm{\Sigma}^{-1}\bm{\eta}.
\end{align*}
Then noticing that $\mathbb{E}_{\bm{x}_{*}}2y_{*}\cdot\bm{x}_{*}^{T}\bm{\Sigma}^{-1}\bm{\eta}=2\bm{\eta}^{T}\bm{\Sigma}^{-1}\bm{\eta}$
we can simplify 
\begin{align}
\eqref{eq:newp_4} & =\mathbb{E}_{\bm{x}_{*}}\left(y_{*}-\bm{x}_{*}^{T}\bm{\Sigma}^{-1}\bm{\eta}\right)^{2}+\bm{\eta}^{T}\bm{\Sigma}^{-1}\bm{\eta}-\underset{(2)}{\underbrace{\mathbb{E}_{\bm{X}}\hat{\bm{\eta}}^{T}\hat{\bm{\Sigma}}^{-1}\hat{\bm{\eta}}}}\label{eq:newp_5}
\end{align}
Now we want to compute the term (2) in \eqref{eq:newp_5}: 
\begin{align*}
(2)= & \mathbb{E}_{\bm{X}}\hat{\bm{\eta}}^{T}\hat{\bm{\Sigma}}^{-1}\hat{\bm{\eta}}\\
= & \mathbb{E}_{\bm{X}}\hat{\bm{\eta}}^{T}\hat{\bm{\Sigma}}^{-1}\hat{\bm{\Sigma}}\hat{\bm{\Sigma}}^{-1}\hat{\bm{\eta}}\\
= & \mathbb{E}_{\bm{X}}\left(\hat{\bm{\Sigma}}^{-1}\hat{\bm{\eta}}-\bm{\Sigma}^{-1}\bm{\eta}+\bm{\Sigma}^{-1}\bm{\eta}\right)^{T}\hat{\bm{\Sigma}}\left(\hat{\bm{\Sigma}}^{-1}\hat{\bm{\eta}}-\bm{\Sigma}^{-1}\bm{\eta}+\bm{\Sigma}^{-1}\bm{\eta}\right)\\
= & \mathbb{E}_{\bm{X}}\left(\hat{\bm{\Sigma}}^{-1}\hat{\bm{\eta}}-\bm{\Sigma}^{-1}\bm{\eta}\right)^{T}\hat{\bm{\Sigma}}\left(\hat{\bm{\Sigma}}^{-1}\hat{\bm{\eta}}-\bm{\Sigma}^{-1}\bm{\eta}\right)+\mathbb{E}_{\bm{X}}\bm{\eta}^{T}\bm{\Sigma}^{-1}\hat{\bm{\Sigma}}\bm{\Sigma}^{-1}\bm{\eta}\\
+ & 2\mathbb{E}_{\bm{X}}\bm{\eta}^{T}\bm{\Sigma}^{-1}\hat{\bm{\Sigma}}\left(\hat{\bm{\Sigma}}^{-1}\hat{\bm{\eta}}-\bm{\Sigma}^{-1}\bm{\eta}\right)\text{ and we use }\eqref{eq:new_q1},\\
= & \mathbb{E}_{\bm{X}}\left(\hat{\bm{\Sigma}}^{-1}\hat{\bm{\eta}}-\bm{\Sigma}^{-1}\bm{\eta}\right)^{T}\hat{\bm{\Sigma}}\left(\hat{\bm{\Sigma}}^{-1}\hat{\bm{\eta}}-\bm{\Sigma}^{-1}\bm{\eta}\right)+\mathbb{E}_{\bm{X}}\bm{\eta}^{T}\bm{\Sigma}^{-1}\hat{\bm{\Sigma}}\bm{\Sigma}^{-1}\bm{\eta}\\
+ & 2\bm{\eta}^{T}\left(\bm{I}+O_{p}\left(\frac{1}{\sqrt{n}}\right)\right)\left(\bm{\Sigma}^{-1}\underset{=0}{\underbrace{\mathbb{E}_{\bm{X}}\left(\hat{\bm{\eta}}-\hat{\bm{\Sigma}}\bm{\Sigma}^{-1}\bm{\eta}\right)}}+O_{p}\left(\frac{1}{n}\right)\right)\\
= & \mathbb{E}_{\bm{X}}\left(\hat{\bm{\Sigma}}^{-1}\hat{\bm{\eta}}-\bm{\Sigma}^{-1}\bm{\eta}\right)^{T}\hat{\bm{\Sigma}}\left(\hat{\bm{\Sigma}}^{-1}\hat{\bm{\eta}}-\bm{\Sigma}^{-1}\bm{\eta}\right)+\bm{\eta}^{T}\bm{\Sigma}^{-1}\bm{\eta}\\
+ & 2\mathbb{E}_{\bm{X}}\bm{\eta}^{T}\bm{\Sigma}^{-1}\hat{\bm{\Sigma}}\left[\bm{\Sigma}^{-1}\left(\hat{\bm{\eta}}-\hat{\bm{\Sigma}}\bm{\Sigma}^{-1}\bm{\eta}\right)\right]+O_{p}\left(\frac{1}{n^{3/2}}\right)\\
= & \mathbb{E}_{\bm{X}}\left(\hat{\bm{\Sigma}}^{-1}\hat{\bm{\eta}}-\bm{\Sigma}^{-1}\bm{\eta}\right)^{T}\hat{\bm{\Sigma}}\left(\hat{\bm{\Sigma}}^{-1}\hat{\bm{\eta}}-\bm{\Sigma}^{-1}\bm{\eta}\right)+\bm{\eta}^{T}\bm{\Sigma}^{-1}\bm{\eta}+O_{p}\left(\frac{1}{n^{3/2}}\right)\\
= & \mathbb{E}_{\bm{X}}\underset{=O_{p}\left(\frac{1}{n}\right)\text{ by }\eqref{eq:new_q1}}{\underbrace{\left(\hat{\bm{\Sigma}}^{-1}\hat{\bm{\eta}}-\bm{\Sigma}^{-1}\bm{\eta}\right)^{T}}}\left(\bm{\Sigma}\underset{=O_{p}\left(\frac{1}{\sqrt{n}}\right)\text{ by }\eqref{eq:new_q0}}{\underbrace{-\bm{\Sigma}+\hat{\bm{\Sigma}}}}\right)\underset{=O_{p}\left(\frac{1}{n}\right)\text{ by }\eqref{eq:new_q1}}{\underbrace{\left(\hat{\bm{\Sigma}}^{-1}\hat{\bm{\eta}}-\bm{\Sigma}^{-1}\bm{\eta}\right)^{T}}}+\bm{\eta}^{T}\bm{\Sigma}^{-1}\bm{\eta}+O_{p}\left(\frac{1}{n^{3/2}}\right)\\
= & \mathbb{E}_{\bm{X}}\left(\hat{\bm{\Sigma}}^{-1}\hat{\bm{\eta}}-\bm{\Sigma}^{-1}\bm{\eta}\right)^{T}\bm{\Sigma}\left(\hat{\bm{\Sigma}}^{-1}\hat{\bm{\eta}}-\bm{\Sigma}^{-1}\bm{\eta}\right)+\bm{\eta}^{T}\bm{\Sigma}^{-1}\bm{\eta}+O_{p}\left(\frac{1}{n^{3/2}}\right)\\
= & \bm{\eta}^{T}\bm{\Sigma}^{-1}\bm{\eta}+\frac{1}{n}\mathbb{E}_{\bm{x}_{*}}\left(y_{*}-\bm{x}_{*}^{T}\bm{\Sigma}^{-1}\bm{\eta}\right)^{2}\left(\bm{x}_{*}^{T}\bm{\Sigma}^{-1}\bm{x}_{*}\right)+O_{p}\left(\frac{1}{n^{3/2}}\right)\\
\end{align*}
where the last line follows the same argument as in \eqref{eq:newp_2A}:
\begin{align}
\frac{1}{n}\mathbb{E}_{\bm{X}}\left\Vert \bm{y}-\bm{X}\hat{\bm{\beta}}\right\Vert _{2}^{2} & =\mathbb{E}_{\bm{x}_{*}}\left(y_{*}-\bm{x}_{*}^{T}\bm{\Sigma}^{-1}\bm{\eta}\right)^{2}-\frac{1}{n}\mathbb{E}_{\bm{x}_{*}}\left(y_{*}-\bm{x}_{*}^{T}\bm{\Sigma}^{-1}\bm{\eta}\right)^{2}\left(\bm{x}_{*}^{T}\bm{\Sigma}^{-1}\bm{x}_{*}\right)\label{eq:new_p_train_err}\\
 & +O_{p}\left(\frac{1}{n^{3/2}}\right).\nonumber 
\end{align}

From \eqref{eq:def opt}, we take the difference between testing and
training error as optimism of the model: 
\begin{align*}
\mathbb{E}_{\bm{X}}\text{Opt }R_{\bm{X}} & \coloneqq\eqref{eq:new_p_test_err}-\eqref{eq:new_p_train_err}\\
 & =2\mathbb{E}_{\bm{x}_{*}}\left(y_{*}-\bm{x}_{*}^{T}\bm{\Sigma}^{-1}\bm{\eta}\right)^{2}\left(\bm{x}_{*}^{T}\bm{\Sigma}^{-1}\bm{x}_{*}\right)+O_{p}\left(\frac{1}{n^{3/2}}\right).
\end{align*}
\begin{align*}
\frac{n\mathbb{E}_{\bm{X}}\text{Opt }R_{\bm{X}}}{2\sigma_{\epsilon}^{2}}\sim & \frac{1}{\sigma_{\epsilon}^{2}}\cdot\mathbb{E}_{\bm{x}_{*}}\left(y_{*}-\bm{x}_{*}^{T}\bm{\Sigma}^{-1}\bm{\eta}\right)^{2}\left(\bm{x}_{*}^{T}\bm{\Sigma}^{-1}\bm{x}_{*}\right)+O\left(\frac{1}{n^{1/2}}\right)\\
\sim & \frac{1}{\sigma_{\epsilon}^{2}}\cdot\mathbb{E}_{\bm{x}_{*}}\left\Vert y_{*}-\bm{x}_{*}^{T}\bm{\Sigma}^{-1}\bm{\eta}\right\Vert ^{2}\left\Vert \bm{\Sigma}^{-1/2}\bm{x}_{*}\right\Vert ^{2}+O\left(\frac{1}{n^{1/2}}\right)
\end{align*}

Statement: Consider model $y_{*}=\mu(\bm{x}_{*})+\epsilon$ with an
independent additive noise $\mathbb{E}\bm{\epsilon}=0$ and a linear
function $\mu(\bm{x}_{*})=\bm{x}_{*}^{T}\bm{w}$ in $\bm{x}_{*}$.
Then, 
\begin{align*}
\frac{n\mathbb{E}_{\bm{X}}\text{Opt }R_{\bm{X}}}{2\sigma_{\epsilon}^{2}}\sim & \frac{1}{\sigma_{\epsilon}^{2}}\cdot\mathbb{E}_{\bm{x}_{*},\epsilon}\left(\mu(\bm{x}_{*})+\epsilon-\bm{x}_{*}^{T}\bm{\Sigma}^{-1}\bm{\eta}\right)^{2}\left(\bm{x}_{*}^{T}\bm{\Sigma}^{-1}\bm{x}_{*}\right)+O\left(\frac{1}{n^{1/2}}\right)\\
\sim & \frac{1}{\sigma_{\epsilon}^{2}}\cdot\mathbb{E}_{\bm{x}_{*},\epsilon}\left[\epsilon^{2}+\left(\mu(\bm{x}_{*})-\bm{x}_{*}^{T}\bm{\Sigma}^{-1}\bm{\eta}\right)^{2}\right]\left(\bm{x}_{*}^{T}\bm{\Sigma}^{-1}\bm{x}_{*}\right)+O\left(\frac{1}{n^{1/2}}\right)
\end{align*}
But we can observe that $\epsilon^{2}\sim\chi^{2}(1)$, then 
\begin{align*}
\mathbb{E}_{\bm{x}_{*},\epsilon}\epsilon^{2}\left(\bm{x}_{*}^{T}\bm{\Sigma}^{-1}\bm{x}_{*}\right) & =\mathbb{E}_{\epsilon}\epsilon^{2}\mathbb{E}_{\bm{x}_{*}}\left(\bm{x}_{*}^{T}\bm{\Sigma}^{-1}\bm{x}_{*}\right)\\
 & =1\cdot\left\{ \left(\mathbb{E}_{\bm{x}_{*}}\bm{x}_{*}\right)^{T}\cdot\bm{\Sigma}^{-1}\cdot\left(\mathbb{E}_{\bm{x}_{*}}\bm{x}_{*}\right)+\text{trace}\left(\bm{\Sigma}^{-1}Var\bm{x}_{*}\right)\right\} \\
 & =1\cdot(0+d)=d,
\end{align*}
therefore 
\begin{align*}
\frac{n\mathbb{E}_{\bm{X}}\text{Opt }R_{\bm{X}}}{2\sigma_{\epsilon}^{2}} & \sim\frac{1}{\sigma_{\epsilon}^{2}}\cdot\mathbb{E}_{\bm{x}_{*}}\left\Vert \mu(\bm{x}_{*})-\bm{x}_{*}^{T}\bm{\Sigma}^{-1}\bm{\eta}\right\Vert ^{2}\left\Vert \bm{\Sigma}^{-1/2}\bm{x}_{*}\right\Vert ^{2}+d+O\left(\frac{1}{n^{1/2}}\right).
\end{align*}
When there is an intercept term, we can repeat the above arguments
with augmented $\bm{X},\bm{x}_{*}$ (augmented by 1) and yield the
same result with $d$ replaced by $d+1$.

\newpage{}

\section{\label{sec:Proof-of-Corollary-reducing}Proof of Corollary \ref{cor:When-reducing-to-classical} }
\begin{proof}
Plug in the $y_{*}=y(\bm{x}_{*})=m(\bm{x}_{*})+\bm{\epsilon}$ back
into \eqref{eq:general_opt_expression}, we can take expectation first
with respect to $(\bm{x}_{*},\bm{\epsilon})$ (they are independent):
\begin{align*}
 & \mathbb{E}\left\Vert y_{*}-\bm{x}_{*}^{T}\bm{\Sigma}^{-1}\bm{\eta}\right\Vert _{2}^{2}\left\Vert \bm{\Sigma}^{-1/2}\bm{x}_{*}\right\Vert _{2}^{2}\\
 & =\mathbb{E}\left\Vert m(\bm{x}_{*})+\bm{\epsilon}-\bm{x}_{*}^{T}\bm{\Sigma}^{-1}\bm{\eta}\right\Vert _{2}^{2}\left\Vert \bm{\Sigma}^{-1/2}\bm{x}_{*}\right\Vert _{2}^{2}\\
 & =\mathbb{E}\left(\left\Vert m(\bm{x}_{*})-\bm{x}_{*}^{T}\bm{\Sigma}^{-1}\bm{\eta}\right\Vert _{2}^{2}+2\bm{\epsilon}^{T}\left(m(\bm{x}_{*})-\bm{x}_{*}^{T}\bm{\Sigma}^{-1}\bm{\eta}\right)+\left\Vert \bm{\epsilon}\right\Vert _{2}^{2}\right)\left\Vert \bm{\Sigma}^{-1/2}\bm{x}_{*}\right\Vert _{2}^{2}\\
 & =\mathbb{E}\left(\left\Vert m(\bm{x}_{*})-\bm{x}_{*}^{T}\bm{\Sigma}^{-1}\bm{\eta}\right\Vert _{2}^{2}+0+\sigma_{\epsilon}^{2}\cdot d\right)\left\Vert \bm{\Sigma}^{-1/2}\bm{x}_{*}\right\Vert _{2}^{2}
\end{align*}
where the last line follows from the fact that $\left\Vert \bm{\epsilon}\right\Vert _{2}^{2}$
is a chi-square distribution with degree of freedom $d$. 
\end{proof}

\section{\label{sec:Proof-of-Corollary}Proof of Corollary \ref{cor:When-1d-normal-1}}
\begin{proof}
Take $\bm{\xi}=0,\bm{\Sigma}=1$ and $d=1$ in \eqref{eq:general_opt_expression},
with the independent standard normal random variables $Z\sim N(0,1)$
$\left(Z^{2}\mathbb{E}Z\mu(Z)-Z\mu(Z)\right)^{2}$ 
\begin{align*}
\frac{n\mathbb{E}_{\bm{X}}\text{Opt }R_{\bm{X}}}{2\sigma_{\epsilon}^{2}}\sim & \frac{n}{\sigma_{\epsilon}^{2}}\cdot\mathbb{E}\left\Vert Z^{2}\mathbb{E}Z\mu(Z)-Z\mu(Z)\right\Vert _{2}^{2}+1+O\left(\frac{1}{n^{1/2}}\right)\\
= & \frac{n}{\sigma_{\epsilon}^{2}}\cdot\mathbb{E}\left\{ Z^{2}\mu(Z)^{2}-2Z^{3}\mu(Z)\mathbb{E}Z\mu(Z)+Z^{4}\left(\mathbb{E}Z\mu(Z)\right)^{2}\right\} \\
 & +1+O\left(\frac{1}{n^{1/2}}\right)\\
= & \frac{n}{\sigma_{\epsilon}^{2}}\cdot\left[\mathbb{E}Z^{2}\mu(Z)^{2}-2\mathbb{E}Z^{3}\mu(Z)\cdot\mathbb{E}Z\mu(Z)+\underset{=3}{\underbrace{\mathbb{E}Z^{4}}}\left(\mathbb{E}Z\mu(Z)\right)^{2}\right]+1+O\left(\frac{1}{n^{1/2}}\right)
\end{align*}
\end{proof}

\section{\label{sec:Computations-of-Special}Proof of Corollary \ref{cor:analytic_signal to linear model-1}}
\begin{proof}
We follow the same procedure 
\begin{align*}
\left(\mathbb{E}_{\bm{X}}x\mu(x)\right)^{2} & =\left(\mathbb{E}_{\bm{X}}\sum_{i=0}^{\infty}A_{i}x^{i+1}\right)^{2}=\left(\sum_{i\neq1}^{\infty}A_{i}\mathbb{E}_{\bm{X}}x^{i+1}+A_{1}\mathbb{E}_{\bm{X}}x^{2}\right)^{2}\\
\mathbb{E}_{\bm{X}}x^{2}\mu(x)^{2} & =\mathbb{E}_{\bm{X}}x^{2}\left(\sum_{i=0}^{\infty}A_{i}x^{i+1}\right)^{2}=\mathbb{E}_{\bm{X}}\left(\sum_{i=0,j=0}^{\infty}A_{i}A_{j}x^{i+j+2}\right)\\
 & =\mathbb{E}_{\bm{X}}\left(\left(\sum_{i\neq1}^{\infty}A_{i}\mathbb{E}_{\bm{X}}x^{i+1}\right)\left(\sum_{j\neq1}^{\infty}A_{j}\mathbb{E}_{\bm{X}}x^{j+1}\right)+2A_{1}\left(\sum_{j\neq1}^{\infty}A_{j}x^{j+3}\right)+2A_{1}^{2}x^{4}\right)\\
\mathbb{E}_{\bm{X}}x^{3}\mu(x) & =\mathbb{E}_{\bm{X}}\left(\sum_{i=0}^{\infty}A_{i}\mathbb{E}_{\bm{X}}x^{i+3}\right)=\sum_{i\neq1}^{\infty}A_{i}\mathbb{E}_{\bm{X}}x^{i+3}+A_{1}\mathbb{E}_{\bm{X}}x^{4}\\
\mathbb{E}_{\bm{X}}x^{3}\mu(x)\cdot\mathbb{E}_{\bm{X}}x^{'}\mu(x^{'}) & =\left(\sum_{i\neq1}^{\infty}A_{i}\mathbb{E}_{\bm{X}}x^{i+3}+A_{1}\mathbb{E}_{\bm{X}}x^{4}\right)\left(\sum_{i\neq1}^{\infty}A_{i}\mathbb{E}_{\bm{X}}x^{i+1}+A_{1}\mathbb{E}_{\bm{X}}x^{2}\right)
\end{align*}
Then, 
\begin{align*}
 & \mathbb{E}_{\bm{X}}\frac{n}{2\sigma_{\epsilon}^{2}}\cdot\text{Opt }R_{\bm{X}}\\
 & \asymp\frac{1}{2\sigma_{\epsilon}^{2}}\left\{ 6\left(\mathbb{E}_{\bm{X}}x_{i}\mu(x_{i})\right)^{2}+2\mathbb{E}_{\bm{X}}x_{i}^{2}\mu(x_{i})^{2}-4\mathbb{E}_{\bm{X}}x_{i}^{3}\mu(x_{i})\cdot\mathbb{E}_{\bm{X}}x_{\ell}\mu(x_{\ell})\right\} +1+o(1)\\
 & \asymp\frac{1}{2\sigma_{\epsilon}^{2}}\left\{ 6\cdot\left(\left(\sum_{i\neq1}^{\infty}A_{i}\mathbb{E}_{\bm{X}}x^{i+1}\right)^{2}+2A_{1}\left(\sum_{i\neq1}^{\infty}A_{i}\mathbb{E}_{\bm{X}}x^{i+1}\right)+A_{1}^{2}\right)\right.\\
 & +2\cdot\left(\mathbb{E}_{\bm{X}}\left(\sum_{i\neq1}^{\infty}A_{i}x^{i+1}\right)\left(\sum_{j\neq1}^{\infty}A_{j}x^{j+1}\right)+2\mathbb{E}_{\bm{X}}A_{1}x^{2}\left(\sum_{j\neq1}^{\infty}A_{j}x^{j+1}\right)+A_{1}^{2}\cdot\mathbb{E}_{\bm{X}}x^{4}\right)\\
 & -4\cdot\left(\mathbb{E}_{\bm{X}}\left(\sum_{i\neq1}^{\infty}A_{i}x^{i+3}\right)\mathbb{E}_{\bm{X}}\left(\sum_{j\neq1}^{\infty}A_{j}x^{j+1}\right)+\mathbb{E}_{\bm{X}}\left(\sum_{i\neq1}^{\infty}A_{i}x^{i+3}\right)\mathbb{E}_{\bm{X}}A_{1}x^{2}+\right.\\
 & -\left.\left.\mathbb{E}_{\bm{X}}A_{1}x^{4}\mathbb{E}_{\bm{X}}\left(\sum_{j\neq1}^{\infty}A_{j}\mathbb{E}_{\bm{X}}x^{j+1}\right)+\mathbb{E}_{\bm{X}}A_{1}x^{4}\mathbb{E}_{\bm{X}}A_{1}x^{2}\right)\right\} +1+o(1).
\end{align*}
\begin{align*}
 & \asymp\frac{1}{2\sigma_{\epsilon}^{2}}\left\{ 6\cdot\left(\left(\sum_{i\neq1}^{\infty}A_{i}\mathbb{E}_{\bm{X}}x^{i+1}\right)^{2}+2A_{1}\left(\sum_{i\neq1}^{\infty}A_{i}\mathbb{E}_{\bm{X}}x^{i+1}\right)+A_{1}^{2}\right)\right.\\
 & +2\cdot\left(\mathbb{E}_{\bm{X}}\left(\sum_{i\neq1}^{\infty}\sum_{j\neq1}^{\infty}A_{i}A_{j}x^{i+j+2}\right)+2A_{1}\mathbb{E}_{\bm{X}}\left(\sum_{j\neq1}^{\infty}A_{j}x^{j+3}\right)+3A_{1}^{2}\right)\\
 & -4\cdot\left(\mathbb{E}_{\bm{X}}\left(\sum_{i\neq1}^{\infty}A_{i}x^{i+3}\right)\mathbb{E}_{\bm{X}}\left(\sum_{j\neq1}^{\infty}A_{j}x^{j+1}\right)+\mathbb{E}_{\bm{X}}\left(\sum_{i\neq1}^{\infty}A_{i}x^{i+3}\right)A_{1}\right.\\
 & \left.+\left.3A_{1}\mathbb{E}_{\bm{X}}\left(\sum_{j\neq1}^{\infty}A_{j}\mathbb{E}_{\bm{X}}x^{j+1}\right)+3A_{1}^{2}\right)\right\} +1+o(1).
\end{align*}
Therefore, 
\begin{align*}
 & F(A_{i},i\neq1)\\
 & =6\left(\sum_{i\neq1}^{\infty}A_{i}\mathbb{E}_{\bm{X}}x^{i+1}\right)^{2}+2\left(\sum_{i\neq1}^{\infty}\sum_{j\neq1}^{\infty}A_{i}A_{j}\mathbb{E}_{\bm{X}}x^{i+j+2}\right)\\
 & -4\left(\sum_{i\neq1}^{\infty}A_{i}\mathbb{E}_{\bm{X}}x^{i+3}\right)\left(\sum_{j\neq1}^{\infty}A_{j}\mathbb{E}_{\bm{X}}x^{j+1}\right)\\
 & =6\left(\sum_{i\neq1}^{\infty}\sum_{j\neq1}^{\infty}A_{i}A_{j}\mathbb{E}_{\bm{X}}x^{i+1}\mathbb{E}_{\bm{X}}x^{j+1}\right)+2\left(\sum_{i\neq1}^{\infty}\sum_{j\neq1}^{\infty}A_{i}A_{j}\mathbb{E}_{\bm{X}}x^{i+j+2}\right)\\
 & -4\left(\sum_{i\neq1}^{\infty}\sum_{j\neq1}^{\infty}A_{i}A_{j}\mathbb{E}_{\bm{X}}x^{i+3}\mathbb{E}_{\bm{X}}x^{j+1}\right),
\end{align*}
by Stein's Lemma (See also Remark \ref{rem:If-Stein}), $\mathbb{E}_{\bm{X}}x^{i+3}=\mathbb{E}_{\bm{X}}x\cdot x^{i+2}=\mathbb{E}_{\bm{X}}(i+2)x^{i+1},$
\begin{align}
 & =\sum_{i\neq1}^{\infty}\sum_{j\neq1}^{\infty}\left(6A_{i}A_{j}-4A_{i}A_{j}(i+2)\right)\mathbb{E}_{\bm{X}}x^{i+1}\mathbb{E}_{\bm{X}}x^{j+1}+2A_{i}A_{j}\mathbb{E}_{\bm{X}}x^{i+j+2}\nonumber \\
 & =\sum_{i\neq1}^{\infty}\sum_{j\neq1}^{\infty}\left[\left(-2-4i\right)\mathbb{E}_{\bm{X}}x^{i+1}\mathbb{E}_{\bm{X}}x^{j+1}+2\mathbb{E}_{\bm{X}}x^{i+j+2}\right]A_{i}A_{j}\label{eq:F_A_i_moment_format-1}\\
 & =\sum_{i\neq1}^{\infty}\sum_{j\neq1}^{\infty}\left[\left(-4i\right)\mathbb{E}_{\bm{X}}x^{i+1}\mathbb{E}_{\bm{X}}x^{j+1}+2Cov\left(x^{i+1},x^{j+1}\right)\right]A_{i}A_{j}\nonumber \\
 & =\sum_{i\neq1}^{\infty}\left\{ \left[\left(-4i\right)\left(\mathbb{E}_{\bm{X}}x^{i+1}\right)^{2}+2Cov\left(x^{i+1},x^{i+1}\right)\right]A_{i}^{2}\right.\label{eq:F_A_i}\\
 & \left.+\sum_{j\neq1,i}^{\infty}\left[\left(-4i\right)\mathbb{E}_{\bm{X}}x^{i+1}\mathbb{E}_{\bm{X}}x^{j+1}+2Cov\left(x^{i+1},x^{j+1}\right)\right]A_{i}A_{j}\right\} .\nonumber 
\end{align}
This finishes the proof. 
\end{proof}

\section{\label{sec:Computational-Examples-for}Computational Examples using
Corollary \ref{cor:analytic_signal to linear model-1}}
\begin{example}
\label{exa:(Polynomial-signal)-When-1}(Polynomial signal) When $\mu(x)=A_{3}x^{3}+A_{2}x^{2}+A_{1}x^{1}+A_{0},x\sim N(0,1)$

\begin{align*}
\left(\mathbb{E}_{\bm{X}}x\mu(x)\right)^{2} & =\left(\mathbb{E}_{\bm{X}}A_{3}x^{4}+A_{2}x^{3}+A_{1}x^{2}+A_{0}x\right)^{2}=\left(3A_{3}+A_{1}\right)^{2}\\
\mathbb{E}_{\bm{X}}x^{2}\mu(x)^{2} & =\mathbb{E}_{\bm{X}}x^{2}\left(A_{3}^{2}x^{6}+A_{2}^{2}x^{4}+A_{1}^{2}x^{2}+A_{0}^{2}+2A_{3}A_{2}x^{5}+2A_{3}A_{1}x^{4}\right.\\
 & \left.+2A_{3}A_{0}x^{3}+2A_{2}A_{1}x^{3}+2A_{2}A_{0}x^{2}+2A_{1}A_{0}x^{1}\right)\\
 & =105A_{3}^{2}+15A_{2}^{2}+3A_{1}^{2}+A_{0}^{2}+30A_{3}A_{1}+6A_{2}A_{0}\\
\mathbb{E}_{\bm{X}}x_{i}^{3}\mu(x_{i}) & =\mathbb{E}_{\bm{X}}\left[A_{3}x^{6}+A_{2}x^{5}+A_{1}x^{4}+A_{0}x^{3}\right]=15A_{3}+3A_{1}\\
\mathbb{E}_{\bm{X}}x^{3}\mu(x)\cdot\mathbb{E}_{\bm{X}}x^{'}\mu(x^{'}) & =\left(15A_{3}+3A_{1}\right)\left(3A_{3}+A_{1}\right)=45A_{3}^{2}+24A_{3}A_{1}+3A_{1}^{2}
\end{align*}
\begin{align}
 & \mathbb{E}_{\bm{X}}\frac{n}{2\sigma_{\epsilon}^{2}}\cdot\text{Opt }R_{\bm{X}}\nonumber \\
 & \asymp\frac{1}{2\sigma_{\epsilon}^{2}}\left\{ 6\left(9A_{3}^{2}+6A_{3}A_{1}+A_{1}^{2}\right)+2\left(105A_{3}^{2}+15A_{2}^{2}+3A_{1}^{2}+A_{0}^{2}+30A_{3}A_{1}+6A_{2}A_{0}\right)\right.\nonumber \\
 & \left.-4\cdot\left(45A_{3}^{2}+24A_{3}A_{1}+3A_{1}^{2}\right)\right\} +1+o(1)\nonumber \\
 & \asymp\frac{1}{2\sigma_{\epsilon}^{2}}\cdot\left(2A_{0}^{2}+\underset{\text{exceeding terms caused by mis-specification}}{\underbrace{30A_{2}^{2}+84A_{3}^{2}+12A_{0}A_{2}}}\right)+1+o(1)\label{eq:polynomial_signal_term-1}
\end{align}

And $G(\mu,P_{\bm{X}})=g(A_{0},A_{1},A_{2},A_{3})=2A_{0}^{2}+30A_{2}^{2}+84A_{3}^{2}+12A_{0}A_{2}$. 
\end{example}

\begin{example}
\label{exa:Dirac signal}(Exponential signal) When $\mu(x)=\exp\left(-a(x-b)^{2}\right),x\sim N(0,1)$,
we have: 
\begin{align}
\left(\mathbb{E}_{\bm{X}}x\mu(x)\right)^{2}=\int_{\mathbb{R}}\frac{1}{\sqrt{2\pi}}\exp\left(-\frac{x^{2}}{2}\right)\times x\mu(x) & dx=\int_{\mathbb{R}}\frac{1}{\sqrt{2\pi}}\exp\left(-\frac{x^{2}}{2}\right)\times x\mu(x)dx\nonumber \\
 & =\int_{\mathbb{R}}\frac{1}{\sqrt{2\pi}}\exp\left(-\frac{x^{2}}{2}\right)\times x\exp\left(-a(x-b)^{2}\right)dx\nonumber \\
 & =\frac{abe^{-\frac{ab^{2}}{a+1}}}{\sqrt{2}(1+a)^{3/2}}\label{eq:exp1}
\end{align}
\begin{align}
\mathbb{E}_{\bm{X}}x^{2}\mu(x)^{2}=\int_{\mathbb{R}}\frac{1}{\sqrt{2\pi}}\exp\left(-\frac{x^{2}}{2}\right)\times x^{2}\mu(x)^{2} & dx=\int_{\mathbb{R}}\frac{1}{\sqrt{2\pi}}\exp\left(-\frac{x^{2}}{2}\right)\times x^{2}\mu(x)^{2}dx\nonumber \\
 & =\int_{\mathbb{R}}\frac{1}{\sqrt{2\pi}}\exp\left(-\frac{x^{2}}{2}\right)\times x^{2}\exp\left(-2a(x-b)^{2}\right)dx\nonumber \\
 & =\frac{(1+2a+8a^{2}b^{2})e^{-\frac{2ab^{2}}{2a+1}}}{2\sqrt{2}(1+2a)^{5/2}}\label{eq:exp2}
\end{align}
\begin{align}
\mathbb{E}_{\bm{X}}x^{3}\mu(x)=\int_{\mathbb{R}}\frac{1}{\sqrt{2\pi}}\exp\left(-\frac{x^{2}}{2}\right)\times x^{3}\mu(x) & dx=\int_{\mathbb{R}}\frac{1}{\sqrt{2\pi}}\exp\left(-\frac{x^{2}}{2}\right)\times x^{3}\mu(x)dx\nonumber \\
 & =\int_{\mathbb{R}}\frac{1}{\sqrt{2\pi}}\exp\left(-\frac{x^{2}}{2}\right)\times x^{3}\exp\left(-a(x-b)^{2}\right)dx\nonumber \\
 & =\frac{ab(3+3a+2a^{2}b^{2})e^{-\frac{ab^{2}}{a+1}}}{2\sqrt{2}(1+a)^{7/2}}\label{eq:exp3}
\end{align}
\begin{align*}
 & \mathbb{E}_{\bm{X}}\frac{n}{2\sigma_{\epsilon}^{2}}\cdot\text{Opt }R_{\bm{X}}\\
 & \asymp\frac{1}{2\sigma_{\epsilon}^{2}}\cdot\left(\frac{3a^{2}b^{2}}{(1+a)^{3}}e^{-\frac{2ab^{2}}{1+a}}+\frac{1+2a^{2}+8a^{2}b^{2}}{\sqrt{2}(1+2a)^{5/2}}e^{-\frac{2ab^{2}}{1+2a}}+\frac{a^{2}b^{2}(2+a(3+2ab^{2}))}{(1+a)^{5}}e^{-\frac{2ab^{2}}{1+a}}\right)+1+o(1)
\end{align*}
And $G(\mu,P_{\bm{X}})=g(a,b)=$\eqref{eq:exp1}+\eqref{eq:exp2}+\eqref{eq:exp3}. 
\end{example}

\section{\label{sec:Proof-of-Theoremlowrank}Proof of Theorem \ref{thm:main theorem-low-rank}}
\begin{proof}
From \eqref{eq:general_opt_expression}, we know that with a rank
$k$ approximation $\bm{\Sigma}_{k}$ to the matrix $\bm{\Sigma}$,
it becomes 
\begin{align}
 & \mathbb{E}_{\bm{X}}\text{Opt }R_{\bm{X}}\nonumber \\
 & =\frac{2}{n}\mathbb{E}_{\bm{X}}\left[\mathbb{E}_{\bm{x}_{*}}{\color{blue}\left\Vert y_{*}-\bm{x}_{*}^{T}\bm{\Sigma}_{k}^{-1}\bm{\eta}+\bm{x}_{*}^{T}\left[\bm{\Sigma}_{k}^{-1}-\bm{\Sigma}^{-1}\right]\bm{\eta}\right\Vert _{2}^{2}}{\color{red}\left(\bm{x}_{*}^{T}\left[\bm{\Sigma}_{k}^{-1}-\bm{\Sigma}_{k}^{-1}+\bm{\Sigma}^{-1}\right]\bm{x}_{*}\right)}\right]+O_{p}\left(\frac{1}{n^{3/2}}\right)\\
 & =\frac{2}{n}\mathbb{E}_{\bm{X}}\mathbb{E}_{\bm{x}_{*}}\left[{\color{blue}\left(\left\Vert y_{*}-\bm{x}_{*}^{T}\bm{\Sigma}_{k}^{-1}\bm{\eta}\right\Vert _{2}^{2}+2\left(\bm{x}_{*}^{T}\left[\bm{\Sigma}_{k}^{-1}-\bm{\Sigma}^{-1}\right]\bm{\eta}\right)\left(y_{*}-\bm{x}_{*}^{T}\bm{\Sigma}_{k}^{-1}\bm{\eta}\right)+\left\Vert \bm{x}_{*}^{T}\left[\bm{\Sigma}_{k}^{-1}-\bm{\Sigma}^{-1}\right]\bm{\eta}\right\Vert _{2}^{2}\right)}\right.\label{eq:EKY1}\\
 & \left.\cdot{\color{red}\left(\left(\bm{x}_{*}^{T}\bm{\Sigma}_{k}^{-1}\bm{x}_{*}\right)+\left(\bm{x}_{*}^{T}\left[\bm{\Sigma}^{-1}-\bm{\Sigma}_{k}^{-1}\right]\bm{x}_{*}\right)\right)}\right]+O_{p}\left(\frac{1}{n^{3/2}}\right)\nonumber 
\end{align}
Then, we suppose that $\bm{\Sigma}=\bm{U}\bm{\Lambda}\bm{V}^{T}$
is the singular value decomposition with orthogonal matrices $\bm{U},\bm{V}$,
$\bm{\Lambda}=\text{diag}(\sigma_{1},\cdots,\sigma_{d})$ such that
$\sigma_{1}\geq\cdots\geq\sigma_{d}\geq0$. Then $\bm{\Sigma}^{-1}=\bm{V}^{-T}\bm{\Lambda}^{-1}\bm{U}^{-1}$
and by Eckhart-Young theorem we have that with an optimal rank $k$
approximation to the original matrix $\bm{\Sigma}$, $\left\Vert \bm{\Sigma}_{k}-\bm{\Sigma}\right\Vert _{2}\geq\sigma_{k+1}$
and $\left\Vert \bm{\Sigma}_{k}^{-1}-\bm{\Sigma}^{-1}\right\Vert _{2}\leq\sigma_{k+1}^{-1}$.
Then \eqref{eq:EKY1} becomes 
\begin{align*}
 & \eqref{eq:EKY1}\\
\leq & \frac{2}{n}\mathbb{E}_{\bm{X}}\mathbb{E}_{\bm{x}_{*}}\left[{\color{blue}\left(\left\Vert y_{*}-\bm{x}_{*}^{T}\bm{\Sigma}_{k}^{-1}\bm{\eta}\right\Vert _{2}^{2}+2\sigma_{k+1}^{-1}\left(\bm{x}_{*}^{T}\bm{\eta}\right)\left(y_{*}-\bm{x}_{*}^{T}\bm{\Sigma}_{k}^{-1}\bm{\eta}\right)+\sigma_{k+1}^{-2}\left\Vert \bm{x}_{*}^{T}\bm{\eta}\right\Vert _{2}^{2}\right)}\right.\\
 & \left.\cdot{\color{red}\left(\bm{x}_{*}^{T}\left(\bm{\Sigma}_{k}^{-1}+\sigma_{k+1}^{-1}\bm{I}\right)\bm{x}_{*}\right)}\right]+O_{p}\left(\frac{1}{n^{3/2}}\right)\\
= & \frac{2}{n}\mathbb{E}_{\bm{X}}\left[{\color{blue}\left(\mathbb{E}_{\bm{x}_{*}}\left\Vert y_{*}-\bm{x}_{*}^{T}\left[\bm{\Sigma}_{k}^{-1}+\sigma_{k+1}^{-1}\bm{I}\right]\bm{\eta}\right\Vert _{2}^{2}\right)}\cdot{\color{red}\left(\bm{x}_{*}^{T}\left(\bm{\Sigma}_{k}^{-1}+\sigma_{k+1}^{-1}\bm{I}\right)\bm{x}_{*}\right)}\right]+O_{p}\left(\frac{1}{n^{3/2}}\right)
\end{align*}
$ $ 
\end{proof}

\section{\label{sec:Proof-of-Theorem ridge}Proof of Theorem \ref{thm:main theorem-ridge-2}}

\textbf{Assumptions A2. } Let $\hat{\bm{\eta}}=\frac{1}{n}\bm{X}^{T}\bm{y}(\bm{X})=\frac{1}{n}\bm{X}^{T}\bm{y}$
and $\hat{\bm{\Sigma}}_{\lambda}=\frac{1}{n}\left(\bm{X}^{T}\bm{X}+\lambda\bm{I}\right)\in\mathbb{R}^{d\times d}$
for a fixed positive $\lambda$. We assume that 
\begin{equation}
\left\Vert \hat{\bm{\eta}}-\bm{\eta}\right\Vert _{2}=O_{p}\left(\frac{1}{\sqrt{n}}\right),\left\Vert \hat{\bm{\Sigma}}_{\lambda}-\bm{\Sigma}_{\lambda}\right\Vert _{2}=O_{p}\left(\frac{1}{\sqrt{n}}\right)\label{eq:assumption_A1-1}
\end{equation}
where $\bm{\eta}=\mathbb{E}_{\bm{x}_{*}}\bm{x}_{*}y(\bm{x}_{*})=\mathbb{E}_{\bm{x}_{*}}\bm{x}_{*}y_{*}$
and $\bm{\Sigma}_{\lambda}=\mathbb{E}_{\bm{x}_{*}}(\bm{x}_{*}\bm{x}_{*}^{T}+\lambda\bm{I})$. 
\begin{thm}
\label{thm:main theorem-ridge-appendix} Under Assumption A2, we can
write down the errors as

\begin{align*}
\mathbb{E}_{\bm{X}}\text{Err}R_{\bm{X}} & =\mathbb{E}_{\bm{x}_{*}}\left\Vert y_{*}-\bm{x}_{*}^{T}\hat{\bm{\beta}}\right\Vert _{2}^{2}\\
 & =\mathbb{E}_{\bm{x}_{*}}\left(y_{*}-\bm{x}_{*}^{T}\bm{\Sigma}_{\lambda}^{-1}\bm{\eta}\right)^{2}\\
 & +\frac{1}{n}\mathbb{E}_{\bm{x}_{*}}\left\Vert \bm{\Sigma}^{1/2}\bm{\Sigma}_{\lambda}^{-1}\left[\hat{\bm{\eta}}-\hat{\bm{\Sigma}}_{\lambda}\bm{\Sigma}_{\lambda}^{-1}\bm{\eta}\right]\right\Vert _{2}^{2}\\
 & +2\mathbb{E}_{\bm{X}}\left(\bm{\eta}^{T}-\bm{\eta}^{T}\bm{\Sigma}_{\lambda}^{-1}\bm{\Sigma}\right)\left(\bm{\Sigma}_{\lambda}^{-1}\bm{\eta}-\hat{\bm{\Sigma}}_{\lambda}^{-1}\hat{\bm{\eta}}\right)+O_{p}\left(\frac{1}{n^{3/2}}\right).\\
\mathbb{E}_{\bm{X}}\text{Err}T_{\bm{X}} & =\frac{1}{n}\mathbb{E}_{\bm{X}}\left\Vert \bm{y}-\bm{X}\hat{\bm{\beta}}\right\Vert _{2}^{2}\\
 & =\frac{1}{n}\mathbb{E}_{\bm{X}}\left(\bm{y}-\bm{X}\bm{\Sigma}_{\lambda}^{-1}\bm{\eta}\right)^{2}\\
 & -\frac{1}{n}\mathbb{E}_{\bm{X}}\left\Vert \sqrt{\bm{X}^{T}\bm{X}}\bm{\Sigma}_{\lambda}^{-1}\left[\hat{\bm{\eta}}-\hat{\bm{\Sigma}}_{\lambda}\bm{\Sigma}_{\lambda}^{-1}\bm{\eta}\right]\right\Vert _{2}^{2}\\
 & +2\mathbb{E}_{\bm{X}}\left(\hat{\bm{\eta}}^{T}-\bm{\eta}^{T}\bm{\Sigma}_{\lambda}^{-1}\hat{\bm{\Sigma}}\right)\left(\bm{\Sigma}_{\lambda}^{-1}\bm{\eta}-\hat{\bm{\Sigma}}_{\lambda}^{-1}\hat{\bm{\eta}}\right)+O_{p}\left(\frac{1}{n^{3/2}}\right).\\
\end{align*}
The expected random optimism for the least squares estimator is 
\begin{align}
\mathbb{E}_{\bm{X}}\text{Opt }R_{\bm{X}} & =\frac{2}{n}\mathbb{E}_{\bm{x}_{*}}\left\Vert \bm{\Sigma}^{1/2}\bm{\Sigma}_{\lambda}^{-1}\left[\hat{\bm{\eta}}-\hat{\bm{\Sigma}}_{\lambda}\bm{\Sigma}_{\lambda}^{-1}\bm{\eta}\right]\right\Vert _{2}^{2}\label{eq:general_opt_expression-2}\\
 & +\frac{1}{n}\mathbb{E}_{\bm{X}}\left\Vert \sqrt{\bm{X}^{T}\bm{X}}\bm{\Sigma}_{\lambda}^{-1}\left[\hat{\bm{\eta}}-\hat{\bm{\Sigma}}_{\lambda}\bm{\Sigma}_{\lambda}^{-1}\bm{\eta}\right]\right\Vert _{2}^{2}\nonumber \\
 & +O_{p}\left(\frac{1}{n^{3/2}}\right)\label{eq:easy-to-read-ridge}
\end{align}
\end{thm}

\begin{rem}
(Positivity) The red part in \eqref{eq:easy-to-read-ridge} is analogous
to $\left\Vert y_{*}-\bm{x}_{*}^{T}\bm{\Sigma}^{-1}\bm{\eta}\right\Vert _{2}^{2}\left(\bm{x}_{*}^{T}\bm{\Sigma}^{-1}\bm{x}_{*}\right)$
in \eqref{eq:general_opt_expression} and remains positive regardless
of the choice of $\lambda$. Now note that $\Sigma_{\lambda_{1}}\succeq\Sigma_{\lambda_{2}}$
for $\lambda_{1}\geq\lambda_{2}$ (i.e., $\Sigma_{\lambda_{1}}-\Sigma_{\lambda_{2}}$
is positive definite) we assume that $0<\text{\ensuremath{\underbar{\ensuremath{\lambda}}}}\leq\lambda\leq\bar{\lambda}<\infty$,
then the blue parts in \eqref{eq:easy-to-read-ridge} consist of $\bm{\eta}^{T}\bm{\Sigma}_{\lambda}^{-1}\left(\bm{\Sigma}-\bm{\Sigma}_{\lambda}\right)\bm{\Sigma}_{\lambda}^{-1}\bm{\eta}\geq\bm{\eta}^{T}\bm{\Sigma}_{\bar{\lambda}}^{-1}\left(\bm{\Sigma}-\bm{\Sigma}_{\bar{\lambda}}\right)\bm{\Sigma}_{\text{\ensuremath{\bar{\lambda}}}}^{-1}\bm{\eta}$.
Therefore, $0<\text{\ensuremath{\underbar{\ensuremath{\lambda}}}}\leq\lambda$
is a sufficient condition to ensure positive optimism under Assumption
A2. 
\end{rem}

\textbf{Proof of Theorem \ref{thm:main theorem-ridge-appendix} (Theorem
\ref{thm:main theorem-ridge-2} in the main text):}

Parallel to the proof of Theorem \eqref{thm:main theorem} in Appendix
\eqref{sec:Proof-of-the-main-THM} Testing Error: We know that the
coefficient estimate for training set $\bm{X}\in\mathbb{R}^{n\times d}$
and single testing point $\bm{x}_{*}\in\mathbb{R}^{d\times1}$ as
a column vector. 
\begin{align*}
\hat{\bm{\beta}} & =\left(\bm{X}^{T}\bm{X}+\lambda\bm{I}\right)^{-1}\bm{X}^{T}\bm{y}=\hat{\bm{\Sigma}}_{\lambda}^{-1}\hat{\bm{\eta}},\\
\hat{\bm{\Sigma}}_{\lambda} & =\frac{1}{n}\left(\left(\begin{array}{cc}
\bm{X}^{T} & \sqrt{n\lambda}\bm{I}\end{array}\right)\left(\begin{array}{c}
\bm{X}\\
\sqrt{n\lambda}\bm{I}
\end{array}\right)\right)=\frac{1}{n}\left(\bm{X}^{T}\bm{X}+n\lambda\bm{I}\right)=\frac{1}{n}\bm{X}^{T}\bm{X}+\lambda\bm{I}\in\mathbb{R}^{d\times d},\\
\bm{\Sigma}_{\lambda} & =\mathbb{E}_{\bm{X}}\hat{\bm{\Sigma}}_{\lambda}=\mathbb{E}_{\bm{x}_{*}}\bm{x}_{*}\bm{x}_{*}^{T}+\lambda\bm{I}\in\mathbb{R}^{d\times d},
\end{align*}
Here we use the $y(\bm{x}_{*})$ to denote the observed values at
this single testing point $\bm{x}_{*}$ which is not necessarily in
the training set. Now consider an arbitrary pair $(\bm{x}_{*},y_{*})$
as an independent draw from the distribution of $\bm{X},\bm{y}$ and
the $L_{2}$ loss function: 
\begin{align}
 & \mathbb{E}_{\bm{x}_{*}}\left\Vert y_{*}-\bm{x}_{*}^{T}\hat{\bm{\beta}}\right\Vert _{2}^{2}\nonumber \\
 & =\mathbb{E}_{\bm{x}_{*}}\left\Vert y_{*}-\bm{x}_{*}^{T}\bm{\Sigma}_{\lambda}^{-1}\bm{\eta}+\bm{x}_{*}^{T}\bm{\Sigma}_{\lambda}^{-1}\bm{\eta}-\bm{x}_{*}^{T}\hat{\bm{\beta}}\right\Vert _{2}^{2}\nonumber \\
 & =\mathbb{E}_{\bm{x}_{*}}\left\Vert \left(y_{*}-\bm{x}_{*}^{T}\bm{\Sigma}_{\lambda}^{-1}\bm{\eta}\right)+\bm{x}_{*}^{T}\left(\bm{\Sigma}_{\lambda}^{-1}\bm{\eta}-\hat{\bm{\Sigma}}_{\lambda}^{-1}\hat{\bm{\eta}}\right)\right\Vert _{2}^{2}\nonumber \\
 & =\mathbb{E}_{\bm{x}_{*}}\left(y_{*}-\bm{x}_{*}^{T}\bm{\Sigma}_{\lambda}^{-1}\bm{\eta}\right)^{2}\nonumber \\
 & +\mathbb{E}_{\bm{x}_{*}}\left[\bm{x}_{*}^{T}\left(\bm{\Sigma}_{\lambda}^{-1}\bm{\eta}-\hat{\bm{\Sigma}}_{\lambda}^{-1}\hat{\bm{\eta}}\right)\right]^{2}\nonumber \\
 & +2\mathbb{E}_{\bm{x}_{*}}\left(y_{*}-\bm{x}_{*}^{T}\bm{\Sigma}_{\lambda}^{-1}\bm{\eta}\right)^{T}\cdot\bm{x}_{*}^{T}\left(\bm{\Sigma}_{\lambda}^{-1}\bm{\eta}-\hat{\bm{\Sigma}}_{\lambda}^{-1}\hat{\bm{\eta}}\right)\label{eq:newp_ridge_1-1}
\end{align}
where we observe that in \eqref{eq:newp_ridge_1-1}: 
\begin{align}
 & \mathbb{E}_{\bm{x}_{*}}\left[\bm{x}_{*}^{T}\left(\bm{\Sigma}_{\lambda}^{-1}\bm{\eta}-\hat{\bm{\Sigma}}_{\lambda}^{-1}\hat{\bm{\eta}}\right)\right]^{2}\nonumber \\
 & =\mathbb{E}_{\bm{x}_{*}}\left(\bm{\Sigma}_{\lambda}^{-1}\bm{\eta}-\hat{\bm{\Sigma}}_{\lambda}^{-1}\hat{\bm{\eta}}\right)^{T}\bm{x}_{*}\bm{x}_{*}^{T}\left(\bm{\Sigma}_{\lambda}^{-1}\bm{\eta}-\hat{\bm{\Sigma}}_{\lambda}^{-1}\hat{\bm{\eta}}\right)\nonumber \\
 & =\left(\bm{\Sigma}_{\lambda}^{-1}\bm{\eta}-\hat{\bm{\Sigma}}_{\lambda}^{-1}\hat{\bm{\eta}}\right)^{T}\bm{\Sigma}\left(\bm{\Sigma}_{\lambda}^{-1}\bm{\eta}-\hat{\bm{\Sigma}}_{\lambda}^{-1}\hat{\bm{\eta}}\right)\\
\nonumber 
\end{align}
Unlike the cross-product term $\left(\bm{\eta}^{T}-\bm{\eta}^{T}\bm{\Sigma}^{-1}\bm{\Sigma}\right)=0$
in \eqref{eq:newp_1} vanishes, we noticed that $\bm{\eta}^{T}-\bm{\eta}^{T}\bm{\Sigma}_{\lambda}^{-1}\bm{\Sigma}\neq0$
for any $\lambda\neq0$. 
\begin{align}
 & \mathbb{E}_{\bm{x}_{*}}\left(y_{*}-\bm{x}_{*}^{T}\bm{\Sigma}_{\lambda}^{-1}\bm{\eta}\right)^{T}\cdot\bm{x}_{*}^{T}\left(\bm{\Sigma}_{\lambda}^{-1}\bm{\eta}-\hat{\bm{\Sigma}}_{\lambda}^{-1}\hat{\bm{\eta}}\right)\nonumber \\
= & \mathbb{E}_{\bm{x}_{*}}\left(y_{*}\bm{x}_{*}^{T}-\bm{\eta}^{T}\bm{\Sigma}_{\lambda}^{-1}\bm{x}_{*}\bm{x}_{*}^{T}\right)\left(\bm{\Sigma}_{\lambda}^{-1}\bm{\eta}-\hat{\bm{\Sigma}}_{\lambda}^{-1}\hat{\bm{\eta}}\right)\nonumber \\
= & \left(\mathbb{E}_{\bm{x}_{*}}y_{*}\bm{x}_{*}^{T}-\bm{\eta}^{T}\bm{\Sigma}_{\lambda}^{-1}\mathbb{E}_{\bm{x}_{*}}\bm{x}_{*}\bm{x}_{*}^{T}\right)\left(\bm{\Sigma}_{\lambda}^{-1}\bm{\eta}-\hat{\bm{\Sigma}}_{\lambda}^{-1}\hat{\bm{\eta}}\right)\nonumber \\
= & \underset{O_{p}\left(\frac{1}{\sqrt{n}}\right)}{\underbrace{\left(\bm{\eta}^{T}-\bm{\eta}^{T}\bm{\Sigma}_{\lambda}^{-1}\bm{\Sigma}\right)}}\left(\bm{\Sigma}_{\lambda}^{-1}\bm{\eta}-\hat{\bm{\Sigma}}_{\lambda}^{-1}\hat{\bm{\eta}}\right)\label{eq:newp_1_RIDGE}\\
\neq & 0.\nonumber 
\end{align}
Taking the expectation with respect to the training set, then $\mathbb{E}_{\bm{X}}$\eqref{eq:newp_ridge_1-1}
still simplifies into 
\begin{align}
\mathbb{E}_{\bm{X}}\eqref{eq:newp_ridge_1-1} & =\mathbb{E}_{\bm{X}}\left\{ \mathbb{E}_{\bm{x}_{*}}\left(y_{*}-\bm{x}_{*}^{T}\bm{\Sigma}_{\lambda}^{-1}\bm{\eta}\right)^{2}+\mathbb{E}_{\bm{x}_{*}}\left[\bm{x}_{*}^{T}\left(\bm{\Sigma}_{\lambda}^{-1}\bm{\eta}-\hat{\bm{\Sigma}}_{\lambda}^{-1}\hat{\bm{\eta}}\right)\right]^{2}\right.\nonumber \\
 & +\left.2\left(\bm{\eta}^{T}-\bm{\eta}^{T}\bm{\Sigma}_{\lambda}^{-1}\bm{\Sigma}\right)\left(\bm{\Sigma}_{\lambda}^{-1}\bm{\eta}-\hat{\bm{\Sigma}}_{\lambda}^{-1}\hat{\bm{\eta}}\right)\right\} \nonumber \\
 & =\mathbb{E}_{\bm{x}_{*}}\left(y_{*}-\bm{x}_{*}^{T}\bm{\Sigma}_{\lambda}^{-1}\bm{\eta}\right)^{2}\nonumber \\
 & +\underset{(1*)}{\underbrace{\mathbb{E}_{\bm{X}}\left(\bm{\Sigma}_{\lambda}^{-1}\bm{\eta}-\hat{\bm{\Sigma}}_{\lambda}^{-1}\hat{\bm{\eta}}\right)^{T}\bm{\Sigma}\left(\bm{\Sigma}_{\lambda}^{-1}\bm{\eta}-\hat{\bm{\Sigma}}_{\lambda}^{-1}\hat{\bm{\eta}}\right)}}\nonumber \\
 & +2\underset{(1**)}{\underbrace{\mathbb{E}_{\bm{X}}\left(\bm{\eta}^{T}-\bm{\eta}^{T}\bm{\Sigma}_{\lambda}^{-1}\bm{\Sigma}\right)\left(\bm{\Sigma}_{\lambda}^{-1}\bm{\eta}-\hat{\bm{\Sigma}}_{\lambda}^{-1}\hat{\bm{\eta}}\right)}}.\label{eq:newp_ridge_2}\\
\nonumber 
\end{align}
 And the part (1{*}) in \eqref{eq:newp_ridge_2} can be expanded using
definition of symbols, using another arbitrary pair $(\bm{x}_{*},y_{*})$
as an independent copy of $\bm{X},\bm{y}$: {} 
\begin{align}
(1*) & =\mathbb{E}_{\bm{X}}\left(\bm{\Sigma}_{\lambda}^{-1}\bm{\eta}-\hat{\bm{\Sigma}}_{\lambda}^{-1}\hat{\bm{\eta}}\right)^{T}\bm{\Sigma}\left(\bm{\Sigma}_{\lambda}^{-1}\bm{\eta}-\hat{\bm{\Sigma}}_{\lambda}^{-1}\hat{\bm{\eta}}\right)\nonumber \\
 & =\mathbb{E}_{\bm{X}}\left(\hat{\bm{\eta}}-\hat{\bm{\Sigma}}_{\lambda}\bm{\Sigma}_{\lambda}^{-1}\bm{\eta}\right)^{T}\bm{\Sigma}_{\lambda}^{-1}\bm{\Sigma}\bm{\Sigma}_{\lambda}^{-1}\left(\hat{\bm{\eta}}-\hat{\bm{\Sigma}}_{\lambda}\bm{\Sigma}_{\lambda}^{-1}\bm{\eta}\right)+O_{p}\left(\frac{1}{n^{2}}\right)\nonumber \\
 & =\frac{1}{n}\mathbb{E}_{\bm{x}_{*}}\left[\bm{x}_{*}y_{*}-\left(\bm{x}_{*}\bm{x}_{*}^{T}+\lambda\bm{I}\right)\bm{\Sigma}_{\lambda}^{-1}\bm{\eta}\right]^{T}\bm{\Sigma}_{\lambda}^{-1}\bm{\Sigma}\bm{\Sigma}_{\lambda}^{-1}\left[\bm{x}_{*}y_{*}-\left(\bm{x}_{*}\bm{x}_{*}^{T}+\lambda\bm{I}\right)\bm{\Sigma}_{\lambda}^{-1}\bm{\eta}\right]+O_{p}\left(\frac{1}{n^{2}}\right)\nonumber \\
 & =\frac{1}{n}\mathbb{E}_{\bm{x}_{*}}\left\Vert \bm{\Sigma}^{1/2}\bm{\Sigma}_{\lambda}^{-1}\left[\bm{x}_{*}y_{*}-\left(\bm{x}_{*}\bm{x}_{*}^{T}+\lambda\bm{I}\right)\bm{\Sigma}_{\lambda}^{-1}\bm{\eta}\right]\right\Vert _{2}^{2}+O_{p}\left(\frac{1}{n^{2}}\right).\label{eq:ridge_cross_term-1}\\
\nonumber 
\end{align}
To sum up, we can plug \eqref{eq:ridge_cross_term-1} back into \eqref{eq:newp_ridge_2}
and yield
\begin{align}
\mathbb{E}_{\bm{x}_{*}}\left\Vert y_{*}-\bm{x}_{*}^{T}\hat{\bm{\beta}}\right\Vert _{2}^{2} & =\mathbb{E}_{\bm{x}_{*}}\left(y_{*}-\bm{x}_{*}^{T}\left(\bm{\Sigma}+\lambda\bm{I}\right)^{-1}\bm{\eta}\right)^{2}\nonumber \\
 & +\frac{1}{n}\mathbb{E}_{\bm{x}_{*}}\left\Vert \bm{\Sigma}^{1/2}\bm{\Sigma}_{\lambda}^{-1}\left[\bm{x}_{*}y_{*}-\left(\bm{x}_{*}\bm{x}_{*}^{T}+\lambda\bm{I}\right)\bm{\Sigma}_{\lambda}^{-1}\bm{\eta}\right]\right\Vert _{2}^{2}\nonumber \\
 & +2\mathbb{E}_{\bm{X}}\left(\bm{\eta}^{T}-\bm{\eta}^{T}\bm{\Sigma}_{\lambda}^{-1}\bm{\Sigma}\right)\left(\bm{\Sigma}_{\lambda}^{-1}\bm{\eta}-\hat{\bm{\Sigma}}_{\lambda}^{-1}\hat{\bm{\eta}}\right)\nonumber \\
 & +O_{p}\left(\frac{1}{n^{3/2}}\right)\text{ by the magnitude in }\eqref{eq:ridge_cross_term-1}.\label{eq:new_ridge_test_errs}\\
\nonumber 
\end{align}

Similarly for training error, we recall that $\hat{\bm{\beta}}=\hat{\bm{\Sigma}}_{\lambda}^{-1}\hat{\bm{\eta}}$
and the rest derivation is similar to that for testing error, we have
\begin{align}
 & \frac{1}{n}\mathbb{E}_{\bm{X}}\left\Vert \bm{y}-\bm{X}\hat{\bm{\beta}}\right\Vert _{2}^{2}\text{ where }\hat{\bm{\beta}}=\hat{\bm{\Sigma}}_{0}^{-1}\hat{\bm{\eta}}_{0}\nonumber \\
 & =\frac{1}{n}\mathbb{E}_{\bm{X}}\left(\bm{y}^{T}\bm{y}-2\bm{y}^{T}\bm{X}\hat{\bm{\Sigma}}_{\lambda}^{-1}\hat{\bm{\eta}}+\hat{\bm{\eta}}^{T}\hat{\bm{\Sigma}}_{\lambda}^{-1}\bm{X}^{T}\bm{X}\hat{\bm{\Sigma}}_{\lambda}^{-1}\hat{\bm{\eta}}\right)\\
 & =\frac{1}{n}\mathbb{E}_{\bm{X}}\left(\bm{y}^{T}\bm{y}-2\bm{y}^{T}\bm{X}\bm{\Sigma}_{\lambda}^{-1}\bm{\eta}+\bm{\eta}^{T}\bm{\Sigma}_{\lambda}^{-1}\bm{X}^{T}\bm{X}\bm{\Sigma}_{\lambda}^{-1}\bm{\eta}\right.\nonumber \\
 & +2\bm{y}^{T}\bm{X}\bm{\Sigma}_{\lambda}^{-1}\bm{\eta}-\bm{\eta}^{T}\bm{\Sigma}_{\lambda}\bm{X}^{T}\bm{X}\bm{\Sigma}_{\lambda}^{-1}\bm{\eta}\left.-2\bm{y}^{T}\bm{X}\hat{\bm{\Sigma}}_{\lambda}^{-1}\hat{\bm{\eta}}+\hat{\bm{\eta}}^{T}\hat{\bm{\Sigma}}_{\lambda}^{-1}\bm{X}^{T}\bm{X}\hat{\bm{\Sigma}}_{\lambda}^{-1}\hat{\bm{\eta}}\right)\\
\nonumber 
\end{align}
\begin{align}
 & =\frac{1}{n}\mathbb{E}_{\bm{X}}\left(\bm{y}-\bm{X}\bm{\Sigma}_{\lambda}^{-1}\bm{\eta}\right)^{2}\nonumber \\
 & +\frac{2}{n}\mathbb{E}_{\bm{X}}\bm{y}^{T}\bm{X}\left(\bm{\Sigma}_{\lambda}^{-1}\bm{\eta}-\hat{\bm{\Sigma}}_{\lambda}^{-1}\hat{\bm{\eta}}\right)-\frac{1}{n}\mathbb{E}_{\bm{X}}\bm{\eta}^{T}\bm{\Sigma}_{\lambda}\bm{X}^{T}\bm{X}\bm{\Sigma}_{\lambda}^{-1}\bm{\eta}+\frac{1}{n}\mathbb{E}_{\bm{X}}\hat{\bm{\eta}}^{T}\hat{\bm{\Sigma}}_{\lambda}^{-1}\bm{X}^{T}\bm{X}\hat{\bm{\Sigma}}_{\lambda}^{-1}\hat{\bm{\eta}}\\
 & =\frac{1}{n}\mathbb{E}_{\bm{X}}\left(\bm{y}-\bm{X}\bm{\Sigma}_{\lambda}^{-1}\bm{\eta}\right)^{2}+\underset{(2**)}{\underbrace{\frac{2}{n}\mathbb{E}_{\bm{X}}\bm{y}^{T}\bm{X}\left(\bm{\Sigma}_{\lambda}^{-1}\bm{\eta}-\hat{\bm{\Sigma}}_{\lambda}^{-1}\hat{\bm{\eta}}\right)}}\nonumber \\
 & +\underset{(2*)}{\underbrace{\frac{1}{n}\mathbb{E}_{\bm{X}}\hat{\bm{\eta}}^{T}\hat{\bm{\Sigma}}_{\lambda}^{-1}\bm{X}^{T}\bm{X}\hat{\bm{\Sigma}}_{\lambda}^{-1}\hat{\bm{\eta}}-\frac{1}{n}\mathbb{E}_{\bm{X}}\bm{\eta}^{T}\bm{\Sigma}_{\lambda}\bm{X}^{T}\bm{X}\bm{\Sigma}_{\lambda}^{-1}\bm{\eta}}}\\
 & =\frac{1}{n}\mathbb{E}_{\bm{X}}\left(\bm{y}-\bm{X}\bm{\Sigma}_{\lambda}^{-1}\bm{\eta}\right)^{2}+\underset{(2**)}{\underbrace{\frac{2}{n}\mathbb{E}_{\bm{X}}\bm{y}^{T}\bm{X}\hat{\bm{\Sigma}}_{\lambda}^{-1}\left(\hat{\bm{\Sigma}}_{\lambda}\bm{\Sigma}_{\lambda}^{-1}\bm{\eta}-\hat{\bm{\eta}}\right)}}\nonumber \\
 & -\underset{(2*)}{\underbrace{\frac{1}{n}\mathbb{E}_{\bm{X}}\hat{\bm{\eta}}^{T}\hat{\bm{\Sigma}}_{\lambda}^{-1}\bm{X}^{T}\bm{X}\hat{\bm{\Sigma}}_{\lambda}^{-1}\hat{\bm{\eta}}-\frac{1}{n}\mathbb{E}_{\bm{X}}\bm{\eta}^{T}\bm{\Sigma}_{\lambda}\bm{X}^{T}\bm{X}\bm{\Sigma}_{\lambda}^{-1}\bm{\eta}}}+\frac{2}{n}\mathbb{E}_{\bm{X}}\hat{\bm{\eta}}^{T}\hat{\bm{\Sigma}}_{\lambda}^{-1}\bm{X}^{T}\bm{X}\hat{\bm{\Sigma}}_{\lambda}^{-1}\hat{\bm{\eta}}\label{eq:key1}\\
\nonumber 
\end{align}
Below, moving from \eqref{eq:key1} to \eqref{eq:key2} we need the
following derivation:
\begin{align*}
 & \mathbb{E}_{\bm{X}}\left\Vert \sqrt{\bm{X}^{T}\bm{X}}\bm{\Sigma}_{\lambda}^{-1}\left[\hat{\bm{\eta}}-\hat{\bm{\Sigma}}_{\lambda}\bm{\Sigma}_{\lambda}^{-1}\bm{\eta}\right]\right\Vert _{2}^{2}\\
 & =\mathbb{E}_{\bm{X}}\left(\sqrt{\bm{X}^{T}\bm{X}}\bm{\Sigma}_{\lambda}^{-1}\left[\hat{\bm{\eta}}-\hat{\bm{\Sigma}}_{\lambda}\bm{\Sigma}_{\lambda}^{-1}\bm{\eta}\right]\right)^{T}\left(\sqrt{\bm{X}^{T}\bm{X}}\bm{\Sigma}_{\lambda}^{-1}\left[\hat{\bm{\eta}}-\hat{\bm{\Sigma}}_{\lambda}\bm{\Sigma}_{\lambda}^{-1}\bm{\eta}\right]\right)\\
 & =\mathbb{E}_{\bm{X}}\left[\hat{\bm{\eta}}-\hat{\bm{\Sigma}}_{\lambda}\bm{\Sigma}_{\lambda}^{-1}\bm{\eta}\right]^{T}\bm{\Sigma}_{\lambda}^{-1}\bm{X}^{T}\bm{X}\bm{\Sigma}_{\lambda}^{-1}\left[\hat{\bm{\eta}}-\hat{\bm{\Sigma}}_{\lambda}\bm{\Sigma}_{\lambda}^{-1}\bm{\eta}\right]\\
 & =\mathbb{E}_{\bm{X}}\left[\hat{\bm{\eta}}-\bm{\Sigma}_{\lambda}\bm{\Sigma}_{\lambda}^{-1}\bm{\eta}\right]^{T}\bm{\Sigma}_{\lambda}^{-1}\bm{X}^{T}\bm{X}\bm{\Sigma}_{\lambda}^{-1}\left[\hat{\bm{\eta}}-\bm{\Sigma}_{\lambda}\bm{\Sigma}_{\lambda}^{-1}\bm{\eta}\right]+O_{p}\left(\frac{1}{n}\right)\\
 & =\mathbb{E}_{\bm{X}}\hat{\bm{\eta}}^{T}\left(\bm{\Sigma}_{\lambda}^{-1}\bm{X}^{T}\bm{X}\bm{\Sigma}_{\lambda}^{-1}\right)\hat{\bm{\eta}}+\mathbb{E}_{\bm{X}}\bm{\eta}^{T}\left(\bm{\Sigma}_{\lambda}^{-1}\bm{X}^{T}\bm{X}\bm{\Sigma}_{\lambda}^{-1}\right)\bm{\eta}\\
 & -2\mathbb{E}_{\bm{X}}\bm{\eta}^{T}\left(\bm{\Sigma}_{\lambda}^{-1}\bm{X}^{T}\bm{X}\bm{\Sigma}_{\lambda}^{-1}\right)\hat{\bm{\eta}}+O_{p}\left(\frac{1}{n}\right)\\
\end{align*}
Therefore, (2{*}) becomes
\begin{align}
\eqref{eq:key1} & =\frac{1}{n}\mathbb{E}_{\bm{X}}\left(\bm{y}-\bm{X}\bm{\Sigma}_{\lambda}^{-1}\bm{\eta}\right)^{2}+\underset{(2**)}{\underbrace{\frac{2}{n}\mathbb{E}_{\bm{X}}\bm{y}^{T}\bm{X}\hat{\bm{\Sigma}}_{\lambda}^{-1}\left(\hat{\bm{\Sigma}}_{\lambda}\bm{\Sigma}_{\lambda}^{-1}\bm{\eta}-\hat{\bm{\eta}}\right)}}\nonumber \\
 & -\underset{(2*)}{\underbrace{\frac{1}{n}\mathbb{E}_{\bm{X}}\left\Vert \sqrt{\bm{X}^{T}\bm{X}}\bm{\Sigma}_{\lambda}^{-1}\left[\hat{\bm{\eta}}-\hat{\bm{\Sigma}}_{\lambda}\bm{\Sigma}_{\lambda}^{-1}\bm{\eta}\right]\right\Vert _{2}^{2}}}+\frac{2}{n}\mathbb{E}_{\bm{X}}\hat{\bm{\eta}}^{T}\hat{\bm{\Sigma}}_{\lambda}^{-1}\bm{X}^{T}\bm{X}\hat{\bm{\Sigma}}_{\lambda}^{-1}\hat{\bm{\eta}}-\frac{2}{n}\mathbb{E}_{\bm{X}}\bm{\eta}^{T}\left(\bm{\Sigma}_{\lambda}^{-1}\bm{X}^{T}\bm{X}\bm{\Sigma}_{\lambda}^{-1}\right)\hat{\bm{\eta}}\label{eq:key2}\\
 & =\frac{1}{n}\mathbb{E}_{\bm{X}}\left(\bm{y}-\bm{X}\bm{\Sigma}_{\lambda}^{-1}\bm{\eta}\right)^{2}+\underset{(2**)}{\underbrace{2\mathbb{E}_{\bm{X}}\left(\hat{\bm{\eta}}^{T}\bm{\Sigma}_{\lambda}^{-1}\bm{\eta}-\hat{\bm{\eta}}^{T}\hat{\bm{\Sigma}}_{\lambda}^{-1}\hat{\bm{\eta}}\right)}}-\underset{(2*)}{\underbrace{\frac{1}{n}\mathbb{E}_{\bm{X}}\left\Vert \sqrt{\bm{X}^{T}\bm{X}}\bm{\Sigma}_{\lambda}^{-1}\left[\hat{\bm{\eta}}-\hat{\bm{\Sigma}}_{\lambda}\bm{\Sigma}_{\lambda}^{-1}\bm{\eta}\right]\right\Vert _{2}^{2}}}\nonumber \\
 & +\underset{(2***)}{\underbrace{2\mathbb{E}_{\bm{X}}\left(\hat{\bm{\eta}}^{T}\hat{\bm{\Sigma}}_{\lambda}^{-1}\hat{\bm{\Sigma}}\hat{\bm{\Sigma}}_{\lambda}^{-1}\hat{\bm{\eta}}-\bm{\eta}^{T}\bm{\Sigma}_{\lambda}^{-1}\hat{\bm{\Sigma}}\bm{\Sigma}_{\lambda}^{-1}\hat{\bm{\eta}}\right)}}+O_{p}\left(\frac{1}{n^{3/2}}\right)\\
 & \text{then we combine }(2**)\text{ and }(2***)\text{ into }2\mathbb{E}_{\bm{X}}\left(\hat{\bm{\eta}}^{T}-\bm{\eta}^{T}\bm{\Sigma}_{\lambda}^{-1}\hat{\bm{\Sigma}}\right)\left(\bm{\Sigma}_{\lambda}^{-1}\bm{\eta}-\hat{\bm{\Sigma}}_{\lambda}^{-1}\hat{\bm{\eta}}\right),\nonumber \\
 & =\frac{1}{n}\mathbb{E}_{\bm{X}}\left(\bm{y}-\bm{X}\bm{\Sigma}_{\lambda}^{-1}\bm{\eta}\right)^{2}-\underset{(2*)}{\underbrace{\frac{1}{n}\mathbb{E}_{\bm{X}}\left\Vert \sqrt{\bm{X}^{T}\bm{X}}\bm{\Sigma}_{\lambda}^{-1}\left[\hat{\bm{\eta}}-\hat{\bm{\Sigma}}_{\lambda}\bm{\Sigma}_{\lambda}^{-1}\bm{\eta}\right]\right\Vert _{2}^{2}}}\\
 & +2\mathbb{E}_{\bm{X}}\left(\hat{\bm{\eta}}^{T}-\bm{\eta}^{T}\bm{\Sigma}_{\lambda}^{-1}\hat{\bm{\Sigma}}\right)\left(\bm{\Sigma}_{\lambda}^{-1}\bm{\eta}-\hat{\bm{\Sigma}}_{\lambda}^{-1}\hat{\bm{\eta}}\right)\text{ where we replace }\text{\ensuremath{\hat{\bm{\Sigma}}_{\lambda}^{-1}}}\text{ with }\bm{\Sigma}_{\lambda}^{-1}\text{ and pool into }O_{p}\\
 & +O_{p}\left(\frac{1}{n^{3/2}}\right)\label{eq:newp_ridge_3-1}\\
\nonumber 
\end{align}
Using \eqref{eq:new_ridge_test_errs} and \eqref{eq:newp_ridge_3-1}:
\begin{align}
 & \mathbb{E}_{\bm{X}}\text{Opt }R_{\bm{X}}\nonumber \\
 & =\mathbb{E}_{\bm{x}_{*}}\underset{\text{test MSE}}{\underbrace{\left\Vert y_{*}-\bm{x}_{*}^{T}\bm{\Sigma}_{\lambda}^{-1}\bm{\eta}\right\Vert _{2}^{2}}}-\underset{\text{train MSE}}{\underbrace{\frac{1}{n}\mathbb{E}_{\bm{X}}\left\Vert \bm{y}-\bm{X}\bm{\Sigma}_{\lambda}^{-1}\bm{\eta}\right\Vert ^{2}}}\nonumber \\
 & =\frac{1}{n}\mathbb{E}_{\bm{x}_{*}}\left(\bm{x}_{*}^{T}\bm{x}_{*}+\bm{\Sigma}_{*}\right)\left\Vert \bm{\Sigma}_{\lambda}^{-1}\left(\hat{\bm{\Sigma}}_{\lambda}\bm{\Sigma}_{\lambda}^{-1}\bm{\eta}_{\lambda}-\hat{\bm{\eta}}_{\lambda}\right)\right\Vert _{2}^{2}+\frac{1}{n}\mathbb{E}_{\bm{X}}\left(\bm{X}^{T}\bm{X}\right)\left\Vert \bm{\Sigma}_{\lambda}^{-1}\left(\hat{\bm{\Sigma}}_{\lambda}\bm{\Sigma}_{\lambda}^{-1}\bm{\eta}_{\lambda}-\hat{\bm{\eta}}_{\lambda}\right)\right\Vert _{2}^{2}\\
 & +2\mathbb{E}_{\bm{X}}\left(\bm{\eta}^{T}-\bm{\eta}^{T}\bm{\Sigma}_{\lambda}^{-1}\bm{\Sigma}\right)\left(\bm{\Sigma}_{\lambda}^{-1}\bm{\eta}-\hat{\bm{\Sigma}}_{\lambda}^{-1}\hat{\bm{\eta}}\right)\\
 & -2\mathbb{E}_{\bm{X}}\left(\hat{\bm{\eta}}^{T}\hat{\bm{\Sigma}}_{\lambda}^{-1}\hat{\bm{\Sigma}}\hat{\bm{\Sigma}}_{\lambda}^{-1}\hat{\bm{\eta}}-\bm{\eta}^{T}\bm{\Sigma}_{\lambda}^{-1}\hat{\bm{\Sigma}}\bm{\Sigma}_{\lambda}^{-1}\hat{\bm{\eta}}\right)+O_{p}\left(\frac{1}{n^{3/2}}\right)\\
 & =\frac{1}{n}\mathbb{E}_{\bm{X}}\left(\bm{X}^{T}\bm{X}\right)\left\Vert \bm{\Sigma}_{0}^{-1}\left(\hat{\bm{\Sigma}}_{0}\bm{\Sigma}_{0}^{-1}\bm{\eta}_{0}-\hat{\bm{\eta}}_{0}\right)\right\Vert _{2}^{2}+\frac{1}{n}\mathbb{E}_{\bm{X}}\left(\bm{X}^{T}\bm{X}\right)\left\Vert \bm{\Sigma}_{\lambda}^{-1}\left(\hat{\bm{\Sigma}}_{\lambda}\bm{\Sigma}_{\lambda}^{-1}\bm{\eta}_{\lambda}-\hat{\bm{\eta}}_{\lambda}\right)\right\Vert _{2}^{2}\nonumber \\
 & +\underset{(3)}{\underbrace{2\mathbb{E}_{\bm{X}}\left(\bm{\Sigma}_{\lambda}^{-1}\bm{\eta}-\hat{\bm{\Sigma}}_{\lambda}^{-1}\hat{\bm{\eta}}\right)\left(\bm{\eta}^{T}-\hat{\bm{\eta}}^{T}+\bm{\eta}^{T}\bm{\Sigma}_{\lambda}^{-1}\hat{\bm{\Sigma}}-\bm{\eta}^{T}\bm{\Sigma}_{\lambda}^{-1}\bm{\Sigma}\right)}}+O_{p}\left(\frac{1}{n^{3/2}}\right)\label{eq:ridge_opt_general-2-1}
\end{align}
 The (3) is of order $O_{p}\left(\frac{1}{n^{3/2}}\right)$ due to
the below arguments:
\begin{align}
 & \bm{\Sigma}_{\lambda}^{-1}\bm{\eta}-\hat{\bm{\Sigma}}_{\lambda}^{-1}\hat{\bm{\eta}}\nonumber \\
 & ={\color{blue}\bm{\Sigma}_{\lambda}^{-1}\hat{\bm{\eta}}}-\hat{\bm{\Sigma}}_{\lambda}^{-1}\hat{\bm{\eta}}+\bm{\Sigma}_{\lambda}^{-1}\bm{\eta}-{\color{blue}\bm{\Sigma}_{\lambda}^{-1}\hat{\bm{\eta}}}\nonumber \\
 & =\left(\bm{\Sigma}_{\lambda}^{-1}-\hat{\bm{\Sigma}}_{\lambda}^{-1}\right)\hat{\bm{\eta}}+{\color{violet}\bm{\Sigma}_{\lambda}^{-1}\left(\bm{\eta}-\hat{\bm{\eta}}\right)}\nonumber \\
 & =\bm{\Sigma}_{\lambda}^{-1}\left(\hat{\bm{\Sigma}}_{\lambda}-\bm{\Sigma}_{\lambda}\right)\hat{\bm{\Sigma}}_{\lambda}^{-1}\hat{\bm{\eta}}+{\color{violet}\bm{\Sigma}_{\lambda}^{-1}\left(\bm{\eta}-\hat{\bm{\eta}}\right)}\label{eq:inter-1}\\
 & =\bm{\Sigma}_{\lambda}^{-1}\left(\hat{\bm{\Sigma}}_{\lambda}-\bm{\Sigma}_{\lambda}\right)\hat{\bm{\Sigma}}_{\lambda}^{-1}\bm{\eta}+\bm{\Sigma}_{\lambda}^{-1}\left(\hat{\bm{\Sigma}}_{\lambda}-\bm{\Sigma}_{\lambda}\right)\hat{\bm{\Sigma}}_{\lambda}^{-1}\left(\hat{\bm{\eta}}-\bm{\eta}\right)+{\color{violet}\bm{\Sigma}_{\lambda}^{-1}\left(\bm{\eta}-\hat{\bm{\eta}}\right)}\nonumber \\
 & =\bm{\Sigma}_{\lambda}^{-1}\left(\hat{\bm{\Sigma}}_{\lambda}-\bm{\Sigma}_{\lambda}\right)\left(\hat{\bm{\Sigma}}_{\lambda}^{-1}\bm{\eta}{\color{blue}-\bm{\Sigma}_{\lambda}^{-1}\bm{\eta}+\bm{\Sigma}_{\lambda}^{-1}\bm{\eta}}\right)+\bm{\Sigma}_{\lambda}^{-1}\left(\hat{\bm{\Sigma}}_{\lambda}-\bm{\Sigma}_{\lambda}\right)\hat{\bm{\Sigma}}_{\lambda}^{-1}\left(\hat{\bm{\eta}}-\bm{\eta}\right)+{\color{violet}\bm{\Sigma}_{\lambda}^{-1}\left(\bm{\eta}-\hat{\bm{\eta}}\right)}\nonumber \\
 & =\bm{\Sigma}_{\lambda}^{-1}\left(\hat{\bm{\Sigma}}_{\lambda}-\bm{\Sigma}_{\lambda}\right)\left(\hat{\bm{\Sigma}}_{\lambda}^{-1}-\bm{\Sigma}_{\lambda}^{-1}\right)\bm{\eta}+\bm{\Sigma}_{\lambda}^{-1}\left(\hat{\bm{\Sigma}}_{\lambda}-\bm{\Sigma}_{\lambda}\right)\bm{\Sigma}_{\lambda}^{-1}\bm{\eta}\nonumber \\
 & +\bm{\Sigma}_{\lambda}^{-1}\left(\hat{\bm{\Sigma}}_{\lambda}-\bm{\Sigma}_{\lambda}\right)\hat{\bm{\Sigma}}_{\lambda}^{-1}\left(\hat{\bm{\eta}}-\bm{\eta}\right)+{\color{violet}\bm{\Sigma}_{\lambda}^{-1}\left(\bm{\eta}-\hat{\bm{\eta}}\right)}\nonumber \\
 & =\bm{\Sigma}_{\lambda}^{-1}\left(\hat{\bm{\Sigma}}_{\lambda}-\bm{\Sigma}_{\lambda}\right){\color{red}\hat{\bm{\Sigma}}_{\lambda}^{-1}}\left(\bm{\Sigma}_{\lambda}-\hat{\bm{\Sigma}}_{\lambda}\right){\color{red}\bm{\Sigma}_{\lambda}^{-1}}\bm{\eta}+\bm{\Sigma}_{\lambda}^{-1}\left(\hat{\bm{\Sigma}}_{\lambda}-\bm{\Sigma}_{\lambda}\right)\bm{\Sigma}_{\lambda}^{-1}\bm{\eta}\label{eq:inter-2}\\
 & +\bm{\Sigma}_{\lambda}^{-1}\left(\hat{\bm{\Sigma}}_{\lambda}-\bm{\Sigma}_{\lambda}\right)\hat{\bm{\Sigma}}_{\lambda}^{-1}\left(\hat{\bm{\eta}}-\bm{\eta}\right)+{\color{violet}\bm{\Sigma}_{\lambda}^{-1}\left(\bm{\eta}-\hat{\bm{\eta}}\right)}\nonumber \\
 & \text{Note that by definition, }\bm{\Sigma}_{\lambda}-\hat{\bm{\Sigma}}_{\lambda}=\bm{\Sigma}-\hat{\bm{\Sigma}}\nonumber \\
 & =\bm{\Sigma}_{\lambda}^{-1}\underset{=O_{p}(1/\sqrt{n})}{\underbrace{\left(\hat{\bm{\Sigma}}_{\lambda}-\bm{\Sigma}_{\lambda}\right)}}{\color{red}\bm{\Sigma}_{\lambda}^{-1}}\underset{=O_{p}(1/\sqrt{n})}{\underbrace{\left(\bm{\Sigma}_{\lambda}-\hat{\bm{\Sigma}}_{\lambda}\right)}}\bm{\Sigma}_{\lambda}^{-1}\bm{\eta}\nonumber \\
 & +\bm{\Sigma}_{\lambda}^{-1}\left(\hat{\bm{\Sigma}}_{\lambda}-\bm{\Sigma}_{\lambda}\right){\color{red}\hat{\bm{\Sigma}}_{\lambda}^{-1}\left(\bm{\Sigma}_{\lambda}-\hat{\bm{\Sigma}}_{\lambda}\right)\bm{\Sigma}_{\lambda}^{-1}}\left(\bm{\Sigma}_{\lambda}-\hat{\bm{\Sigma}}_{\lambda}\right)\bm{\Sigma}_{\lambda}^{-1}\bm{\eta}\nonumber \\
 & +\bm{\Sigma}_{\lambda}^{-1}\left(\hat{\bm{\Sigma}}_{\lambda}-\bm{\Sigma}_{\lambda}\right)\hat{\bm{\Sigma}}_{\lambda}^{-1}\left(\hat{\bm{\eta}}-\bm{\eta}\right)+{\color{violet}\bm{\Sigma}_{\lambda}^{-1}\left(\bm{\eta}-\hat{\bm{\eta}}\right)}\nonumber \\
 & =\bm{\Sigma}_{\lambda}^{-1}\left(\hat{\bm{\Sigma}}-\bm{\Sigma}\right)\bm{\Sigma}_{\lambda}^{-1}\left(\bm{\Sigma}-\hat{\bm{\Sigma}}\right)\bm{\Sigma}_{\lambda}^{-1}\bm{\eta}+\bm{\Sigma}_{\lambda}^{-1}\left(\hat{\bm{\Sigma}}-\bm{\Sigma}\right)\hat{\bm{\Sigma}}_{\lambda}^{-1}\left(\hat{\bm{\eta}}-\bm{\eta}\right)+{\color{violet}\bm{\Sigma}_{\lambda}^{-1}\left(\bm{\eta}-\hat{\bm{\eta}}\right)}+O_{p}\left(\frac{1}{n^{3/2}}\right).\label{eq:inter-cross}
\end{align}
where two leading terms are of $O_{p}(1/n)$, and $\left(\bm{\eta}^{T}-\hat{\bm{\eta}}^{T}+\bm{\eta}^{T}\bm{\Sigma}_{\lambda}^{-1}\hat{\bm{\Sigma}}-\bm{\eta}^{T}\bm{\Sigma}_{\lambda}^{-1}\bm{\Sigma}\right)$
factor is of $O_{p}(1/\sqrt{n})$.\clearpage{}

\section{Additional Experiments}

\begin{figure}[h!]
\centering

\includegraphics[clip,width=0.95\textwidth,viewport=0bp 0bp 1073bp 777.902bp]{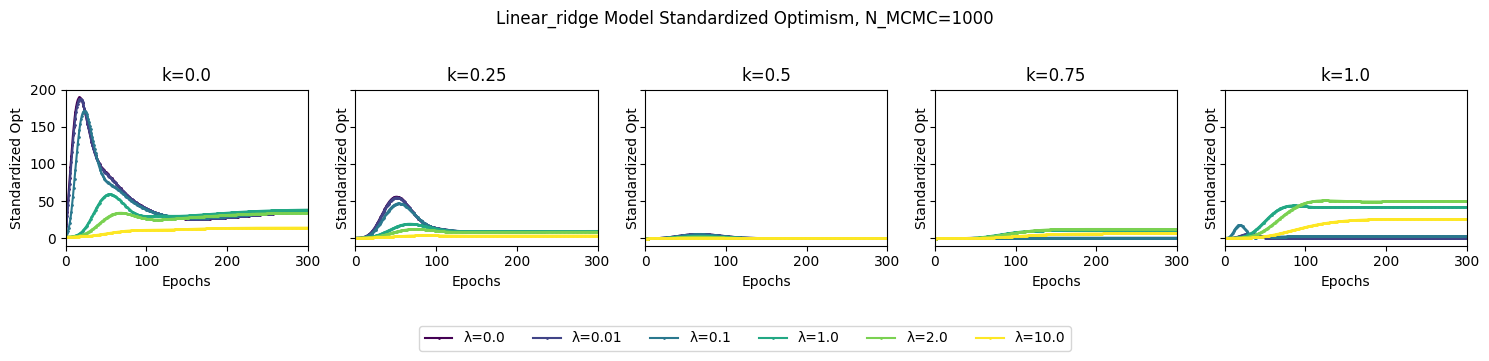}

\caption{\label{fig:Opt_vs_epoch_ridge-1}Expected scaled optimism (averaged
from 1000 MC simulations) versus the number of NN epochs for different
$k$ in \eqref{eq:test fun} with $\sigma_{\epsilon}^{2}=0.01$ for
a training set sampled from $N(0,1)$ of size 1000; and a testing
set sampled from $N(0,1)$ of size 1000. Models are optimized via
Adam optimizer with learning rate 0.01 and we provide the optimism
for the ridge linear regression model for comparison.}
\end{figure}

\begin{table}[h!]
\centering

\begin{tabular}{lrrrrrr}
\toprule 
\multicolumn{7}{l}{End of training scaled optimism for different combinations of $(k,\lambda_{ridge})$
300 epochs}\tabularnewline
\midrule 
\backslashbox{$k$}{$\lambda_{ridge}$}  & 0.00  & 0.01  & 0.10  & 1.00  & 2.00  & 10.00\tabularnewline
\midrule 
0.00  & 35.953481  & 35.941453  & 36.177433  & 38.346467  & 33.919739  & 14.174405\tabularnewline
0.25  & 9.564725  & 9.544049  & 9.480763  & 9.686489  & 8.514813  & 3.547711\tabularnewline
0.50  & 0.021075  & 0.020862  & 0.019124  & 0.010376  & 0.006821  & 0.001653\tabularnewline
0.75  & 0.021074  & 0.028329  & 0.642780  & 10.412455  & 12.357121  & 6.415719\tabularnewline
1.00  & 0.021075  & 0.052095  & 2.525693  & 41.667581  & 49.464615  & 25.714433\tabularnewline
\bottomrule
\end{tabular}\caption{Final scaled optimism (at the end of 300 epoches from Figure \ref{fig:Opt_vs_epoch_ridge-1}.}
\end{table}

\hfill{}
\end{document}